\documentclass{article}

\usepackage{geometry}

\usepackage{bbm}
\usepackage{mathtools}
\usepackage{amsmath,amsthm,amssymb}
\usepackage{comment}
\usepackage{amsfonts}
\usepackage{bm}
\usepackage{url}
\usepackage{microtype}
\usepackage{bbm}
\usepackage{mathtools}
\usepackage{amsfonts}
\usepackage{bm}
\usepackage{url}
\usepackage{graphicx} 
\usepackage{enumitem}
\usepackage[ruled,vlined]{algorithm2e}
\usepackage{xcolor}
\usepackage{natbib}

\newtheorem{assumption}{Assumption}
\newtheorem{defn}{Definition}
\newtheorem{prop}{Proposition}
\newcommand{\E}{\mathbb{E}}
\newcommand{\N}{\mathbb{N}}
\newcommand{\R}{\mathbb{R}}

\DeclareMathOperator*{\arginf}{arg\,inf}

\newcommand{\argmin}{\text{argmin}}

\newcommand{\one}{\mathbf{1}}
\newcommand{\classifier}{\phi}

\newcommand{\X}{\mathcal{X}}
\newcommand{\Y}{\mathcal{Y}}
\newcommand{\Z}{\mathcal{Z}}
\newcommand{\F}{\mathcal{F}}
\newcommand{\G}{\mathcal{G}}

\newcommand{\T}{\mathbb{T}}
\newcommand{\sample}{\mathcal{D}}
\newcommand{\Prob}{\mathbb{P}}
\newcommand{\supp}{\text{supp}}
\newcommand{\lossFunction}{\mathcal{L}}

\newcommand{\measurableMaps}{\mathcal{M}}

\newcommand{\rademacherComplexityEmpirical}{\hat{\mathfrak{R}}}

\newcommand{\kullbackLeiblerDivergence}{\mathrm{KL}}

\newcommand{\hypothesisClassBound}{\beta}
\newcommand{\probDistribution}{P}

\newcommand{\sampleXY}{\sample_{X,Y}}

\newcommand{\probDistributionZ}{\probDistribution}

\newcommand{\xSequenceSizeN}{{\bm{x}}_{1:n}}
\newcommand{\uSequenceSizeN}{{\bm{u}}_{1:n}}
\newcommand{\zSequenceSizeN}{{\bm{z}}_{1:n}}

\newcommand{\randomZSequenceSizeN}{\sample}
\newcommand{\sigmaSequenceSizeN}{{\bm{\sigma}}_{1:n}}

\newcommand{\lossFunctionBound}{B}
\newcommand{\riskArg}{\mathcal{R}}
\newcommand{\risk}{\riskArg_{\lossFunction,\probDistribution}}
\newcommand{\excessRiskArg}{\mathcal{E}}
\newcommand{\excessRisk}{\excessRiskArg_{\lossFunction,\probDistribution}}
\newcommand{\empiricalRiskArg}{\hat{\mathcal{R}}}
\newcommand{\empiricalRisk}{\empiricalRiskArg_{\lossFunction,\sample}}

\newcommand{\lipschitzConstantLoss}{\Lambda_{\mathrm{Lip}}}
\newcommand{\actionSpace}{\mathcal{V}}

\newcommand{\dist}{{d}}
\newcommand{\coveringNumber}{\mathcal{N}}
\renewcommand{\sampleXY}{\sample}
\renewcommand{\lossFunctionBound}{b}
\newcommand{\variance}{v}

\newcommand{\empiricalProbDistributionRandomZ}{\hat{\probDistribution}_{\randomZSequenceSizeN}}
\newcommand{\empiricalProbDistributionDeterministicZ}{\hat{\probDistribution}_{\zSequenceSizeN}}
\newcommand{\empiricalProbDistributionSample}{\hat{\probDistribution}_{\sample}}
\newcommand{\randomProjectionMeasure}{\nu_k}
\newcommand{\setOfRandomProjections}{\mathcal{A}_k}
\newcommand{\singleRandomProjection}{A}
\newcommand{\randomProjection}{A}
\newcommand{\bernsteinExponent}{\alpha}
\newcommand{\bernsteinConstant}{C_{\mathrm{B}}}
\newcommand{\lowDimensionalFunctionClass}{\mathcal{F}_k}
\newcommand{\Avg}{\textbf{Ens}}
\newcommand{\metricActionSpace}{\dist_{\actionSpace}}
\newcommand{\metricX}{\dist_{\X}}
\newcommand{\metricY}{\dist_{\Y}}
\newcommand{\sign}{\text{sgn}}
\newcommand{\quasiConvexityConstant}{\Lambda_{\mathrm{qc}}}
\newcommand{\coveringNumberConstant}{C_{\mathrm{cn}}}

\newcommand{\logBar}{{\log}_+}
\newcommand{\compressibilityFunction}{\psi}

\newcommand{\functionEstimate}{\hat{\phi}}
\newcommand{\functionGeneral}{\phi}
\newcommand{\oracleFunction}{\phi^*}

\newcommand{\geometricMarginConstant}{C_{\textrm{G}}}
\newcommand{\geometricMarginExponent}{\gamma}
\newcommand{\regressionFunction}{\eta}
\newcommand{\decisionBoundary}{\mathcal{S}_0}
\newcommand{\nearDecisionBoundary}{\mathcal{S}_{\xi}}
\newcommand{\momentConstant}{C_{\textrm{M}}}
\newcommand{\momentExponent}{\rho}
\newcommand{\johnsonLindenstraussConstant}{C_{\mathrm{JL}}}
\newcommand{\tjohnsonLindenstraussConstant}{{c}_{\mathrm{B}}}
\newcommand{\zeroOneLoss}{\lossFunction_{\text{0,1}}}
\newcommand{\sqrLoss}{\lossFunction_{\text{sqr}}}
\newcommand{\klLoss}{\lossFunction_{\text{kl}}}
\newcommand{\pseudoDimension}{\text{PDim}}

\newcommand{\nonDegeneracyConstant}{\varpi}

\newcommand{\nkDeltaComplexityTerm}{\mathfrak{C}_{n,k,\delta}}
\newcommand{\nDeltaDataComplexityTerm}{\hat{\mathfrak{C}}_{n,\delta}(\randomZSequenceSizeN)}
\newcommand{\approxError}{\theta}

\newcommand{\pbad}{{P_{-}}}

\newcommand{\pgood}{{P_{+}}}

\newtheorem{theorem}{Theorem}
\newtheorem{example}{Example}

\newtheorem{lemma}[theorem]{Lemma}
\newtheorem{corollary}[theorem]{Corollary}

\begin{document}

\title{Statistical optimality conditions for compressive ensembles}

\author{Henry W.J. Reeve and Ata Kab\'an}

\maketitle

\begin{abstract}

We present a framework for the theoretical analysis of ensembles of low-complexity empirical risk minimisers trained on independent random compressions of high-dimensional data. First we introduce a general distribution-dependent upper-bound on the excess risk, framed in terms of a natural notion of compressibility. This bound is independent of the dimension of the original data representation, and explains the in-built regularisation effect of the compressive approach. We then instantiate this general bound to classification and regression tasks, considering Johnson-Lindenstrauss mappings as the compression scheme. For each of these tasks, our strategy is to develop a tight upper bound on the compressibility function, and by doing so we discover distributional conditions of geometric nature under which the compressive algorithm attains minimax-optimal rates up to at most poly-logarithmic factors. In the case of compressive classification, this is achieved with a mild geometric margin condition along with a flexible moment condition that is significantly more general than the assumption of bounded domain. In the case of regression with strongly convex smooth loss functions we find that compressive regression is capable of exploiting spectral decay with near-optimal guarantees. In addition, a key ingredient for our central upper bound is a high probability uniform upper bound on the integrated deviation of dependent empirical processes, which may be of independent interest.

\end{abstract}

\section{Introduction}
Compressive learning aims to make use of inexpensive dimensionality reduction or sketching methods to overcome the curse of dimensionality in statistical learning. 
The term was coined by \cite{Calderbank09} in analogy with compressive sensing (CS) \citep{Donoho,Candes}, which established sparsity conditions under which high dimensional signals are recoverable from their low dimensional linear random projection.  
The spectacular advances of CS provide data acquisition devices 
that directly collect random projections of the data without storing the original \citep{Durate}, and most recently, dedicated photonic computing hardware 
became available that can perform random projection in a massively parallel fashion \citep{OPU,GribonvalPhaseRetrieval}. 
Such technologies open new doors for dealing with massive high dimensional data sets, and inspire new research in areas as diverse as numerical analysis \citep{Halko}, statistical methodology \citep{Heinze,cannings2018,Tian}, pattern recognition \citep{Calderbank16}, clustering \citep{Boutsidis,Biau,nips2}, optimisation \citep{PilanciIT,PilanciJMLR,PilanciSIAM,nips20}, search based software engineering
\citep{evoInSW}, imaging \citep{Lusting,YeReview,Styles,Bentley}, medical research \citep{medical}, neuroscience \citep{Arriaga15}, and computer vision \citep{Jiao}. The interested reader may also refer to recent surveys 
\citep{GibsonSurvey}, \citep{CanningsSurvey}, and references therein.

Whilst the theory of compressive sensing and random projection based dimensionality reduction is well understood, the use of these methods in machine learning raises important questions about a theoretical understanding of the  risk of compressive learning. 

Firstly, the goal in statistical learning is very different from both compressive sensing and dimensionality reduction, as we do not aim to recover, or even to approximate the seen data, instead we aim to produce accurate predictions on unseen data. This motivates the search for sufficient conditions for controlling the excess risk of compressive ensembles as a function of natural geometric characteristics of the problem.

The setting we consider is analogous to that of compressive sensing in that the data features are only available in compressive form -- that is, the data features undergo compression before being fed to a learning algorithm. We do not impose any other regularisation to make the high dimensional problem learnable from a finite sample, and are interested in conditions on the unknown data distribution under which this compression alone makes the resulting ensemble of empirical risk minimisers nearly minimax optimal.

Previous work in this setting provided upper bounds on the error for several compressive learning machines -- most often under (a combination of) existing assumptions from statistical learning theory (e.g. large margin, norm constraints, generative model assumptions) or from compressed sensing (e.g. sparse representation, low complexity feature space), or seek to interpret bounds in these terms -- for compressive classification \citep{Arriaga,Balcan,Calderbank09,DurrantKabanKDD10,Calderbank2,
Renna16
}, compressive regression \citep{Maillard,Fard,kaban14,
Thanei,Slawski}, and other learning tasks. 
The work of \cite{Chen} explicitly studied the rate of convergence of the excess risk for compressive regularised kernel-based learning in a reproducing Hilbert space with the least square and hinge losses, obtaining upper bounds with a rate of order $n^{-1/4}$. However, the statistical optimality of these results, e.g. in the minimax sense, has not been established. 
\cite{Reeve} obtained the minimax-optimal rates for the compressive k-nearest neighbour algorithm in cost-sensitive classification where the data support was assumed to have a manifold structure. The assumptions and analysis are very different from those of the present work. Optimality results have also been obtained in other problem settings, where the compression is used purely to speed up a statistically optimal predictor in a way that preserves its optimality -- for instance \cite{yang2017} give matching upper and lower bounds for compressive kernel ridge regression with compression applied to the kernel matrix.

Secondly, ensembles of compressive learners trained on independent random compressions of the high dimensional data have been found to increase performance in practice, as computations can be run in parallel, and random variations are reduced. However, a bottleneck for the theoretical understanding of such ensembles is that, unlike traditional ensembles, which combine predictors from the same function class, here each predictor added to the ensemble corresponds to a new random projection of the data. 
Previous work in machine learning \citep{DurrantMLJ} analysed in detail a special case of compressive Fisher Linear Discriminant ensemble, and found a desirable implicit regularisation effect, which prevents overfitting. However, the analytic methods used there are specific to a particular generative model and it is not clear whether a similar effect occurs more generally. Subsequent work in statistics considered a more general approach in terms of the base learners employed \citep{cannings2018,Slawski}, and  provided upper bounds on the excess risk of random projection ensembles in expectation w.r.t. the training set. Bounds in expectation have also been given very recently for axis-aligned random subspace ensembles \citep{Tian}. However, we are interested in high probability bounds to reveal more information about the worst case behaviour for the excess risk, subject to a failure probability. In another line of research, recent work by \cite{Lopes} determined the asymptotic speed of convergence as the number of predictors in the ensemble grows. While this is informative for very large ensembles, we are interested in non-asymptotic guarantees for ensembles of any given finite size.

\subsection{Overview of contributions}
In this paper we introduce a general framework for the theoretical study of compressive learning. This framework facilitates the discovery of new conditions, specific to the learning task, which allow favourable convergence rates for compressive learning. 
The algorithmic approach we consider throughout this study consists of ensembles of any number of empirical predictors made of compressive empirical risk minimisers (ERM) that are trained in parallel on independent randomised compressions of data sets of arbitrarily many dimensions, with their predictions combined by an averaging-type operation. This simple procedure is presented in Algorithm \ref{RPEERMAlgo}.

Below we summarise our main results.

\begin{itemize}
    \item We introduce the concept of a compressibility function, which quantifies the average excess loss incurred by working with a low-complexity function of compressed features and plays a key role in guiding the analysis of compressive learning. 
    \item We give a general distribution-dependent high probability upper-bound on the excess risk of ERM ensembles of arbitrary size, composed of low-complexity predictors 
(Theorem \ref{ThmMainResult}). Our bound contains three terms: (1) A statistical error that decays with the sample size at a rate depending upon the Bernstein-Tsybakov noise exponent, (2) A term which converges to zero as the size of our ensemble grows and (3) The compressibility function. This reveals an implicit regularisation effect that is exhibited by a wide variety of randomised heterogeneous ensembles.
\end{itemize}
The general form of our upper bound and compressibility function allow us 
to study specific learning problems in a unified framework. Specifically, we are interested in the following question: Under what natural conditions can the compressive ERM ensemble attain minimax-optimal rates of convergence with respect to the sample size? 
To approach this question we restrict attention to Johnson-Lindenstrauss mappings as the compression scheme, which makes it feasible to control the compressibility function, and we demonstrate our approach by instantiating our general upper bound in two fundamental classic learning tasks. 
\begin{itemize}
\item In the case of compressive classification we find that very mild conditions of geometric nature suffice for a high probability upper bound on the excess error (Theorem \ref{thm:geometricMarginThm}). These conditions include a flexible geometric margin condition that differs significantly from previously considered margin assumptions, along with a flexible moment condition that allows for distributions supported on an unbounded domain. For a wide range of parameters we find better rates than the previous upper bounds implied by \citep{Chen} (albeit in a different setting) even in the absence of favourable Tsybakov-margin. In fact, in Theorem \ref{clRLBtext} we show that the upper bound for compressive ensembles given by Theorem \ref{thm:geometricMarginThm} is minimax optimal up to logarithmic factors.
\item In the case of regression with strongly convex loss functions we find that Johnson-Lindenstrauss compressors are capable of exploiting spectral decay with near-minimax optimal guarantees (Theorems \ref{thm:OLSThm}-\ref{regRLBtext}). Our high probability guarantee highlights the role of spectral decay in attaining near-optimality. These results complement recent findings by \cite{Slawski} which give an expectation bound for compressive OLS with fixed design.
\item
Our general upper bound builds on a high probability uniform bound on the integrated deviation for dependent empirical processes that allows to exploit local Rademacher complexities (Theorem \ref{ThmMainUniform}), which may be of independent interest.
\end{itemize}

Our approach provides a framework that places a computationally attractive and empirically successful algorithmic scheme on solid theoretical foundations. Our framework can be extended and used to unearth novel conditions which help gaining more understanding in specific compressive learning problems. 

\section{Problem setting}
\label{statisticalSettingSec}
We shall consider supervised learning. Suppose we have complete separable metric spaces $(\X,\metricX)$, $(\Y,\metricY)$, $(\actionSpace,\metricActionSpace)$, where $\X$ is a feature space, $\Y$ is a target space, and $\actionSpace$ is a prediction space, which may or may not equal $\Y$. In typical applications $(\X,\metricX)$, $(\Y,\metricY)$, $(\actionSpace,\metricActionSpace)$ will be subsets of Euclidean space with their respective Euclidean norms. However, the additional level of generality in this section comes at no expense. We shall assume that there is an unknown Borel probability distribution $\probDistribution$ over random variables $(X,Y)$, where $X$ takes values in $\X$, and $Y$ takes values in $\Y$.  The quality of a prediction for a given target is quantified through a loss function $\lossFunction: \actionSpace \times \Y \rightarrow [0,\infty)$.

Given a pair of measurable spaces $\mathcal{Z}$ and $\mathcal{W}$, we let $\measurableMaps(\Z,\mathcal{W})$ denote the set of measurable functions $g:\mathcal{Z} \rightarrow \mathcal{W}$. For brevity we let $\measurableMaps(\Z)$ denote $\measurableMaps(\Z,\R)$ and for each $\lossFunctionBound>0$ let $\measurableMaps_{\lossFunctionBound}(\Z) := \measurableMaps(\Z,[-\lossFunctionBound,\lossFunctionBound])$, where both $\R$ and $[-\lossFunctionBound,\lossFunctionBound]$ are endowed with the Borel sigma algebra. The goal of the learner is to obtain ${\functionGeneral} \in \measurableMaps(\X,\actionSpace)$ such that the corresponding risk 
\begin{align*}
\risk({\functionGeneral}):= \E_{(X,Y) \sim \probDistribution}[  \lossFunction({\functionGeneral}(X),Y)] = \int \lossFunction(\functionGeneral(x),y)d\probDistribution(x,y)
\end{align*}
is as low as possible. Given a Borel probability distribution $\probDistribution$ we let $\oracleFunction \equiv \oracleFunction_{\probDistribution}\in \measurableMaps(\X,\actionSpace)$ denote the Bayes optimal predictor, satisfying 
\begin{align}\label{defnBayesOptimalPredictor}
 \oracleFunction \in \argmin_{{\functionGeneral} \in \measurableMaps(\X,\actionSpace)}\left\lbrace \risk({\functionGeneral})\right\rbrace.  
\end{align}
For simplicity, we shall assume throughout that $\actionSpace$ is a compact metric space and $\lossFunction$ is continuous in its first argument, which ensures that a Bayes optimal predictor $\oracleFunction$ exists ({Proposition \ref{bayesOptimalExistsWhenLossContinuousActionSpaceCompact}}), 
although it need not be unique. We view the Bayes optimal predictor $\oracleFunction \in \measurableMaps(\X,\actionSpace)$ as the mapping we would select if we knew the distribution $\probDistribution$. 

Of course, in practice the learner does not have direct access to the distribution $\probDistribution$. Instead, the learner selects ${\functionGeneral} \in \measurableMaps(\X,\actionSpace)$ based upon a sample $\sampleXY:=\{(X_j,Y_j)\}_{j \in [n]}$, where $(X_j,Y_j)\sim \probDistribution$ are independent copies of $(X,Y)$. Whilst the true risk $\risk({\functionGeneral})$ cannot be directly observed, the learner does have access to the \emph{empirical risk}  $\empiricalRisk({\functionGeneral},\sample):=n^{-1}\cdot \sum_{j \in [n]}\lossFunction({\functionGeneral}(X_j),Y_j)$. Given ${\functionGeneral} \in \measurableMaps(\X,\actionSpace)$,
we write 
\begin{align}
\excessRisk(\functionGeneral):=\risk({\functionGeneral})-\risk(\oracleFunction) = \risk({\functionGeneral})-\inf_{{\functionGeneral} \in \measurableMaps(\X,\actionSpace)}\left\lbrace \risk({\functionGeneral})\right\rbrace
\end{align}
for the excess risk. For a positive integer $m \in \N$ we shall use the notation $[m]:= \{1,\cdots,m\}$. 
We also define the notation $\logBar(x):=\max\{\log(x),1\}$, where $\log$ is the natural logarithm.

\subsection{Learning from compressive data sketches}

\newcommand{\Xk}{\X_k}
In this work we consider a \emph{high-dimensional} setting where the dimensionality of the feature space $\X$ is arbitrarily large, and working directly with the features themselves becomes computationally and statistically prohibitive. Instead, we work with randomised compressions of the feature representation of the data, for instance via \emph{random projections}. 

Given $k \in \N$ we let $\setOfRandomProjections \subseteq \measurableMaps(\X,\R^k)$ be a set of random feature mappings. Given a data sample $\sampleXY:=\{(X_j,Y_j)\}_{j \in [n]}$ and a mapping $\singleRandomProjection \in \setOfRandomProjections$, we define the corresponding compressed sample ${\singleRandomProjection}(\sampleXY):= \{(\singleRandomProjection(X_j),Y_j)\}_{j \in [n]}$. Let $\randomProjectionMeasure$ be a probability distribution on the set of random feature mappings $\setOfRandomProjections$. Given $m \in \N$, we take $m$ random projections $\randomProjection_1, \cdots, \randomProjection_m$ that are independent and identically distributed with each $\randomProjection_i \sim \randomProjectionMeasure$, and consider $\randomProjection_1(\sampleXY),\cdots,\randomProjection_m(\sampleXY)$, that is $m$ random projections of the data. We let $\lowDimensionalFunctionClass \subseteq \measurableMaps(\R^k,\actionSpace)$ be a set of functions on the transformed feature space $\R^k$. We shall view $\lowDimensionalFunctionClass$ as being of relatively small capacity in a sense that will be made precise in Section \ref{assumptions} (Assumption \ref{logarithmicCoveringNumbersAssumption}). Examples will include sets of linear classifiers on $\R^k$. For each $i \in [m]$, we shall choose $\hat{f}_{i} $ in $\lowDimensionalFunctionClass$ based on the compressed sample $\randomProjection_i(\sampleXY)$ by minimising $\empiricalRisk\left(f_i,\randomProjection_i(\sampleXY)\right)$ over $f \in \lowDimensionalFunctionClass$. At test time, the set of predictions $\{ \hat{f}_{i}(\randomProjection_i(x))\}_{i \in [m]}$ is combined into a voting ensemble through a function $\Avg$:
\begin{align*}
{\functionEstimate}(x)\equiv \Avg\left(\left\lbrace \hat{f}_{i}(\randomProjection_i(x))\right\rbrace_{i \in [m]}\right).
\end{align*}
The appropriate combination rule $\Avg$ depends on the learning task. For example, in the case of classification with the zero-one loss we advocate taking the modal average, and in the case of regression with a squared loss we advocate taking the mean average. The reasons for this will become clear shortly in Section \ref{assumptions}.
The pseudo-code for the procedure that we study in the remainder of this paper is described in Algorithm \ref{RPEERMAlgo}.

\begin{algorithm}[H]
\SetAlgoLined
\SetKwInOut{Input}{Input}
\SetKwInOut{Output}{Output}
\Input{A data sample $\sampleXY$, a number of projections $m$, a distribution over random compressors $\randomProjectionMeasure$, a loss function $\lossFunction$ and a low-dimensional function class $\lowDimensionalFunctionClass$.}
\vspace{1mm}
\For{$i \in [m]$}{
\vspace{1mm}
Sample $\randomProjection_i \sim \randomProjectionMeasure$\;\vspace{1mm}
Compute ${\randomProjection}_i(\sampleXY):= \{(\randomProjection_i(X_j),Y_j)\}_{j \in [n]}$\;
Choose $\hat{f}_{i} \in  \lowDimensionalFunctionClass$ to minimise $\empiricalRisk\left(f,\randomProjection_i(\sampleXY)\right)$\;
}
Combine ${\functionEstimate}(x):= \Avg\left(\left\lbrace \hat{f}_{i}(\randomProjection_i(x))\right\rbrace_{i \in [m]}\right)$\;
\Output{Compressive ensemble predictor ${\functionEstimate}$.}
\caption{Compressive ensemble empirical risk minimisers \label{RPEERMAlgo}}
\end{algorithm}
\newcommand{\optimisationError}{\mathbb{O}}
\newcommand{\averageOptimisationError}{\overline{\mathbb{O}}}

Minimising the empirical risk can often be a challenging optimisation problem \citep{feldman2012agnostic}, and, in general, an exact minimiser need not even exist.  In this work, we focus on the statistical challenge of learning with compressive ensembles and assume that the optimisation error is dominated by the statistical error (see Section \ref{sec:introCompressibility} for details). 

We also remark that strictly speaking $\hat{\phi} \in \measurableMaps((\X\times\Y)^n \times \setOfRandomProjections^m \times \X, \actionSpace)$, since it implicitly depends upon both the random sample $\sample$, which takes values in $(\X\times\Y)^n$, and the sequence of random projections $(A_i)_{i \in [m]}$, which takes values in $\setOfRandomProjections^m$. However, we typically view $\hat{\phi}$ as a random element of $\measurableMaps(\X,\actionSpace)$, suppressing the dependence upon $\sample$ and $(A_i)_{i \in [m]}$ for notational convenience.

\section{General framework and main upper bound}\label{sec:general}
This section presents the main assumptions we employ throughout of this work, along with some illustrative examples. We also introduce a notion of compressibility, which will allow us to state our main result in general terms in the next section, before studying specific instances.

\subsection{Initial assumptions}\label{assumptions}
We begin with two standard assumptions on the loss function.
\begin{assumption}[Bounded loss function]\label{boundedLossFunctionAssumption} We shall assume that the loss function $\lossFunction:\actionSpace\times \Y \rightarrow [0,\infty)$ is bounded by some constant $\lossFunctionBound\geq 1$, so for all $v \in \actionSpace$ and $y \in \Y$ we have $\lossFunction(v,y) \in [0,\lossFunctionBound]$.
\end{assumption}

\begin{assumption}[Lipschitz loss function]\label{lipschitzLossFunctionAssumption} We shall assume that $\actionSpace$ is a compact metric space with metric $\metricActionSpace$ and there exists $\lipschitzConstantLoss \geq 1$ such that $\left|\lossFunction(v_0,y)-\lossFunction(v_1,y)\right| \leq \lipschitzConstantLoss \cdot \metricActionSpace(v_0,v_1)$  for all $v_0,v_1 \in \actionSpace$ and $y \in \Y$.
\end{assumption}
Typically we have $\actionSpace\subseteq \R$, in which case we can take $\metricActionSpace$ to be the standard metric defined by $\metricActionSpace(v_0,v_1)= |v_0-v_1|$. 

The following assumption will connect the excess risk of the ensemble with the average excess risk of its members.
The function $\Avg$ is a measurable map of the form $\Avg:\actionSpace^m \rightarrow \actionSpace$, so $\Avg \in \measurableMaps\left(\actionSpace^m,\actionSpace\right)$. We can extend $\Avg \in \measurableMaps\left(\actionSpace^m,\actionSpace\right)$ to a map $\Avg: \measurableMaps(\X,\actionSpace)^m \rightarrow \measurableMaps(\X,\actionSpace)$ in a point-wise fashion by defining
$\Avg\left( \left\lbrace \phi_i \right\rbrace_{i \in [m]}\right) (x) =  \Avg\left( \left\lbrace \phi_i(x) \right\rbrace_{i \in [m]}\right)$,
for any $\left\lbrace \phi_i \right\rbrace_{i \in [m]} \in \measurableMaps(\X,\actionSpace)^m$ and $x \in \X$.

\begin{assumption}[Quasi-convexity]\label{quasiConvexityAssumption} We shall say that $\lossFunction$ satisfies the \emph{quasi-convexity} assumption with constant $\quasiConvexityConstant\geq 1$ and averaging function $\Avg$ if
\begin{align*}
\excessRisk\left\lbrace \Avg\left( \left\lbrace \phi_i \right\rbrace_{i \in [m]} \right)\right\rbrace\leq  \frac{\quasiConvexityConstant}{m}\sum_{i \in [m]} \excessRisk\left( \phi_i ,\probDistribution\right)
\end{align*}
 for all Borel probability distributions $\probDistribution$ on $\X \times \Y$  and all $\left\lbrace \phi_i \right\rbrace_{i \in [m]} \in \measurableMaps(\X,\actionSpace)^m$.
\end{assumption}

The nomenclature comes from the following consequence of Jensen's inequality (Lemma \ref{convexImpliesQuasiConvex}). However the benefit of the quasi-convexity assumption is that it applies also to the 0-1 loss (Lemma \ref{verifyingQuasiConvexityAssumptionsLemma}).

\begin{lemma}\label{convexImpliesQuasiConvex} Suppose that $\actionSpace$ is a vector space and $\lossFunction$ is convex in its first argument. Then Assumption \ref{quasiConvexityAssumption} holds with $\quasiConvexityConstant=1$ and $\Avg((v_i)_{i \in [m]})=\frac{1}{m}\sum_{i \in [m]}v_i$ for $(v_i)_{i \in [m]} \in \actionSpace^m$.
\end{lemma}

Next, we make precise the idea that the set of functions on the low dimensional space $\lowDimensionalFunctionClass \subseteq \measurableMaps\left(\R^k,\actionSpace\right)$ is of low capacity, through the following assumption. First, recall the notion of covering numbers. Given a set $\T$ with metric $\dist_{\T}$ and $\epsilon>0$, the $\epsilon$-covering number $\coveringNumber(\T,\dist_{\T},\epsilon)$ is the cardinality of the smallest subset $\tilde{\T}\subseteq \T$ such that for every $t \in \T$ there exists some $\tilde{t} \in \tilde{\T}$ with $\dist_{\T}(\tilde{t},t)\leq \epsilon$. Given any $n\in \N$ and any $\uSequenceSizeN = \{u_j\}_{j \in [n]} \in (\Xk)^n$ the empirical $\ell_2$ metric ${\dist_{\uSequenceSizeN}^{\actionSpace}}$ on $\lowDimensionalFunctionClass$ is defined for $f_0$, $f_1 \in \lowDimensionalFunctionClass$ by
\begin{align*}
{\dist_{\uSequenceSizeN}^{\actionSpace}}(f_0,f_1):=\sqrt{\frac{1}{n}\sum_{j \in [n]}\metricActionSpace\left(f_0(u_j),f_1(u_j)\right)^2}.
\end{align*}

\begin{assumption}[Covering number condition]\label{logarithmicCoveringNumbersAssumption} We shall say that $\lowDimensionalFunctionClass \subseteq \measurableMaps\left(\R^k,\actionSpace\right)$ satisfies the  \emph{logarithmic covering number} assumption with constant $\coveringNumberConstant \geq 1$ and bound $\hypothesisClassBound\geq 1$ if for every $n, k \in \N$ with $n > k$, $\uSequenceSizeN \in \left(\R^k\right)^n$, $\epsilon>0$,
\begin{align*}
\coveringNumber\left(\lowDimensionalFunctionClass,{\dist_{\uSequenceSizeN}^{\actionSpace}},\epsilon\right) \leq  \left(\frac{\hypothesisClassBound n}{\epsilon }\right)^{\coveringNumberConstant  k}.
\end{align*}
\end{assumption}

Finally, we shall make use of 
the Bernstein-Tsybakov condition, which has been key to obtaining fast rates in the statistical learning literature \cite{tsybakov2004optimal}. 

\begin{assumption}[Bernstein-Tsybakov condition]\label{bernsteinTysbakovMarginTypeAssumption} We shall say that $\lossFunction$ and $\probDistribution$ satisfy the Bernstein-Tsybakov condition with exponent $\bernsteinExponent \in [0,1]$ and constant $\bernsteinConstant \geq 1$ if
\begin{align*}
\E_{(X,Y)\sim\probDistribution}\left[\left\lbrace \lossFunction\left(\functionGeneral(X),Y\right)-\lossFunction\left(\oracleFunction_{\mathrm{P}}(X),Y\right)\right\rbrace^2\right] \leq \bernsteinConstant \cdot\excessRisk\left(\functionGeneral\right)^{\bernsteinExponent},
\end{align*} 
 for all $\functionGeneral \in \measurableMaps\left(\X,\actionSpace\right)$. 
\end{assumption}

\noindent Here $\oracleFunction_{\mathrm{P}} \in \measurableMaps\left(\X,\actionSpace\right)$ denotes a Bayes optimal predictor satisfying \eqref{defnBayesOptimalPredictor}. Note that the Bernstein-Tsybakov condition is not necessarily restrictive, since it always holds with $\bernsteinExponent=0$ and $\bernsteinConstant = \lossFunctionBound^2$. However, faster rates than $n^{-1/2}$ are obtainable whenever Assumption \ref{bernsteinTysbakovMarginTypeAssumption} holds with $\alpha>0$.

\begin{example}[Binary classification with the zero-one loss]\label{e1} Take $\actionSpace=\Y = \{-1,+1\}$ and the \emph{zero-one} loss function $\zeroOneLoss(v,y) = \one\left\lbrace v\neq y\right\rbrace$ for $(v,y) \in \actionSpace\times \Y$, with the class $\lowDimensionalFunctionClass^{\text{0,1}}=\left\lbrace \bm{u} \mapsto \sign\left(\bm{w} \cdot \bm{u}-t\right):\hspace{2mm}\bm{w}\in \R^k,\hspace{1mm}t \in \R\right\rbrace \subseteq \measurableMaps\left(\R^k,\{-1,+1\}\right)$. 
\end{example}

\begin{example}[Bounded regression with the squared loss]\label{e2} Take
$\actionSpace=\Y \subseteq [-\hypothesisClassBound,+\hypothesisClassBound]$ for some $\hypothesisClassBound>0$,
and consider the \emph{squared} loss function $\sqrLoss(v,y) =\left( v- y\right)^2$ for $(v,y) \in \actionSpace\times \Y$, with the class
$\lowDimensionalFunctionClass^{\text{bl}}=\left\lbrace \bm{u} \mapsto \max\left( \min\left( \bm{w} \cdot \bm{u}-t,\hypothesisClassBound\right),-\hypothesisClassBound\right):\hspace{2mm}\bm{w}\in \R^k,\hspace{1mm}t \in \R\right\rbrace \subseteq \measurableMaps_{\hypothesisClassBound}\left(\R^k\right)$.
\end{example}

\begin{example}[Conditional probability estimation with the Kullback-Leibler loss]\label{e3} 
  Take $\Y = \{0,1\}$, define a mapping $\pi:\R\rightarrow [0,1]$ by $\pi(a)=e^a/(1+e^a)$, and consider the {Kullback Leibler} divergence $\text{kl}(v,y) = 
  y\log(y/v)+(1-y)\log((1-y)/(1-v))$
for $(v,y) \in [0,1]\times \Y$. Take $\actionSpace = [-\hypothesisClassBound,\hypothesisClassBound]$ for some $\hypothesisClassBound>0$, and consider the \emph{Kullback Leibler} loss function $\klLoss := \text{kl} \circ \pi$, acting on the class of functions $\lowDimensionalFunctionClass^{\text{bl}}$.
\end{example}

In Appendix \ref{verifying} we verify that Examples \ref{e1}-\ref{e3} satisfy Assumptions 1-5. We remark that Example \ref{e1} satisfies the quasi-convexity condition (Assumption \ref{quasiConvexityAssumption}) with the modal average (i.e. majority voting), Examples \ref{e2} and \ref{e3} satisfy the same condition with the arithmetic average of $\{\functionGeneral_i\}_{i\in[m]} 
\in \lowDimensionalFunctionClass^{\text{bl}}$, 
and 
in the case of Example \ref{e3} this corresponds to the product of experts combination \citep{poe} of the nonlinear probabilistic outputs $\{\pi(\functionGeneral_i)\}_{i\in[m]}$.

\subsection{Compressibility}\label{sec:introCompressibility}
Our main upper bound on the excess risk of Algorithm \ref{RPEERMAlgo} in the next section will be expressed in terms of a compressibility function $\compressibilityFunction_{\probDistribution}:\N \rightarrow [0,\infty)$, defined for each $k \in \N$ as the expected approximation error of the compressive class $\lowDimensionalFunctionClass$:
\begin{align*}
\compressibilityFunction_{\probDistribution}(k):=\E_{\randomProjection \sim \randomProjectionMeasure}\left[ \inf_{f \in \lowDimensionalFunctionClass
}\left\lbrace \excessRisk\left(f \circ \randomProjection\right)\right\rbrace \right].
\end{align*}
The compressibility function $\compressibilityFunction_\probDistribution$ quantifies the average amount of loss incurred by predicting with the best member of the class $\lowDimensionalFunctionClass$ with compressed inputs $\randomProjection(x)$, rather than the Bayes-optimal predictor $\oracleFunction_{\probDistribution}$. In order to focus on the statistical aspects of the problem we shall assume that the optimisation error $\optimisationError:=\frac{1}{m}\sum_{i=1}^m(\empiricalRisk(\hat{f}_i,\randomProjection_i(\sampleXY))-\inf_{f \in \lowDimensionalFunctionClass}\{ \empiricalRisk(f,\randomProjection_i(\sampleXY))\})$ is dominated by the compressibility term $\compressibilityFunction(k)$. The functional form of the compressibility function is specific to the randomisation scheme employed, and the learning task. Examples  will be given in Sections \ref{sec:class} and \ref{sec:reg}.

\subsection{Main upper bound}
\newcommand{\constantForMainHighProbThm}{C_{\ref{ThmMainResult}}}
With our framework in place, 
we give a general upper bound on the worst case excess risk of the empirical predictor returned by Algorithm \ref{RPEERMAlgo}. 

\begin{theorem}
\label{ThmMainResult}  Suppose that Assumptions \ref{boundedLossFunctionAssumption}, \ref{lipschitzLossFunctionAssumption}, \ref{quasiConvexityAssumption}, \ref{logarithmicCoveringNumbersAssumption} and \ref{bernsteinTysbakovMarginTypeAssumption} hold with parameters $\lossFunctionBound$, $\hypothesisClassBound$, $\lipschitzConstantLoss$, $\quasiConvexityConstant$, $\coveringNumberConstant$,  $\bernsteinConstant$ $> 1$, $\bernsteinExponent \in [0,1]$. 
Take $n, k,m \in \N$ and let $\functionEstimate_{n,k,m}$ denote the compressive ensemble predictor from Algorithm \ref{RPEERMAlgo} with a sample size $n$, a projection dimension $k$, an ensemble size $m$. There exists a constant $\constantForMainHighProbThm\geq 1$, depending only upon $\lossFunctionBound$, $\hypothesisClassBound$, $\lipschitzConstantLoss$, $\quasiConvexityConstant$, $\coveringNumberConstant$, $\bernsteinConstant$, $\bernsteinExponent$, 
such that, given $\sample \sim \probDistribution^{\otimes n}$ and $(A_i)_{i \in [m]} \sim \randomProjectionMeasure^{\otimes m}$, the following holds with probability at least  $1-\delta$,
\begin{align*}
\excessRisk\left(\functionEstimate_{n,k,m}\right)
&\leq \constantForMainHighProbThm \left\lbrace
\compressibilityFunction_{\probDistribution}(k)+
\left(\frac{k \cdot \logBar(n)+\logBar(1/\delta)}{ n} \right)^{\frac{1}{2-\bernsteinExponent}}
+\frac{\logBar(1/\delta)}{m} 
\right\rbrace.
\end{align*}
\end{theorem}

Theorem \ref{ThmMainResult} provides generalisation guarantees in arbitrary learning problems which are independent of the dimensionality of the original feature space. The bound consists of three terms. The first term $\compressibilityFunction_{\probDistribution}(k)$ corresponds to the amount of accuracy lost by working on a $k$-dimensional compression of the original problem. The second term corresponds to the statistical difficulty of learning in the $k$-dimensional setting. The first two terms are in tension: By decreasing $k$ we can reduce the statistical difficulty of our $k$-dimensional problem.  However, this reduction in statistical error comes at the expense of an increase in the compressibility term $\compressibilityFunction_{\probDistribution}(k)$. The reduction in statistical error for low $k$ shows that compressive ERM ensembles in Algorithm \ref{RPEERMAlgo} perform in-built regularisation effect to guard against overfitting. However, for well-behaved distributions with \emph{compressible structure} the term $\compressibilityFunction_{\probDistribution}(k)$ can be made small with modest values of $k$ yielding efficient dimension independent rates. We shall discuss examples of compressible structure for specific learning problems in Sections \ref{sec:class} and \ref{sec:reg}.  

The third term in the bound corresponds to the error contribution from working with a finite ensemble size $m$. Note that the bound is non-asymptotic and holds for any finite ensemble size $m$, and one may set $m$ to be of order $n$. The excess risk guarantee improves as the ensemble size $m$ grows, at a speed that matches the asymptotically optimal rate $m^{-1}$ determined in recent work of \cite{Lopes}.

The proof of Theorem \ref{ThmMainResult} is given in Section \ref{SecProofs}. 
The main bottleneck is to prevent the excess risk probabilities of individual ensemble members from accumulating with the ensemble size. To achieve this,
the starting point is a uniform upper bound on the integrated deviations of dependent empirical processes, which might find applications elsewhere.

So far we have left unspecified the randomised dimensionality compression scheme to be used. 
Indeed the general result presented in this section could potentially be applied to any independently randomised ensemble, including random coordinate projections \citep{Ho,Tian}, and various sketching methods \citep{Cormode,CanningsSurvey}. For the bound to be useful, we need to be able to control the compressibility function $\compressibilityFunction_{\probDistribution}(k)$.
In the next section we instantiate the general bound presented in this section to classification and regression problems,
considering low-distortion compressions, i.e. random projections. Such low-distortion compressions permit bounding the compressibility term to yield generalisation guarantees for ensembles of \emph{any} size, even for a singleton. In contrast, coordinate projections are known to require a sufficiently large ensemble.

\section{Near-minimax optimality conditions for compressive classification and regression with convex losses}\label{sec:app}

In this
section we instantiate the general bound of the previous section (Theorem \ref{ThmMainResult}) in two fundamental learning problems. Throughout this section, $\X$ will be a separable Hilbert space of arbitrary dimension. We shall consider low-distortion Johnson-Lindenstrauss mappings which make it feasible to control the compressibility in terms of distributional conditions of geometric nature.

\begin{assumption}[Johnson-Lindenstrauss property]\label{JLAssumption} We shall say that $(\randomProjectionMeasure)_{k \in \N}$ satisfies the Johnson Lindenstrauss property with constant  $\johnsonLindenstraussConstant\geq 1$ if given any set $\{x_1,\cdots,x_{q}\}\subseteq \X$
of cardinality $q$, any $\epsilon \in (0,1)$, $\delta \in (0,1)$ and $k\geq \johnsonLindenstraussConstant \log(q/\delta)\cdot\epsilon^{-2}$ we have
{\small
\begin{align*}
\randomProjectionMeasure\left( \left\lbrace \randomProjection \in \setOfRandomProjections\hspace{1mm}:\hspace{1mm}\forall j, j' \in [q]\hspace{1mm} (1-\epsilon) \|x_j-x_{j'}\|^2 \leq \|\randomProjection(x_j)-\randomProjection(x_{j'})\|^2 \leq (1+\epsilon) \|x_j-x_{j'}\|^2  \right\rbrace \right) \geq 1- \delta.
\end{align*}
}
\end{assumption}

There are many examples of  Johnson-Lindenstrauss (JL) mappings, and the results of this section hold for any of these. In particular, all subgaussian linear maps satisfy JL \citep{Matousek} - a class that includes the Gaussian random projection \citep{DasguptaGupta}, as well as computation-friendly bit-flip based transforms \citep{Achlioptas}. For practical implementation, typically $\X=\R^d$ is taken, although a separable Hilbert space is sufficient in theory \citep[Theorem 3.1.]{Biau}.

Moreover, Assumption \ref{JLAssumption} is also satisfied by certain structured random matrices that enable efficient computation of the  compressive mapping, most notably the Fast Johnson-Lindenstrauss transform  \citep{ailon2006approximate}, and random matrices that exploit sparse matrix multiplications \citep{KaneNelson}. The quest for developing efficient Johnson-Lindenstrauss transforms is currently an active research area; some constructions require a slightly larger target dimension in exchange of greater savings in computation time \citep{FreksenLarsen}. However, the target dimension of order $\log(q)\epsilon^{-2}$ is known to be optimal in that any transform that satisfies Assumption \ref{JLAssumption} uniformly over any set of $q$ points must have a target dimension of this order \citep{LarsenNelson}.

\newcommand{\setOfMeasures}{\mathcal{P}}
\subsection{Compressive classification with a geometric margin condition}\label{sec:class}

Consider the classification setting discussed in Example \ref{e1}. 
Take $\actionSpace=\Y = \{-1,+1\}$ and the \emph{zero-one} loss function $\zeroOneLoss(v,y) = \one\left\lbrace v\neq y\right\rbrace$. Given a Borel probability distribution $\mathrm{P}$ on $\X \times \Y$ we let ${{\regressionFunction}}: \X \rightarrow [0,1]$ denote the regression function defined by ${{\regressionFunction}}(x):= \E_{(X,Y) \sim \mathrm{P}}\left[Y=1|X\right]$ and let $\mathrm{P}_X$ denote the marginal  distribution over $X$ where $(X,Y) \sim \mathrm{P}$. We let $\|\cdot\|$ denote the Euclidean norm on a Hilbert space $\X$. 
The low complexity function class of interest in this section is $\lowDimensionalFunctionClass=\{z\rightarrow \sign(w^\top z-t) : w\in\R^k, t\in \R\}\subseteq \measurableMaps_1(\R^k)$.

We introduce three distributional assumptions. These capture benign characteristics of the problem that allow a tight bound on the compressibility $\compressibilityFunction(k)$.  Our first assumption is a geometric margin condition.

\begin{assumption}[Geometric margin condition]\label{geometricMarginAssumption} 
We shall say that a distribution $\mathrm{P}$ on $\X \times \{0,1\}$ satisfies the geometric margin condition with exponent $\geometricMarginExponent >0$, constant $\geometricMarginConstant \geq 1$ and approximation error $\approxError \in [0,1]$, if $\X$ is a separable Hilbert space and there exists $(w_{\circ},t_{\circ}) \in \X \times \R$ with $\|w_{\circ}\|=1$ with the following properties:
\begin{enumerate}[label=(\roman*)]
\item The linear classifier ${\classifier^{\circ}}:\X \rightarrow \{-1,+1\}$ defined by ${\classifier^{\circ}}(x):=\sign(w_{\circ}^{\top}x -t_{\circ})$ for $x \in \X$, has excess error $\excessRisk({\classifier^{\circ}})\leq \approxError$.
\item Letting $\nearDecisionBoundary:=\{x \in \X\hspace{0.5mm}:\hspace{0.5mm} |w_{\circ}^{\top}x-t_{\circ}|\leq \xi\}$ we have $\int_{\nearDecisionBoundary} \left|2{{\regressionFunction}}(x)-1\right| d{\mathrm{P}}_X(x)  \leq\geometricMarginConstant \cdot \xi^{\geometricMarginExponent}$  for each $\xi>0$.
\end{enumerate}
\end{assumption}

The reason we call
Assumption \ref{geometricMarginAssumption} a geometric margin condition is that it involves the set of points $\nearDecisionBoundary$ near the $\classifier^{\circ}$-decision boundary $\decisionBoundary:=
\{x \in \X\hspace{0.5mm}:\hspace{0.5mm} w_{\circ}^{\top}x=t_{\circ}\}$. However, observe that Assumption \ref{geometricMarginAssumption}
differs in an essential way from previously considered geometric margin conditions in compressive classification, e.g. by \cite{Balcan, Arriaga} and others, which were inherited directly from traditional data-space classification theory. In particular, assumption \ref{geometricMarginAssumption} does not require the classes to be separable. 
Assumption \ref{geometricMarginAssumption} requires that there is a linear classifier $\classifier^{\circ}$ with low excess error $\approxError$, for which there is not too much mass close to its decision boundary $\decisionBoundary$. Moreover, this mass is weighted by $\left|2{{\regressionFunction}}(\cdot)-1\right|$ so that very little penalty is incurred for difficult to classify points near the decision boundary. Our next assumption controls the tails of our marginal distribution.
\begin{assumption}[Moment condition]\label{momentAssumption} We say that $\probDistribution$ satisfies the moment condition with exponent $\momentExponent \in (0,\infty)$ and constant $\momentConstant \geq 1$ if
\[\int_{\|x\|>s} \left|2{{\regressionFunction}}(x)-1\right|d \mathrm{P}_X(x)  \leq\momentConstant \cdot s^{-\momentExponent}\] for all $s >0$.
\end{assumption}

Assumption \ref{momentAssumption} is a significant relaxation of the assumption of bounded support often utilised within the classification literature. We refer to Assumption \ref{momentAssumption} as a moment condition since it holds whenever  $2\E_{X \sim \mathrm{P}_X}[\|X\|^{\momentExponent}]\leq \momentConstant$, by Markov's inequality.  Finally we shall make use of the classification form of the Tsybakov noise condition  
\citep{mammen1999}. This will ensure the Bernstein-Tsybakov condition (Assumption \ref{bernsteinTysbakovMarginTypeAssumption}) is satisfied. 

\newcommand{\tysbakovNoiseConstant}{C_\mathrm{T}}
\newcommand{\tysbakovNoiseEpsilon}{\epsilon_0}

\begin{assumption}[Noise condition]\label{tsybakovNoiseAssumption} We say that $\mathrm{P}$ satisfies the Tsybakov noise condition with exponent $\bernsteinExponent \in [0,1)$ and constant $\tysbakovNoiseConstant \geq 1$ if $ \mathrm{P}_X\left( x \in \X:\left|{2{\regressionFunction}}(x)-1\right|\leq \epsilon\right) \leq \tysbakovNoiseConstant \cdot \epsilon^{\frac{\bernsteinExponent}{1-\bernsteinExponent}}$ for all $\epsilon \in (0,1)$. 
\end{assumption}

Tysbakov et al. have shown that Assumption \ref{tsybakovNoiseAssumption} implies Assumption \ref{bernsteinTysbakovMarginTypeAssumption} with a suitable choice of $\bernsteinConstant$ depending upon $\bernsteinExponent$ and $\tysbakovNoiseConstant$ \cite[Proposition 1]{tsybakov2004optimal}. We now introduce a class of distributions with compressible structure.

\begin{defn}[Compressive classification measure class]\label{classificationMeasureClass} Given parameters\\ $\Gamma=(  (\geometricMarginExponent,\geometricMarginConstant ),(\momentExponent,\momentConstant),(\bernsteinExponent, \tysbakovNoiseConstant))$ where $\tysbakovNoiseConstant$, $\geometricMarginConstant$, $\momentConstant\geq 1$, $\bernsteinExponent \in [0,1)$, $\geometricMarginExponent$, $\momentExponent \in (0, \infty)$ we let $\setOfMeasures_{\text{0,1}}(\Gamma,\approxError)$ denote the set of all Borel probability $\mathrm{P}$ on $\X \times \Y$ which satisfy Assumption \ref{geometricMarginAssumption} with parameters $(\geometricMarginExponent,\geometricMarginConstant,\approxError)$, Assumption \ref{momentAssumption} with parameters $(\momentExponent,\momentConstant)$, and Assumption \ref{tsybakovNoiseAssumption}  with parameters $(\bernsteinExponent,\tysbakovNoiseConstant)$. 
\end{defn}

\newcommand{\constantForGeometricMarginThm}{C_{\ref{thm:geometricMarginThm}}}
\newcommand{\constantForGeometricMarginLemma}{\tilde{C}_{\ref{thm:geometricMarginThm}}}

\begin{theorem}[Compressive classification upper bound]\label{thm:geometricMarginThm} Let $\zeroOneLoss$ be the zero-one loss function and take distributional parameters $ \Gamma=(  (\geometricMarginExponent,\geometricMarginConstant ),(\momentExponent,\momentConstant),(\bernsteinExponent, \tysbakovNoiseConstant))$ where $\bernsteinConstant$, $\geometricMarginConstant$, $\momentConstant\geq 1$, $\bernsteinExponent \in [0,1)$, $\geometricMarginExponent$, $\momentExponent \in (0, \infty)$. Suppose further that our random projection measures $(\randomProjectionMeasure)_{k \in \N}$ satisfy the Johnson Lindenstrauss property with constant  $\johnsonLindenstraussConstant\geq 1$. There exists a constant $\constantForGeometricMarginThm \geq 1$, depending only on $\Gamma$ and $\johnsonLindenstraussConstant$, such that for any $n,k,m \in \N$, $\approxError\in[0,1]$, 
$\delta \in (0,1)$, and $\probDistribution \in
\setOfMeasures_{\text{0,1}}(\Gamma,\approxError)$, then the following holds with probability at least $1-\delta$ over $\sample \sim \probDistribution^{\otimes n}$ and $(A_i)_{i \in [m]} \sim \randomProjectionMeasure^{\otimes m}$,
\begin{align*}
\excessRiskArg_{\zeroOneLoss,\probDistribution}\left(\functionEstimate_{n,k,m}\right)\leq \constantForGeometricMarginThm \left\lbrace  \left(\frac{\logBar(k)}{k}\right)^{\frac{\geometricMarginExponent\momentExponent}{2(\geometricMarginExponent+\momentExponent)}}+   \left(\frac{k \cdot \logBar(n)+\logBar(1/\delta)}{ n} \right)^{\frac{1}{2-\bernsteinExponent}}+
{\frac{\log(1/\delta)}{ m}} +{\approxError}
\right\rbrace.
\end{align*}
Moreover, if $k=\lceil \left(n/\logBar(n)\right)^{\frac{2(\geometricMarginExponent+\momentExponent)}{2(\geometricMarginExponent+\momentExponent)+\geometricMarginExponent\momentExponent(2-\bernsteinExponent)}} \rceil$ then with probability at least $1-\delta$ we have,
\begin{align*}
\excessRiskArg_{\zeroOneLoss,\probDistribution}\left(\functionEstimate_{n,k,m}\right)\leq \constantForGeometricMarginThm \left\lbrace  \left(\frac{\logBar(n)}{n}\right)^{\frac{\geometricMarginExponent\momentExponent}{2(\geometricMarginExponent+\momentExponent)+\geometricMarginExponent\momentExponent(2-\bernsteinExponent)}}+   \left(\frac{\logBar(1/\delta)}{ n} \right)^{\frac{1}{2-\bernsteinExponent}}+
{\frac{\log(1/\delta)}{ m}} +{\approxError}
\right\rbrace. 
\end{align*}
\end{theorem}

Theorem \ref{thm:geometricMarginThm} 
highlights the dependence of the worst case excess risk on the geometric exponents $\geometricMarginExponent$ and $\momentExponent$ in combination with the statistical Bernstein-Tsybakov exponent $\bernsteinExponent$. Theorem \ref{clRLBtext} below implies that the rate is minimax optimal up to logaritimic factors.

A direct comparison with the rate previously obtained for compressive learning by \cite{Chen} would be difficult to make, as their assumptions are very different in flavour from ours, and they consider regularised models in the reduced space whereas our result highlights the regularisation effect of the compressive approach itself. 
However, whenever $\geometricMarginExponent \momentExponent > \geometricMarginExponent + \momentExponent$, 
the rate in Theorem \ref{thm:geometricMarginThm} is always (i.e. even if $\bernsteinExponent=0$) 
no worse than the rate $n^{-1/4}$ found previously in the analysis of \cite{Chen}, which did not exploit geometric properties of the distribution. 

Theorem \ref{thm:geometricMarginThm} is a consequence of Theorem \ref{ThmMainResult} combined with 
the following proposition.

\begin{prop}[Compressiblity of linear classification]\label{prop:compressibilityTermForGeometricMarginThm}  Let $\zeroOneLoss$ be the zero-one loss function and take distributional parameters $ \Gamma=(  (\geometricMarginExponent,\geometricMarginConstant ),(\momentExponent,\momentConstant),(\bernsteinExponent, \tysbakovNoiseConstant))$ where $\bernsteinConstant$, $\geometricMarginConstant$, $\momentConstant\geq 1$, $\bernsteinExponent \in [0,1]$, $\geometricMarginExponent$, $\momentExponent \in (0, \infty)$. Suppose further that our random projection measures $(\randomProjectionMeasure)_{k \in \N}$ satisfy the Johnson Lindenstrauss property with constant  $\johnsonLindenstraussConstant\geq 1$. There exists a constant $\constantForGeometricMarginLemma \ge 1$, depending only on $\Gamma$, $\johnsonLindenstraussConstant$, such that for all $k\in \N$, and all $\mathrm{P}\in \setOfMeasures_{\text{0,1}}(\Gamma,\approxError)$,
\begin{align}
 \compressibilityFunction_{\mathrm{P}}(k)\leq \constantForGeometricMarginLemma  \left(\frac{\logBar(k)}{k}\right)^{\frac{\geometricMarginExponent \momentExponent}{2(\geometricMarginExponent+\momentExponent)}}+\approxError.
\end{align}
\end{prop}

To prove Proposition \ref{prop:compressibilityTermForGeometricMarginThm} we begin with the following consequence of Assumption \ref{JLAssumption}.

\begin{lemma}\label{lemma:ClassificationPreservationForSinglePointAwayFromMargin} Suppose that Assumption \ref{JLAssumption} holds. Take $k \in \N$, $w_\circ \in \X$ with $\|w_\circ\|=1$, $t_\circ\in \R$, $\xi \in (0,1]$ and $s \in [1,\infty)$. Given $x\in \X$ with $|w_\circ^{\top}x-t_\circ| \geq \xi$ and $\|x\|_2 \leq s$ we have,
\begin{align}\label{eq:ProbClassificationPreservationAwayFromMargin}
\randomProjectionMeasure\left(\left\lbrace \randomProjection \in \setOfRandomProjections: \, \sign( \randomProjection(w_\circ)^{\top}\randomProjection(x)-t_\circ) \neq \sign( w_\circ^{\top}x-t_\circ) \right\rbrace\right) \leq 3e^{-\xi^2k/(4\johnsonLindenstraussConstant s^2)}.
\end{align}
\end{lemma}
\begin{proof} Take $\epsilon:=\xi/(2s)$ and define an event  ${E}_{x,k} \subseteq \setOfRandomProjections$ by
{\small
\begin{align*}
{E}_{x,k}:= &\bigcap_{z \in \{ \pm {x}/\|x\|\}}\left\lbrace \randomProjection \in \setOfRandomProjections\hspace{1mm}:\hspace{1mm}(1-\epsilon) \cdot \left\| w_\circ-z\right\|^2 \leq \left\| \randomProjection(w_\circ)-\randomProjection(z)\right\| \leq (1+\epsilon) \cdot \left\|w_\circ- z\right\|^2  \right\rbrace.
\end{align*}}
Take $A \in {E}_{x,k}$ and suppose that $w_\circ^{\top}x-t_\circ >0$. Taking $\tilde{x}=x/\|x\|_2$,
\begin{align*}
\randomProjection(w_{\circ})^{\top}\randomProjection(x) &= \frac{\|x\|}{4} \left(\|\randomProjection(w_{\circ})+\randomProjection\left(\tilde{x}\right)\|_2^2-\|\randomProjection(w_{\circ})-\randomProjection(\tilde{x})\|_2^2\right)\\
&\geq \frac{\|x\|}{4} \left((1-\epsilon)\|w_{\circ}+\tilde{x}\|_2^2-(1+\epsilon)\|w_{\circ}-\tilde{x}\|_2^2\right)\\
& \geq w_\circ^\top x -\epsilon \cdot \|x\| > t_\circ+ \xi -\epsilon \cdot s \geq t_\circ,
\end{align*}
so $\sign( \randomProjection(w_\circ)^{\top}\randomProjection(x)-t_\circ) = \sign( w_\circ^{\top}x-t_\circ)$. By symmetry we also have  $\sign( \randomProjection(w_\circ)^{\top}\randomProjection(x)-t_\circ) = \sign( w_\circ^{\top}x-t_\circ)$ when $A \in {E}_{x,k}$ and $w_\circ^{\top}x-t_\circ <0$. Moreover, by Assumption \ref{JLAssumption} we have $\randomProjectionMeasure({E}_{x,k}) \geq 1-3e^{-\epsilon^2k/\johnsonLindenstraussConstant}$ so result follows.

\end{proof}

Next we deduce Proposition \ref{prop:compressibilityTermForGeometricMarginThm}  from Lemma \ref{lemma:ClassificationPreservationForSinglePointAwayFromMargin}.

\begin{proof}[Proof of Proposition \ref{prop:compressibilityTermForGeometricMarginThm}] Take $(w_{\circ},t_{\circ}) \in \X \times \R$ with $\|w_{\circ}\|=1$ with the properties guaranteed by Assumption \ref{geometricMarginAssumption} and let ${\classifier_{\circ}}(x):=\sign(w_{\circ}^{\top}x -t_{\circ})$ be the associated linear classifier. Take $\xi \in (0,1]$ and $s  \in [1,\infty)$ and let $\mathcal{U}:= \bigl\{ x \in  \overline{B_s(\bm{0})} \backslash \nearDecisionBoundary :\phi_\circ(x) = \oracleFunction(x) \bigr\}$. By Assumptions \ref{geometricMarginAssumption} and \ref{momentAssumption}, for any $\phi \in \measurableMaps(\X,\actionSpace)$, we have
\begin{align*}
\excessRiskArg_{\zeroOneLoss,\probDistribution}\left( \phi\right) &\leq  \int_{\mathcal{U}} |2\eta(x)-1| \cdot \one\left\lbrace \phi(x) \neq {\classifier_{\circ}}(x)\right\rbrace d\probDistribution_X(x)+\int_{\X \backslash \mathcal{U}}|2\eta(x)-1|\probDistribution_X(x)\\
&\leq \probDistribution_X(\{ x \in \mathcal{U}:\phi(x) \neq {\classifier_{\circ}}(x)\})+\geometricMarginConstant \cdot \xi^{\geometricMarginExponent}+\momentConstant \cdot s^{-\momentExponent}+\approxError.
\end{align*}
Next for each $\randomProjection \in \setOfRandomProjections$ let $f_\randomProjection \in \lowDimensionalFunctionClass$ denote the map $f_\randomProjection(z) = \sign(A(w_\circ)^\top z-t_\circ)$. Hence, by Lemma \ref{lemma:ClassificationPreservationForSinglePointAwayFromMargin} for each $x \in \mathcal{U}$ we have $\randomProjectionMeasure(\{\randomProjection \in \setOfRandomProjections: f_\randomProjection\circ \randomProjection(x) \neq {\classifier_{\circ}}(x)\})\leq  3e^{-\xi^2k/(4\johnsonLindenstraussConstant s^2)}$. Thus, by Fubini's theorem we have
\begin{align*}
\compressibilityFunction(k)&\leq \int_{\setOfRandomProjections}  \excessRiskArg_{\zeroOneLoss,\probDistribution}\left(f_\randomProjection \circ \randomProjection\right) d\randomProjectionMeasure(\randomProjection)\\
&\leq \int_{\setOfRandomProjections} \probDistribution_X(\{ x \in \mathcal{U}:f_\randomProjection\circ \randomProjection(x) \neq {\classifier_{\circ}}(x)\})d\randomProjectionMeasure(A)+\geometricMarginConstant \cdot \xi^{\geometricMarginExponent}+\momentConstant \cdot s^{-\momentExponent}+\approxError\\
&\leq 3e^{-\xi^2k/(4\johnsonLindenstraussConstant s^2)}+\geometricMarginConstant \cdot \xi^{\geometricMarginExponent}+\momentConstant \cdot s^{-\momentExponent}+\approxError
\end{align*}
To complete the proof we take $\xi= 1 \wedge \left(4\johnsonLindenstraussConstant\log(3k)/k\right)^{\frac{\momentExponent}{2(\geometricMarginExponent+\momentExponent)}}$ and $s = \xi^{-\geometricMarginExponent/\momentExponent}$ to yield the required bound.
\end{proof}

Our next result (Theorem \ref{clRLBtext}) is a minimax lower bound for the class of distributions introduced in Definition \ref{classificationMeasureClass}. In conjunction with Theorem \ref{thm:geometricMarginThm}, this result shows that compressive ensembles achieve the minimax optimal rate, up to logarithmic factors.

\begin{theorem}[Minimax lower bound] \label{clRLBtext}
Let $\zeroOneLoss$ be the zero-one loss function and take $ \Gamma=(  (\bernsteinExponent, \tysbakovNoiseConstant),
(\geometricMarginConstant, \geometricMarginExponent),(\momentConstant,\momentExponent))$ where $\bernsteinConstant$, $\momentConstant\geq 1$, $\geometricMarginConstant \ge  {2}^{\geometricMarginExponent/2}$, $\bernsteinExponent \in [0,1)$, $\geometricMarginExponent$, $\momentExponent \in (0, \infty)$. Given $n \in \N$, let $\X$ be a Hilbert space of dimension at least $n$, and take $\approxError \in [0,1]$. There exists a constant $c_{\ref{clRLBtext}}>0$ depending only upon $\Gamma$, such that for any empirical classifier  $\hat{\phi}:(\X\times\Y)^n\times \X\rightarrow \Y$, there exists a distribution $\mathrm{P} \in \setOfMeasures_{\text{0,1}}(\Gamma,\approxError)$ with
\begin{align}\label{eq:mainClaimClassificationLB}
\E_{\sampleXY\sim \probDistribution^{\otimes n}}[\excessRiskArg_{\zeroOneLoss,\probDistribution}(\hat{\phi})]
\ge c_{\ref{clRLBtext}} \cdot \left\{
n^{-\frac{\momentExponent\geometricMarginExponent}{2(\momentExponent+\geometricMarginExponent)+\momentExponent\geometricMarginExponent(2-\bernsteinExponent)}} +\approxError \right\}.
\end{align}
\end{theorem}
The proof of Theorem \ref{clRLBtext} is given in Section \ref{Proofs4LB}. The proof 
involves the construction of a family of distributions within the class which are, simultaneously, sufficiently close that they are hard to tell apart based on a sample of size $n$, and sufficiently far apart that failing to do so must incur a large error.

\subsection{Compressive regression with a fast decaying covariance spectrum}\label{sec:reg}
\newcommand{\constantForOLSThm}{C_{\ref{thm:OLSThm}}}
\newcommand{\exponentForOLSThm}{c_{\ref{thm:OLSThm}}}
\newcommand{\constantForOLSLemma}{\tilde{C}_{\ref{thm:OLSThm}}}
\newcommand{\exponentForOLSLemma}{\tilde{c}_{\ref{thm:OLSThm}}}
\newcommand{\weightConstant}{W_{max}}
\newcommand{\spectralConstant}{C_{sp}}
\newcommand{\tspectralConstant}{\tilde{C}_{sp}}
\newcommand{\spectralExponent}{\omega}
\newcommand{\XX}{\mathbb{X}}
\newcommand{\strongConvexityConstant}{\mathrm{H}}
\newcommand{\smoothnessConstant}{\mathrm{H_U}}
\newcommand{\goodX}{(\pgood)_X}
\newcommand{\badX}{(\pbad)_X}
\newcommand{\Wmax}{W_{\max}}

We now turn our attention to strongly convex losses such as the squared loss and logistic loss discussed earlier in the Examples \ref{e2}-\ref{e3}. Throughout the section we take $\X$ to be a separable Hilbert space, $\Y=\R$, $\actionSpace=[-\hypothesisClassBound,\hypothesisClassBound]$, and consider the bounded linear function class $\lowDimensionalFunctionClass\equiv \lowDimensionalFunctionClass^{\text{bl}}=\left\lbrace \bm{u} \mapsto \max\left( \min\left( \bm{w} \cdot \bm{u}-t,\hypothesisClassBound\right),-\hypothesisClassBound\right):\hspace{2mm}\bm{w}\in \R^k,\hspace{1mm}t \in \R\right\rbrace \subseteq \measurableMaps_{\hypothesisClassBound}\left(\R^k\right)$. A function $\varphi:[a,b]\rightarrow \R$ is said to be $\strongConvexityConstant$-strongly convex on $[a,b]$ if 
\begin{align}\label{eq:strongConvexDef}
\varphi((1-t)v_0+tv_1) \leq (1-t) \varphi(v_0)+t\varphi(v_1) - \frac{\strongConvexityConstant}{2}\cdot t(1-t)\cdot (v_0-v_1)^2,
\end{align}
for all $v_0$, $v_1 \in [a,b]$ and $t \in [0,1]$. We say that $\varphi:[a,b]\rightarrow \R$ is  $\strongConvexityConstant$-strongly mid-point convex on $[a,b]$ if \eqref{eq:strongConvexDef} holds for all $v_0$, $v_1 \in [a,b]$ and $t =1/2$. Note that any twice differentiable function $\varphi$ with strictly positive second derivative $\varphi'' \geq \strongConvexityConstant$ is $\strongConvexityConstant$-strongly convex (Lemma \ref{lemma:striclyPositiveSecondDerivativeImpliesStronglyConvex}), and hence $\strongConvexityConstant$-strongly mid-point convex.

\begin{assumption}[Strongly Convex loss]\label{stronglyConvexLossAssumption} We shall say that the loss function $\lossFunction: [-\hypothesisClassBound,\hypothesisClassBound]\times \Y \rightarrow [0,\lossFunctionBound]$ is strongly mid-point convex with constant $\strongConvexityConstant >0$ if for all $y \in \Y$, the function $v \mapsto \lossFunction(v,y)$ is $\strongConvexityConstant$-strongly mid-point convex on $[-\hypothesisClassBound,\hypothesisClassBound]$.
\end{assumption}

Lemma \ref{lemma:BernsteinConditionForLipschitzStronglyConvexLosses} shows that any loss function satisfying Assumption \ref{stronglyConvexLossAssumption} also satisfies the Bernstein-Tsybakov condition with exponent $\bernsteinExponent =1$, and constant $\bernsteinConstant = {4\lipschitzConstantLoss^2}/{\strongConvexityConstant}$. We shall introduce two further distributional assumptions. Firstly, Assumption \ref{ass:linearAproximationCondition} concerns the existence of a linear predictor of bounded norm with low excess error. Secondly, Assumption \ref{spectraldecay} requires that the eigenvalues of the underlying covariance matrix decay at an exponential rate.

\begin{assumption}[Linear approximation condition]\label{ass:linearAproximationCondition} We shall say that $\probDistribution$ on $\X \times \Y$ satisfies the linear approximation condition with loss $\lossFunction:[-\hypothesisClassBound,\hypothesisClassBound] \times \Y \rightarrow [0,\lossFunctionBound]$ and constants $\Wmax \geq 1$, $\approxError \in [0,1]$ if $\X$ is a Hilbert space and there exists $(w_\circ,t_\circ) \in \X \times \R$ with and $\|w_\circ\|_2 \leq \Wmax$ such that the predictor $\phi_\circ:\X\rightarrow \R$ defined by $\phi_\circ(x):=\min\{\hypothesisClassBound,\max\{-\hypothesisClassBound,w_\circ^\top x+t_\circ\}\}$ has excess risk $\excessRisk(\phi_\circ)\leq \approxError$.
\end{assumption}

\newcommand{\spectralDecay}{\omega}

Furthermore, we assume a fast decay on the eigen-spectrum of the covariance operator of the marginal distribution.

\begin{assumption}[Spectral decay condition]\label{spectraldecay}
We say that $\probDistribution$ satisfies the spectral decay condition with constant $\spectralConstant\geq 1$, decay $\spectralDecay\in (0,1)$ if
$\probDistribution_X$ has covariance operator $\Upsilon$  with singular values
$\lambda_r(\Upsilon) \le \spectralConstant \cdot 
\spectralDecay^r$ for all $r\in\N$.
\end{assumption}

We define the following class of distributions.

\begin{defn}[Compressive regression measure class]\label{regressionMeasureClass} Given a bounded loss function  $\lossFunction: [-\hypothesisClassBound,\hypothesisClassBound]\times \Y \rightarrow [0,\lossFunctionBound]$ and parameters $ \Gamma=(\Wmax,\spectralConstant,\spectralDecay) \in [1,\infty) \times \left([1,\infty)\times (0,1)\right)$ we let $\setOfMeasures_{\lossFunction}(\Gamma,\approxError)$ denote the set of all Borel probability distributions $\probDistribution$ on $\X \times \Y$, for which  Assumption \ref{ass:linearAproximationCondition} is satisfied with loss function $\lossFunction$ and parameters $(\Wmax,\approxError)$ and Assumption \ref{spectraldecay} is satisfied with parameters $(\spectralConstant,\spectralDecay)$.
\end{defn}

\begin{theorem}[Compressive regression upper bound]\label{thm:OLSThm}
Let $\lossFunction: [-\hypothesisClassBound,\hypothesisClassBound]\times \Y \rightarrow [0,\lossFunctionBound]$ be a loss function satisfying Assumptions \ref{boundedLossFunctionAssumption}, \ref{lipschitzLossFunctionAssumption}, \ref{stronglyConvexLossAssumption} with parameters $\Omega:=(\hypothesisClassBound,\lossFunctionBound,\lipschitzConstantLoss,\strongConvexityConstant) \in [1,\infty)^2 \times (0,\infty)$ and take distributional parameters 
$\Gamma=(\Wmax,\spectralConstant,\spectralDecay) \in [1,\infty)^2\times (0,1)$ and $\approxError \in [0,1]$. Suppose further that $(\randomProjectionMeasure)_{k \in \N}$ satisfy the Johnson Lindenstrauss property with constant  $\johnsonLindenstraussConstant\geq 1$. There exist constants $\constantForOLSThm \ge 1$, $\exponentForOLSThm \in (0,1)$
depending only on $\Gamma$, $\Omega$ and $\johnsonLindenstraussConstant$, such that for any $n,k,m \in \N$, $\approxError \in [0,1]$ 
and $\delta \in (0,1)$, for any probability distribution $\mathrm{\probDistribution} \in {\mathcal P}_{\lossFunction}(\Gamma,\approxError)$, 
with probability at least $1-\delta$ 
\begin{align*}
\excessRiskArg_{\lossFunction,\mathrm{\probDistribution}}\left(\functionEstimate_{n,k,m}\right)\leq \constantForOLSThm \left\lbrace \exp(-\exponentForOLSThm \cdot k)+   
\frac{k \cdot \logBar(n)+\logBar(1/\delta)}{ n} 
+{\frac{\log(1/\delta)}{ m}} +\approxError
\right\rbrace.
\end{align*}
Moreover, with $k=\lceil \logBar(n)\rceil$ with probability at least $1-\delta$ we have
\begin{align}
\excessRiskArg_{\lossFunction,\mathrm{\probDistribution}}\left(\functionEstimate_{n,k,m}\right)\leq \constantForOLSThm \left\lbrace  
\frac{\logBar^2(n)+\logBar(1/\delta)}{ n}
+{\frac{\log(1/\delta)}{ m}} +\approxError
\right\rbrace.  \label{regUB}
\end{align}
\end{theorem}

Theorem \ref{thm:OLSThm} complements recent work by \cite{Slawski}, which gives an expectation bound for compressive ordinary least squares (OLS) regression in the fixed design setting and highlights a parallel with Principal Component Regression (PCR) through the role of spectral decay. 
Our result shows that a similar effect holds in the in random design setting, for a larger class of loss functions, and our guarantees hold with high probability rather than just in expectation w.r.t. the training sample. {We should mention that guarantees that do not require a spectral decay condition are known for a form of compressive ridge-regularised kernel regression \citep{yang2017}. However, a comparison would be difficult as the setting is very different. In particular, the goal in that work was to preserve the optimality of a ridge-regularised model when compressing the kernel for computational speedup, whereas our analysis is aimed to bring out the regularisation effect of a compressive feature representation itself, highlighting its ability to transform an arbitrary high dimensional problem that is not learnable from a small sample into a low dimensional problem that is nearly optimally learnable.} Theorem \ref{thm:OLSThm} is a consequence of Theorem \ref{ThmMainResult} combined with Lemma \ref{lemma:BernsteinConditionForLipschitzStronglyConvexLosses} and Proposition \ref{prop:compressibilityTermForOLSThm}. 

\begin{prop}[Compressibility of generalised linear regression]\label{prop:compressibilityTermForOLSThm} Take $\hypothesisClassBound$, $\lipschitzConstantLoss \in [1,\infty)$, $\johnsonLindenstraussConstant \in [1,\infty)$ and $\Gamma=(\Wmax,\spectralConstant,\spectralDecay) \in [1,\infty)^2\times (0,1)$. Let $\lossFunction: [-\hypothesisClassBound,\hypothesisClassBound]\times \Y \rightarrow [0,\infty)$ be a loss function satisfying Assumption \ref{lipschitzLossFunctionAssumption} with parameter $\lipschitzConstantLoss$. Suppose that $(\randomProjectionMeasure)_{k \in \N}$ satisfy the Johnson Lindenstrauss property with constant  $\johnsonLindenstraussConstant\geq 1$. 
There exist constants $\constantForOLSLemma,\exponentForOLSLemma>0$  depending only on $\hypothesisClassBound$, $\lipschitzConstantLoss$, $\Gamma$ and $\johnsonLindenstraussConstant$ such that for all $k \in \N$, $\approxError \in [0,1]$, and all $\probDistribution\in \setOfMeasures_{\lossFunction}(\Gamma,\approxError)$,
\begin{align}
 \compressibilityFunction_\probDistribution(k)\leq 
\constantForOLSLemma  \cdot  \exp(-\exponentForOLSLemma\cdot k)+\approxError.
\end{align}
\end{prop}

To prove Proposition \ref{prop:compressibilityTermForOLSThm} we use the following consequence of \cite[Theorem 1]{Slawski}. For completeness, we give a short alternative proof of Lemma \ref{LSlemma} in Appendix \ref{A:Slawski}.

\newcommand{\tU}{\tilde{U}}
\newcommand{\tL}{\tilde{\Lambda}}
\newcommand{\johnsonLindenstraussConstantSlawski}{c_{S}}

\begin{lemma}[Excess risk of compressive OLS with fixed design]\label{LSlemma} Given $q$, $d$, $k$ $\in \N$, take a vector $w_{\diamond} \in \R^d$, a $d \times q$ matrix $\XX$ and let $A$ a random $k \times d$ matrix satisfying Assumption \ref{JLAssumption} with constant $\johnsonLindenstraussConstant$. There exists a constant $\johnsonLindenstraussConstantSlawski \in (0,1)$, that depends only on $\johnsonLindenstraussConstant$, such that for any $r <\min\{q,k\}$ the following holds with probability at least $1-(24^r+2q)  e^{-\johnsonLindenstraussConstantSlawski k}$,
\begin{align}\label{eq:conclusionOfSlawskiTypeLemma}
\inf_{w\in \R^k}\|w^\top\randomProjection \XX - w_{\diamond}^\top \XX\|_2^2
\le 18\|w_{\diamond}\|_2^2~\sum_{j \ge r+1} \lambda_j(\XX\XX^\top).
\end{align}
\end{lemma}

\begin{proof}[Proof of Proposition \ref{prop:compressibilityTermForOLSThm}] Fix $q$, $r$ $\in \N$, to be specified later. For each $\randomProjection \in \setOfRandomProjections$ define $\Delta_\randomProjection:\X^q \rightarrow \R$ by
\begin{align*}
\Delta_\randomProjection(\bm{x}_{1:q}):=\sup_{f \in \lowDimensionalFunctionClass} \left\lbrace  \int_{\X}(f\circ \randomProjection(x)-\phi_{\circ}(x))^2d\probDistribution_X(x) -\frac{1}{q}\sum_{\ell=1}^q (f\circ \randomProjection(x_\ell)-\phi_{\circ}(x_\ell))^2 \right\rbrace
\end{align*}
for $\bm{x}_{1:q}=(x_\ell)_{\ell \in [q]} \in \X^q$. Define subsets $E_{\mathrm{rd}}$ and $ E_{\mathrm{sp}}\subseteq \X^q$ by
\begin{align*}
E_{\mathrm{rd}}&:=\left\lbrace \bm{x}_{1:q} \in \X^q~:~
\int_{\setOfRandomProjections}\Delta_\randomProjection(\bm{x}_{1:q})d\randomProjectionMeasure(\randomProjection) \leq 48\hypothesisClassBound^2 \sqrt{\frac{\coveringNumberConstant \cdot k}{q}} \cdot  \logBar^{\frac{1}{2}} \left(\frac{2 \hypothesisClassBound^2 q}{ k  }\right) \right\rbrace\\
E_{\mathrm{sp}}&:=\left\lbrace \bm{x}_{1:q} \in \X^q~:~
\sum_{j\ge r+1}\lambda_j\left(\frac{1}{q}\sum_{\ell= 1}^q x_\ell x_\ell^\top\right) \leq  3\sum_{j\ge r+1}\lambda_j\left(\Sigma \right)\right\rbrace.
\end{align*}
We shall use the probabilistic method to show that $E_{\mathrm{rd}} \cap E_{\mathrm{sp}} \neq \emptyset$. By Lemma \ref{logCoverNumbersImplySmallLocalRademacherBoundLemma} with $r=(2\hypothesisClassBound)^4$, for each $\randomProjection \in \setOfRandomProjections$ we have
\begin{align}\label{eq:rademacherBoundInCompressiveRegPf}
\rademacherComplexityEmpirical\left( \left\lbrace x\mapsto  (f\circ \randomProjection(x)-\phi_{\circ}(x))^2 \right\rbrace_{f \in \lowDimensionalFunctionClass}, \bm{x}_{1:q} \right) \leq 2^3\hypothesisClassBound^2 \sqrt{\frac{\coveringNumberConstant \cdot k}{q}} \cdot  \logBar^{\frac{1}{2}} \left(\frac{ q}{ k  }\right).
\end{align}
Hence, by symmetrization (eg. \cite[Chapter 3]{mohri2012foundations}) for each $\randomProjection \in \setOfRandomProjections$ we have
\begin{align*}
\int_{\X^q} \Delta_\randomProjection(\bm{x}_{1:q}) d(\probDistribution_X)^q(\bm{x}_{1:n}) & \leq 2 \int_{\X^q} \rademacherComplexityEmpirical\left( \left\lbrace x\mapsto  (f\circ \randomProjection(x)-\phi_{\circ}(x))^2 \right\rbrace_{f \in \lowDimensionalFunctionClass}, \bm{x}_{1:q} \right) d(\probDistribution_X)^q(\bm{x}_{1:n}).
\end{align*}
By applying Fubini's theorem and combining with \eqref{eq:rademacherBoundInCompressiveRegPf} we have
\begin{align*}
\int_{\X^q} \left( \int_{\setOfRandomProjections} \Delta_\randomProjection(\bm{x}_{1:q}) d\randomProjectionMeasure(\randomProjection)\right) d(\probDistribution_X)^q(\bm{x}_{1:n})  \leq 2^4\hypothesisClassBound^2\sqrt{\frac{\coveringNumberConstant \cdot k}{q}} \cdot  \logBar^{\frac{1}{2}} \left(\frac{q}{ k }\right).
\end{align*}
Hence, by Markov's inequality we have $(\probDistribution_X)^q(E_{\mathrm{rd}} ) \geq 2/3$. Similarly, by Markov's inequality, we also have $(\probDistribution_X)^q(E_{\mathrm{sp}} ) \geq 2/3$. Combining these two bounds we conclude that  $E_{\mathrm{rd}} \cap E_{\mathrm{sp}} \neq \emptyset$. For the remainder of the proof we fix $\bm{x}_{1:q}= (x_\ell)_{\ell \in [q]} \in E_{\mathrm{rd}} \cap E_{\mathrm{sp}}$. Let $\Phi: \X \rightarrow \R^d$ be a linear map which is isometric on the $d$-dimensional linear subspace spanned by $\{x_1,\ldots,x_q, w_\circ\}$ and let $\XX \in \R^{d\times q}$ be the matrix with columns $\{q^{-1/2}\cdot \Phi(x_1),\ldots,q^{-1/2}\cdot \Phi(x_q)\}$. It follows that given a random projection ${\randomProjection} :\X \rightarrow \R^k$ with ${\randomProjection} \sim \randomProjectionMeasure$ satisfying Assumption \ref{JLAssumption} with constant $\johnsonLindenstraussConstant$, the induced random projection $\bm{\randomProjection} \Phi^{\top}: \R^d \rightarrow \R^k$ also satisfies Assumption \ref{JLAssumption} with constant $\johnsonLindenstraussConstant$. Hence, noting that $z \mapsto \min\{\hypothesisClassBound,\max\{-\hypothesisClassBound,z\}\}$ is $1$-Lipschitz and applying Lemma \ref{LSlemma} we see that with probability at least $1-(24^r+2q)  e^{-\johnsonLindenstraussConstantSlawski k}$ over $A \sim \randomProjectionMeasure$ we have
\begin{align*}
\inf_{f \in \lowDimensionalFunctionClass}\frac{1}{q}\sum_{\ell=1}^q (f\circ \randomProjection(x_\ell)-\phi_{\circ}(x_\ell))^2 & \leq \inf_{w \in \R^k,t \in \R}\frac{1}{q}\sum_{\ell=1}^q \left\lbrace  (w^\top \randomProjection x_\ell+t)- ( w_{\circ}^\top x_\ell+t_\circ)\right\rbrace^2\\
& \leq \inf_{w \in \R^k}\frac{1}{q}\sum_{\ell=1}^q \left(w^\top \randomProjection x_\ell- w_{\circ}^\top x_\ell\right)^2\\
& =\inf_{w \in \R^k}\frac{1}{q}\sum_{\ell=1}^q \left(w^\top \randomProjection \Phi^\top \Phi(x_\ell)- \Phi(w_{\circ})^\top \Phi(x_\ell)\right)^2\\
& = \inf_{w \in \R^k}\left\| w^\top \randomProjection \Phi^\top \XX- \Phi(w_{\circ})^\top \XX\right\|^2\\
&\leq 18\|\Phi(w_{\circ})\|^2 \sum_{j\ge r+1}\lambda_j\left(\XX\XX^\top\right)\\
&= 18\|w_{\circ}\|^2 \sum_{j\ge r+1}\lambda_j\left(\frac{1}{q}\sum_{\ell =1}^q x_\ell x_\ell^\top\right)\\
&\leq 54\Wmax^2 \sum_{j\ge r+1}\lambda_j\left(\Sigma\right),
\end{align*}
where we have used $\bm{x}_{1:q}\in  E_{\mathrm{sp}}$ in the final inequality. Hence, by combining with $\bm{x}_{1:q} \in E_{\mathrm{rd}}$ we deduce that 
\begin{align*}
&\int_{\setOfRandomProjections}   \inf_{f \in \lowDimensionalFunctionClass} \left\lbrace \int_{\X}\left(f\circ \randomProjection(x)-\phi_{\circ}(x)\right)^2d\probDistribution_X(x)\right\rbrace^{1/2} d\randomProjectionMeasure(\randomProjection)  \\ &\leq \int_{\setOfRandomProjections}   \left(   \inf_{f \in \lowDimensionalFunctionClass} \left\lbrace \frac{1}{q}\sum_{\ell=1}^q (f\circ \randomProjection(x_\ell)-\phi_{\circ}(x_\ell))^2\right\rbrace +\Delta_\randomProjection(\bm{x}_{1:q})\right)^{1/2} d\randomProjectionMeasure(\randomProjection) \nonumber\\
&\leq \left\{ 54\Wmax^2 \sum_{j\ge r+1}\lambda_j\left(\Sigma\right)+(2\hypothesisClassBound)^2\{24^r+2q\}  e^{-\johnsonLindenstraussConstantSlawski k}+48\hypothesisClassBound^2 \sqrt{\frac{\coveringNumberConstant \cdot k}{q}} \cdot  \logBar^{\frac{1}{2}} \left(\frac{q}{ k  }\right) \right\}^{1/2}.
\end{align*}
Taking $q=\lceil e^{\johnsonLindenstraussConstantSlawski k/2}\rceil$ and $r:=\lceil \johnsonLindenstraussConstantSlawski k/(2\log 24)\rceil$ we see that there exists $\constantForOLSLemma \geq 1$
and $\exponentForOLSLemma \in (0,1)$, both depending only upon  $\hypothesisClassBound$, $\lipschitzConstantLoss$, $\Gamma$ and $\johnsonLindenstraussConstant$ such that 
\begin{align}
\lipschitzConstantLoss
\int_{\setOfRandomProjections}   \inf_{f \in \lowDimensionalFunctionClass} \left\lbrace \int_{\X}\left(f\circ \randomProjection(x)-\phi_{\circ}(x)\right)^2d\probDistribution_X(x)\right\rbrace^{1/2} d\randomProjectionMeasure(\randomProjection) \label{eq:boundingSquareDiffFromOptLinear} \leq {\constantForOLSLemma}\exp(-\exponentForOLSLemma k),
\end{align}
By the Lipschitz property of the loss function combined with Jensen's inequality we have
\begin{align*}
\excessRisk(\functionGeneral) & \leq
\int_{\X \times \Y}\bigl\{\lossFunction(\functionGeneral(x),y) - \lossFunction(\phi_\circ(x),y)d\probDistribution(x,y)
+ \excessRisk(\phi_\circ)\bigr\}\\ &\leq \lipschitzConstantLoss\cdot 
\left\{\int_\X \left(\functionGeneral(x)-{\phi_\circ}(x)\right)^2 d\probDistribution_X(x) 
\right\}^{1/2} +\theta.
\end{align*}
Hence, by applying the bound \eqref{eq:boundingSquareDiffFromOptLinear} we obtain $\compressibilityFunction_{\probDistribution}(k)\leq {\constantForOLSLemma}\exp(-\exponentForOLSLemma k)+\approxError$, as required.
\end{proof}

Finally, we show that the upper bound in Theorem \ref{thm:OLSThm} is minimax-optimal up to logarithmic factors (Theorem \ref{regRLBtext}) for non-trivial loss functions. More precisely, we require the following additional non-degeneracy condition.

\begin{assumption}[Non-degenerate loss functions]\label{ass:nonDegenerateLoss} We shall say that the loss function $\lossFunction: \actionSpace\times \Y \rightarrow [0,\lossFunctionBound]$ is $\nonDegeneracyConstant$-non-degenerate if there exists $y_0, y_1 \in \Y$ such that 
\begin{align}\label{eq:nonDegeneracyIneqAssumption}
\inf_{v \in \actionSpace}\left\lbrace \lossFunction(v,y_0)+\lossFunction(v,y_1) \right\rbrace \geq \inf_{v \in \actionSpace}\{\lossFunction(v,y_0)\}+ \inf_{v \in \actionSpace}\{\lossFunction(v,y_1)\}+2\nonDegeneracyConstant.
\end{align}
\end{assumption}

We emphasise that Assumption \ref{ass:nonDegenerateLoss} will always hold, for an appropriately chosen $\nonDegeneracyConstant>0$, on any interesting learning problems in our setting. Indeed, if Assumption \ref{ass:nonDegenerateLoss} is violated for some loss function $\lossFunction: \actionSpace\times \Y \rightarrow [0,\lossFunctionBound]$ where $\actionSpace$ is compact and $\lossFunction$ satisfies Assumptions \ref{lipschitzLossFunctionAssumption} and \ref{stronglyConvexLossAssumption} then any function satisfying $\phi(x)\in \bigcap_{y \in \Y}\arginf_{v \in \actionSpace}\lossFunction(v,y)$ will be simultaneously Bayes optimal for all distributions $\probDistribution$ on $\X\times \Y$. Note in particular that the squared loss $\sqrLoss:[-\hypothesisClassBound,\hypothesisClassBound]^2 \rightarrow [0, (2\hypothesisClassBound)^2]$ satisfies Assumption \ref{ass:nonDegenerateLoss} with $\nonDegeneracyConstant=\hypothesisClassBound^2$ and the Kullback-Leibler loss $\klLoss:[-\hypothesisClassBound,\hypothesisClassBound]^2 \rightarrow [0, \log(1+e^\hypothesisClassBound)]$ satisfies Assumption \ref{ass:nonDegenerateLoss} with $\nonDegeneracyConstant=\log(2/(1+e^{-\hypothesisClassBound}))$.

\begin{theorem}[Minimax lower bound]\label{regRLBtext} Let $\lossFunction:[-\hypothesisClassBound,\hypothesisClassBound] \times \Y \rightarrow [0,\lossFunctionBound]$ be a loss function satisfying Assumptions \ref{boundedLossFunctionAssumption}, \ref{lipschitzLossFunctionAssumption} and \ref{ass:nonDegenerateLoss} with parameter $\nonDegeneracyConstant \in (0,\infty)$. Let $\X$ be Hilbert space containing a non-zero element, and take $\Gamma=(\Wmax,\spectralConstant,\spectralDecay) \in [1,\infty)^2\times (0,1)$ with $\Wmax \geq 2\hypothesisClassBound\cdot \spectralDecay^{-1/2}$
and $\approxError \in [0,1]$. There exists a constant $c_{\ref{regRLBtext}}>0$ depending only upon  $\nonDegeneracyConstant$ and $\lossFunctionBound$ such that for any $n\in N$, and any empirical predictor $\hat{\phi}:(\X\times\Y)^n\times \X\rightarrow \Y$, there exists a distribution $\probDistribution \in \setOfMeasures_{\lossFunction}(\Gamma,\approxError)$ such that \[\E_{\sampleXY\sim \probDistribution^{\otimes n}}[\excessRiskArg_{\lossFunction,\mathrm{\probDistribution}}(\hat{\phi})]
\ge c_{\ref{regRLBtext}} \cdot \left\lbrace n^{-1}+\approxError\right\rbrace.\]
\end{theorem}
Theorem \ref{regRLBtext} shows that the upper bound achieved by compressive regression in Theorem \ref{thm:OLSThm} is minimax optimal up to logarithmic factors. A proof of Theorem \ref{regRLBtext} is given in Section \ref{Proofs4LB}.

\section{Proof of the general upper bound}\label{SecProofs}
The first stage of the proof is to establish a high probability uniform upper bound on the integrated deviations of dependent empirical processes where the integration is with respect to the distribution over random projections (Section \ref{SecMainUniform},
Theorem  \ref{ThmMainUniform}).  This result may be of independent interest, and is therefore stated in a more general setting.

The second stage of the proof of Theorem \ref{ThmMainResult} is based on applying Theorem \ref{ThmMainUniform} to the ensemble of compressive ERMs that each act on a random mapping of the input space (Section \ref{SecPfMainResult}).

\subsection{A concentration inequality for the integrated empirical processes}
\label{SecMainUniform}
In this section we shall derive a concentration inequality for sequences of functions $(g_{\omega})_{\omega \in \Omega} \in \prod_{\omega \in \Omega}\G_{\omega}$ where each $\G_{\omega}\subseteq \measurableMaps(\Z,\R)$ is a function class. Later we will apply this result to settings in which each function class is associated with a random projection.

We begin by recalling the concept of empirical Rademacher complexity. Let's suppose $\G\subseteq \measurableMaps(\Z,\R)$ is a function class. We shall say that $\G$ is \emph{separable} if it is separable with respect to the topology of pointwise convergence, so there exists a countable subset $\G^{\circ} \subseteq \G$ with the property that for any $g \in \G$ there exists a sequence $(g_{\ell})_{\ell \in \N}$ such that $\lim_{\ell \rightarrow \infty}g_{\ell}(z)=g(z)$ for all $z \in \Z$. Given a separable function class $\G\subseteq \measurableMaps(\Z,\R)$ and a sequence $\zSequenceSizeN = \{z_j\}_{j \in [n]}\in \Z^n$, the corresponding \emph{empirical Rademacher complexity} is defined by 
\begin{align*}
\rademacherComplexityEmpirical\left(\G,\zSequenceSizeN\right): = \E_{\sigmaSequenceSizeN}\left( \sup_{g \in \G} \left\lbrace \frac{1}{n}\sum_{j \in [n]}\sigma_j\cdot g(z_j)\right\rbrace \right),    
\end{align*}
where the expectation is taken of independent Rademacher random variables $\bm{\sigmaSequenceSizeN} \in \{-1,+1\}^n$. Note that the assumption that $\G\subseteq \measurableMaps(\Z,\R)$ is seperable ensures that the supremum $\sup_{g \in \G} \left\lbrace \frac{1}{n}\sum_{j \in [n]}\sigma_j\cdot g(z_j)\right\rbrace=\sup_{g \in \G^{\circ}} \left\lbrace \frac{1}{n}\sum_{j \in [n]}\sigma_j\cdot g(z_j)\right\rbrace$ for a countable subset $\G^{\circ}\subseteq \G$. Consequently this supremum is a measurable function and has a well-defined expectation \cite[Chapter 11]{boucheron2013concentration}. Given a probability measure $\probDistribution$ on $\Z$ and a function $g \in \measurableMaps(\Z,\R)$ we let $\probDistribution(g)= \int g d\probDistribution$. Given a sequence $\zSequenceSizeN = \{z_j\}_{j \in [n]} \in \Z^n$ we define an empirical probability measure $\empiricalProbDistributionDeterministicZ$ by
\begin{align*}
\empiricalProbDistributionDeterministicZ(g):= \frac{1}{n}\sum_{j \in [n]}g(z_j).
\end{align*}
In particular, given a random sequence $\randomZSequenceSizeN = \{Z_j\}_{j \in [n]}$ where $Z_j$ are independent $\Z$-valued random variables, $\empiricalProbDistributionSample$ is a random probability measure.

\begin{theorem}[Local Rademacher concentration inequality for integrated deviation]\label{ThmMainUniform}
Suppose we have a set $\Omega$ along with a probability measure $\nu$ on $\Omega$. For each $\omega \in \Omega$, we have a separable class of functions $\G_{\omega}\subseteq \measurableMaps_{1}(\Z)$ and a function $\phi_{\omega}:\Z^n \times (0,\infty) \rightarrow (0,\infty)$ such that for each $\zSequenceSizeN \in \Z^n$, $r\mapsto \phi_{\omega}(\zSequenceSizeN,r)$ is non-decreasing, $r \mapsto \phi_{\omega}(\zSequenceSizeN,r)/\sqrt{r}$ is non-increasing, and for all $r>0$ and $\zSequenceSizeN = \{z_j\}_{j \in [n]} \in \Z^n$,
\begin{align*}
\rademacherComplexityEmpirical\left( \left\lbrace g_{\omega} \in \G_{\omega}: \empiricalProbDistributionDeterministicZ(g_{\omega}^2)\leq r\right\rbrace,\zSequenceSizeN\right) \leq \phi_{\omega}(\zSequenceSizeN,r).
\end{align*}
For each $\zSequenceSizeN \in \Z^n$ we choose $\rho_{\omega}^*(\zSequenceSizeN)\in (0,\infty)$ so that $\phi_{\omega}(\zSequenceSizeN,\rho_{\omega}^*(\zSequenceSizeN))={\rho^*_{\omega}(\zSequenceSizeN)}$.  Suppose we have a sequence of independent random variables $\randomZSequenceSizeN=\{Z_j\}_{j \in [n]}$  with common distribution $\probDistributionZ$. Given any $\delta \in (0,1)$, with probability at least $1-\delta$ over $\randomZSequenceSizeN$ the following holds for all sequences $(g_{\omega})_{\omega \in \Omega} \in \prod_{\omega \in \Omega}\G_{\omega}$,
\begin{align*}
\int_{\Omega} \left|\empiricalProbDistributionSample(g_{\omega})-\probDistribution(g_{\omega})\right|d\nu(\omega) \leq  \sqrt{ \int_{\Omega}\probDistribution(g_{\omega}^2)d\nu(\omega) \cdot \nDeltaDataComplexityTerm }+\nDeltaDataComplexityTerm,
\end{align*}
where $\nDeltaDataComplexityTerm=650 \cdot \int \rho_{\omega}^*(\randomZSequenceSizeN)d\nu(\omega)+{152\log(4\log(n)/\delta)}/{n}$.
\end{theorem}
Theorem \ref{ThmMainUniform} may be viewed as a generalisation of concentration inequalities in the literature which provide similar high probability bounds for the special case in which $\Omega$ is a singleton \citep{massart2000some,massart2006risk,koltchinskii2006local,bartlett2005local,boucheron2013concentration}. Theorem \ref{ThmMainUniform} builds upon these results and implies that the tail of the integrated deviations of infinitely many (possibly dependent) processes are of the same order of those for a single process. In other words, the failure probabilities of the individual concentration guarantees do not accumulate despite there is arbitrary dependence among them. This will translate into desirable learning guarantees for randomised ensembles that use an averaging-type combination rule. 

Note that the high probability bound in Theorem \ref{ThmMainUniform} cannot be immediately deduced from the special case in which $\Omega$ is a singleton, as integrating both sides would require a union bound over all $\omega \in \Omega$, which would blow up the failure probability. On the other hand, such an approach would allow us to obtain an expectation bound by Fubini's theorem. To prove Theorem \ref{ThmMainUniform} we can apply this idea to the logarithm of the moment generating function to show that the tail of the integrated process is not much larger than the tail of the individual processes (Lemma \ref{highProbExpTailsAveragingLemma}). In order to apply this idea to obtain a local bound we must first decouple the $\delta$-dependency from the function dependent terms in our bound which requires Lemma \ref{lambdaAPlusOneOverLambdaBLemma}.

\begin{lemma}[
Ensemble tail bound]\label{highProbExpTailsAveragingLemma}Suppose that we have a sequence of real-valued random variable $\{X({\omega})\}_{\omega \in \Omega}$ such that for some constant $\kappa>0$ and all $\omega \in \Omega$, $\delta \in (0,1)$ we have $\Prob\left( X({\omega}) > \kappa \cdot \log(1/\delta)\right) \leq \delta$. Given any probability measure $\nu$ on $\Omega$ we have 
\begin{align*}
\Prob\left( \int X({\omega})d\nu(\omega) > 2\kappa \cdot \log(2/\delta)\right) \leq \delta.
\end{align*}
\end{lemma}
\begin{proof} Given any $\lambda \in (0,1/\kappa)$ and $\omega \in \Omega$ we have,
\begin{align*}
\E\left[\exp(\lambda \cdot X({\omega}))\right] &= \int_0^{\infty} \Prob\left[ X({\omega})>\frac{\log(t)}{\lambda}\right]dt=\int_0^{\infty}\max\{1,t^{-\frac{1}{\lambda\cdot \kappa}}\}dt= \frac{1}{1-\lambda\cdot \kappa}.
\end{align*}
By Jensen's inequality we have $\E\left[\exp(\lambda \cdot \int X({\omega})d\nu(\omega))\right]\leq \int \E\left[\exp(\lambda \cdot X({\omega}))\right] d\nu(\omega)\leq 1/(1-\lambda\cdot \kappa)$. Hence, for each $t>0$, $\lambda \in (0,1/\kappa)$ we have
\begin{align*}
\Prob\left( \int X({\omega})d\nu(\omega) > t\right) &\leq \E\left[ \exp\left( \lambda \cdot \int X({\omega})d\nu(\omega) \right)\right] \cdot e^{-\lambda \cdot t} = \frac{e^{-\lambda\cdot t}}{1-\lambda\cdot \kappa}.
\end{align*}
The lemma follows by taking $t=2\kappa \cdot \log(2/\delta)$ and $\lambda = 1/(2\kappa)$.
\end{proof}

\begin{lemma}\label{lambdaAPlusOneOverLambdaBLemma} Let $\Lambda(n) = \left\lbrace  \frac{2^{1+q}}{n}: {q \in \{0,\cdots,\lfloor \log_2(n)\rfloor -1\}}\right\rbrace$. Given any $x \in [0,1]$ and $y \geq 0$ we have 
\begin{align*}
\min_{\lambda \in \Lambda(n)}\left\lbrace \lambda \cdot x + \frac{y}{\lambda} \right\rbrace \leq \max\left\lbrace 3\sqrt{xy},3y,\frac{4}{n}\right\rbrace.
\end{align*}
\end{lemma}
\begin{proof} We consider three cases. 
\begin{enumerate}
    \item If $x\leq y$ then with $\lambda = n^{-1}\cdot 2^{\lfloor \log_2(n)\rfloor} \in \Lambda(n)$ we have $\lambda \in (1/2,1]$, so $ \lambda \cdot x + {y}/{\lambda} \leq 3y$.
    \item If $\sqrt{{y}/{x}}\leq {2}/{n}$ then with $\lambda = 2\cdot n^{-1} \in \Lambda(n)$ we use $x \leq 1$ to infer $\lambda \cdot x \leq 2\cdot n^{-1}$ and $\sqrt{{y}}\leq {2}\cdot{n}^{-1}$ so $y/ \lambda \leq 2/n$. Hence,  $ \lambda \cdot x + {y}/{\lambda} \leq 4/n$.
    \item If $\sqrt{{y}/{x}} \in (2/n,1]$ then there exists some $\lambda \in \Lambda(n)$ with $\lambda < \sqrt{{y}/{x}} \leq 2\lambda$. Hence, $\lambda \cdot x + {y}/{\lambda}\leq 3\sqrt{xy}$.
\end{enumerate}
\end{proof}

We can now complete the proof of Theorem \ref{ThmMainUniform}.

\begin{proof}[Proof of Theorem \ref{ThmMainUniform}] 
We begin by fixing $\lambda>0$. By considering a special case in which $\Omega$ is a singleton (Theorem \ref{localRademacherConcentrationInequality}, Appendix \ref{A1}) we see that for each $\omega \in \Omega$ and  $\delta \in (0,1)$, with probability at least $1-\delta$ over $\randomZSequenceSizeN$, the following holds for all $g_{\omega} \in \G_{\omega}$,
\begin{align}\label{firstClaimPfOfLocalRademacherConcentrationInequalityForEnsembles}
\left|\empiricalProbDistributionRandomZ(g_{\omega})-\probDistribution(g_{\omega})\right| &\leq    2\cdot \sqrt{\probDistribution(g_{\omega}^2)\cdot\left( 72\cdot \rho_{\omega}^*(\randomZSequenceSizeN)+\frac{ 2\log(4\log(n)/\delta)}{n}\right)} \nonumber\\&\hspace{1cm}+ \left(132 \cdot \rho_{\omega}^*(\randomZSequenceSizeN)+\frac{30\log(4\log(n)/\delta)}{n}\right)\nonumber\\
 &\leq   \lambda \cdot \probDistribution(g_{\omega}^2)+ \left(132+\frac{72}{\lambda}\right) \cdot \rho_{\omega}^*(\randomZSequenceSizeN)+\left(30+\frac{2}{\lambda}\right)\cdot\frac{\log(4\log(n)/\delta)}{n}.
\end{align}
We define random variables for each $\omega \in \Omega$ by
{\small
\begin{align*}
X_{\lambda}(\omega):= \sup_{g_{\omega} \in \G_{\omega}}\left\lbrace 
\left|\empiricalProbDistributionRandomZ(g_{\omega})-\probDistribution(g_{\omega})\right| - \lambda \cdot \probDistribution(g_{\omega}^2)\right\rbrace-\left(132+\frac{72}{\lambda}\right) \cdot \rho_{\omega}^*(\randomZSequenceSizeN)-\left(30+\frac{2}{\lambda}\right)\cdot\frac{\log(4\log(n))}{n}.
\end{align*}
}
By (\ref{firstClaimPfOfLocalRademacherConcentrationInequalityForEnsembles}) we have $\Prob\left(X_{\lambda}(\omega) > \left(30+\frac{2}{\lambda}\right)\cdot \frac{\log(1/\delta)}{n} \right) \leq \delta$ for all $\delta \in (0,1)$. Thus, by Lemma \ref{highProbExpTailsAveragingLemma} the following holds for all $\delta \in (0,1)$,
\begin{align*}
\Prob\left( \int X_{\lambda}({\omega})d\nu(\omega) > \left(60+\frac{4}{\lambda}\right)\cdot \frac{\log(2/\delta)}{n}\right) \leq \delta.
\end{align*}
Hence, for each  $\delta \in (0,1)$ the following holds with probability at least $1-\delta$,
\begin{align*}
\sup_{(g_{\omega})_{\omega \in \Omega} \in \prod_{\omega \in \Omega}\G_{\omega}} &\left\lbrace \int_{\Omega} \left|\empiricalProbDistributionRandomZ(g_{\omega})-\probDistribution(g_{\omega})\right|d\nu(\omega)-\lambda \cdot \int_{\Omega} \probDistribution(g_{\omega}^2)d\nu(\omega)\right\rbrace \nonumber\\
& \leq \int_{\Omega}\sup_{g_{\omega} \in \G_{\omega}}\left\lbrace 
\left|\empiricalProbDistributionRandomZ(g_{\omega})-\probDistribution(g_{\omega})\right| - \lambda \cdot \probDistribution(g_{\omega}^2)\right\rbrace d\nu(\omega)\nonumber\\
&\leq \left(132+\frac{72}{\lambda}\right) \cdot \int_{\Omega} \rho_{\omega}^*(\randomZSequenceSizeN)d\nu(\omega)+\left(60+\frac{4}{\lambda}\right)\cdot\frac{\log(4\log(n)/\delta)}{n}.
\end{align*}
Thus, given any $\delta \in (0,1)$ and $\lambda>0$  the following holds with probability at least $1-\delta$, for all ${(g_{\omega})_{\omega \in \Omega} \in \prod_{\omega \in \Omega}\G_{\omega}}$,
\begin{align}\label{secondClaimPfOfLocalRademacherConcentrationInequalityForEnsembles}
\int_{\Omega} \left|\empiricalProbDistributionRandomZ(g_{\omega})-\probDistribution(g_{\omega})\right|d\nu(\omega)
&\leq \lambda \cdot \int_{\Omega} \probDistribution(g_{\omega}^2)d\nu(\omega)+\frac{1}{\lambda} \cdot \left(72 \cdot \int \rho_{\omega}^*(\randomZSequenceSizeN)d\nu(\omega)+\frac{4\log(4\log(n)/\delta)}{n}\right)\nonumber\\
&\hspace{1cm}+132 \cdot \int \rho_{\omega}^*(\randomZSequenceSizeN)d\nu(\omega)+\frac{60\log(4\log(n)/\delta)}{n}.
\end{align}
Now take  $\Lambda(n) = \left\lbrace  \frac{2^{1+q}}{n}: {q \in \{0,\cdots,\lfloor \log_2(n)\rfloor -1\}}\right\rbrace$ as in Lemma \ref{lambdaAPlusOneOverLambdaBLemma}. Note that $\Lambda(n)$ has cardinality no more that $\log_2(n) \leq 2\log(n)$. By (\ref{secondClaimPfOfLocalRademacherConcentrationInequalityForEnsembles}) combined with Lemma \ref{lambdaAPlusOneOverLambdaBLemma} we see that with probability at least $1-2\log(n) \cdot \delta$, the following holds for all ${(g_{\omega})_{\omega \in \Omega} \in \prod_{\omega \in \Omega}\G_{\omega}}$,
{\small
\begin{align*}
\int_{\Omega} \left|\empiricalProbDistributionRandomZ(g_{\omega})-\probDistribution(g_{\omega})\right|d\nu(\omega)
&\leq \min_{\lambda \in \Lambda(n)}\left\lbrace \lambda \cdot \int_{\Omega} \probDistribution(g_{\omega}^2)d\nu(\omega)+\frac{1}{\lambda} \cdot \left(72 \cdot \int \rho_{\omega}^*(\randomZSequenceSizeN)d\nu(\omega)+\frac{4\log(4\log(n)/\delta)}{n}\right)\right\rbrace\\
&\hspace{1cm}+132 \cdot \int \rho_{\omega}^*(\randomZSequenceSizeN)d\nu(\omega)+\frac{60\log(4\log(n)/\delta)}{n}\\
& \leq 6 \sqrt{ \int_{\Omega}\probDistribution(g_{\omega}^2)d\nu(\omega) \cdot \left(18 \cdot \int \rho_{\omega}^*(\randomZSequenceSizeN)d\nu(\omega)+\frac{\log(4\log(n)/\delta)}{n}\right) }\\
&\hspace{1cm}+348 \cdot \int \rho_{\omega}^*(\randomZSequenceSizeN)d\nu(\omega)+\frac{72\log(4\log(n)/\delta)}{n}.
\end{align*}
}
Taking $\delta/(2\log(n))$ in place of $\delta$ we see that the following holds with probability at least $1-\delta$ over $\randomZSequenceSizeN$, for all $(g_{\omega})_{\omega \in \Omega} \in \prod_{\omega \in \Omega}\G_{\omega}$,
\begin{align*}
\int_{\Omega} \left|\empiricalProbDistributionRandomZ(g_{\omega})-\probDistribution(g_{\omega})\right|d\nu(\omega)
& \leq 6 \sqrt{ \int_{\Omega}\probDistribution(g_{\omega}^2)d\nu(\omega) \cdot \left(18 \cdot \int \rho_{\omega}^*(\randomZSequenceSizeN)d\nu(\omega)+\frac{2\log(4\log(n)/\delta)}{n}\right) }\\
&\hspace{1cm}+\left(348 \cdot \int \rho_{\omega}^*(\randomZSequenceSizeN)d\nu(\omega)+\frac{152\log(4\log(n)/\delta)}{n}\right).
\end{align*}
\end{proof}

\subsection{High probability bound on the ensemble error of compressive empirical risk minimisers}\label
{SecPfMainResult}

The second stage of the proof of Theorem \ref{ThmMainResult} is to establish Proposition \ref{infiniteEnsembleHighProbUpperBoundProp} below, which gives a high probability upper bound on the ensemble error of compressive empirical risk minimisers. First, we require some additional notation. Given a mapping $\randomProjection \in \setOfRandomProjections$ and a Borel probability distribution $\probDistribution$ on $\X\times \Y$, we define an associated function class 
\begin{align*}
\G_{\randomProjection}:= \left\lbrace (x,y)\mapsto \lossFunction\left(f(A(x)),y\right)-\lossFunction\left(\oracleFunction(x),y\right): f \in \lowDimensionalFunctionClass \right\rbrace \subseteq \measurableMaps\left( \X\times \Y, \R\right).
\end{align*}
We shall also refer to \emph{data-dependent elements of $\G_{\randomProjection}$} which are random elements of $\G_{\randomProjection}$ which implicitly depend upon the data $\sample=((X_j,Y_j))_{j\in [n]}$. We also define the compressibility of a finite ensemble of size $m\in\N$ as the $1-\delta$ upper quantile of the approximation error of the ensemble given by
\begin{align*}
\compressibilityFunction_{\probDistribution,\delta,m}(k):=\inf\left\lbrace {\bar{\psi}} \in \R:  \int_{\setOfRandomProjections^m}
\one\left\lbrace \frac{1}{m}\sum_{i=1}^m \inf_{f \in \lowDimensionalFunctionClass
}\lbrace \excessRisk\left(f \circ \randomProjection_i\right)\rbrace \ge 
\bar{\psi}\right\rbrace d\randomProjectionMeasure(A_1) \ldots d\randomProjectionMeasure(A_m)\le \delta \right\rbrace.
\end{align*}
When the distribution $\probDistribution$ is clear from context we shall denote $\compressibilityFunction_{\probDistribution,\delta,m}(k)$ by $\compressibilityFunction_{\delta,m}(k)$.

\begin{prop}[Ensemble error of compressive ERMs]\label{infiniteEnsembleHighProbUpperBoundProp}  Suppose that Assumptions \ref{boundedLossFunctionAssumption}, \ref{lipschitzLossFunctionAssumption}, \ref{logarithmicCoveringNumbersAssumption} and \ref{bernsteinTysbakovMarginTypeAssumption} hold with parameters $\lossFunctionBound$, $\hypothesisClassBound$, $\lipschitzConstantLoss$, $\coveringNumberConstant$, $\bernsteinConstant$ $\geq 1$, $\bernsteinExponent \in [0,1]$. Take $n \in \N$, $k \in \N$ and $\delta \in (0,1)$,
and for each $\randomProjection \in \setOfRandomProjections$ let $\hat{g}_{\randomProjection}$ be a data-dependent element of $\G_{\randomProjection}$. Then with probability at least $1-{2}\delta$ we have
\begin{align*}
\frac{1}{m}\sum_{i=1}^m  
 \probDistribution(\hat{g}_{\randomProjection_i})  \leq 16&\left( \bernsteinConstant   \nkDeltaComplexityTerm \right)^{\frac{1}{2-\bernsteinExponent}}+\frac{2}{m}\sum_{i=1}^m \left( 
 \empiricalProbDistributionSample(\hat{g}_{\randomProjection_i})- \inf_{g \in \G_{\randomProjection_i}} \empiricalProbDistributionSample(g) \right)+4\left(\lossFunctionBound \cdot \nkDeltaComplexityTerm+\compressibilityFunction_{\delta,m}(k)\right),
\end{align*} 
where $\nkDeltaComplexityTerm=\left((4000 \coveringNumberConstant k)  \cdot \logBar\left(\lipschitzConstantLoss \hypothesisClassBound n\right)+{152\cdot\logBar(4\log(n)/\delta)}\right)\cdot n^{-1}$.
\end{prop}

We prove Proposition \ref{infiniteEnsembleHighProbUpperBoundProp} 
via Lemmas \ref{logCoverNumbersImplySmallLocalRademacherBoundLemma} and \ref{uniformConvergenceBoundForProofOfInfiniteEnsembleHighProbUpperBoundProp} below.  

\begin{lemma}[Local Rademacher complexity bound ]\label{logCoverNumbersImplySmallLocalRademacherBoundLemma} Suppose that Assumptions \ref{lipschitzLossFunctionAssumption} and \ref{logarithmicCoveringNumbersAssumption} hold  with parameters $\hypothesisClassBound$, $\lipschitzConstantLoss$, $\coveringNumberConstant$. Then for each $\randomProjection \in \setOfRandomProjections$, $\mathcal{G}_A$ is separable. Moreover, given a distribution $\probDistribution$ on $\X \times \Y$, along with $r>0$ and $\zSequenceSizeN \in \Z^n$ we have
\begin{align*}
\rademacherComplexityEmpirical\left( \left\lbrace g \in \G_{\randomProjection}: \empiricalProbDistributionSample(g^2)\leq r\right\rbrace, \zSequenceSizeN \right) \leq 2\sqrt{\frac{\coveringNumberConstant \cdot k \cdot r}{n}} \cdot  \logBar^{\frac{1}{2}} \left(\frac{\lipschitzConstantLoss \hypothesisClassBound n}{ k \sqrt{r}}\right).
\end{align*}
\end{lemma}
The proof of Lemma \ref{logCoverNumbersImplySmallLocalRademacherBoundLemma} is standard and is contained in Appendix \ref{A2} for completeness. We shall now deduce the following consequence of   
Lemma \ref{logCoverNumbersImplySmallLocalRademacherBoundLemma} combined with Theorem \ref{ThmMainUniform}.

\begin{lemma}\label{uniformConvergenceBoundForProofOfInfiniteEnsembleHighProbUpperBoundProp} Suppose that we  are in the setting of Proposition \ref{infiniteEnsembleHighProbUpperBoundProp}. Then with probability at least $1-\delta$ the following holds for all $(g_{\randomProjection_i})_{i=1}^m \in \prod_{i=1}^m\G_{\randomProjection_i,\mathrm{P}}$,
\begin{align}\label{eqn:uniformConvergenceBoundForProofOfInfiniteEnsembleHighProbUpperBoundProp}
\frac{1}{m}\sum_{i=1}^m \left|\empiricalProbDistributionSample(g_{\randomProjection_i})-\probDistribution(g_{\randomProjection_i})\right| \leq \sqrt{  \bernsteinConstant \cdot \nkDeltaComplexityTerm \cdot 
  \left\lbrace \frac{1}{m}\sum_{i=1}^m 
 \probDistribution\left(g_{\randomProjection_i}\right) 
 \right\rbrace^{\bernsteinExponent} }+\lossFunctionBound \cdot \nkDeltaComplexityTerm 
 .
\end{align}
\end{lemma}

\begin{proof} We introduce function classes ${\G}_{\randomProjection}^{\flat}$ for each $\randomProjection \in \setOfRandomProjections$ by ${\G}_{\randomProjection,\mathrm{P}}^{\flat}:= \{\lossFunctionBound^{-1}\cdot g\}_{g \in \G_{\randomProjection,\mathrm{P}}}$ so that each function class ${\G}_{\randomProjection,\mathrm{P}}^{\flat}\subseteq \measurableMaps_1(\Z)$, where $\Z=\X\times \Y$, which will allow us to apply Theorem \ref{ThmMainUniform}. Let $\phi:(0,\infty) \rightarrow (0,\infty)$ denote the function defined by 
\begin{align*}
\phi(r):=  2\sqrt{\frac{\coveringNumberConstant \cdot k \cdot r}{n}} \cdot  \logBar^{\frac{1}{2}} \left(\frac{\lipschitzConstantLoss \hypothesisClassBound n}{ \sqrt{r} }\right).
\end{align*}
Observe that $r\mapsto \phi(r)$ is non-decreasing, $r \mapsto \phi(r)/\sqrt{r}$ is non-increasing and by Lemma \ref{logCoverNumbersImplySmallLocalRademacherBoundLemma} we have 
\begin{align*}
\rademacherComplexityEmpirical\left( \left\lbrace  g^{\flat} \in \G_{\randomProjection}^{\flat}: \empiricalProbDistributionSample\{(g^{\flat})^2\}\leq r\right\rbrace, \zSequenceSizeN \right) = \lossFunctionBound^{-1} \cdot \rademacherComplexityEmpirical\left( \left\lbrace  g \in \G_{\randomProjection}: \empiricalProbDistributionSample(g^2)\leq \lossFunctionBound^2 \cdot r\right\rbrace, \zSequenceSizeN \right) \leq \phi(r),
\end{align*}
for all $r>0$. Choose $\rho^*$ so that $\phi(\rho^*)=\rho^*$ and observe that $\rho^* \leq (6 \coveringNumberConstant k) \cdot {n}^{-1} \cdot \logBar\left(\lipschitzConstantLoss \hypothesisClassBound n\right)$. Note that the samples $Z_j=(X_j,Y_j)$ are assumed to be independent and identically distributed. Moreover, random projections $(\randomProjection_i)_{i=1}^m$ are independent from the data $\sample=((X_j,Y_j))_{j=1}^n$ and so samples $Z_j=(X_j,Y_j)$ are also independent and identically distributed with respect to the conditional distribution. Hence, by Theorem \ref{ThmMainUniform} we have,
{\small
\begin{align*}
\Prob\left(\sup_{(g_{\randomProjection_i}^{\flat})_{i=1}^m \in \prod_{i=1}^m\G_{\randomProjection_i}^{\flat}}\left\lbrace \frac{1}{m}\sum_{i=1}^m \left|\empiricalProbDistributionSample(g_{\randomProjection_i}^{\flat})-\probDistribution(g_{\randomProjection_i}^{\flat})\right|-  \sqrt{ \frac{1}{m}\sum_{i=1}^m\probDistribution\left\lbrace (g_{\randomProjection_i}^{\flat})^2\right\rbrace \cdot \nkDeltaComplexityTerm }\right\rbrace>\nkDeltaComplexityTerm \Bigg|  (A_i)_{i=1}^m\right)\leq \delta.
\end{align*}
}
Hence, by the law of total expectation with probability at least $1-\delta$ the following holds for all $(g_{\randomProjection_i})_{i=1}^m \in \prod_{i=1}^m\G_{\randomProjection_i}$
\begin{align}\label{consequenceOfMainEnsembleConcentrationBound}
\frac{1}{m}\sum_{i=1}^m \left|\empiricalProbDistributionSample(g_{\randomProjection_i})-\probDistribution(g_{\randomProjection_i})\right|\leq   \sqrt{ \frac{1}{m}\sum_{i=1}^m\probDistribution\left\lbrace (g_{\randomProjection_i})^2\right\rbrace \cdot \nkDeltaComplexityTerm }+\lossFunctionBound \cdot \nkDeltaComplexityTerm.
\end{align}
Thus, to complete the proof it suffices to deduce \eqref{eqn:uniformConvergenceBoundForProofOfInfiniteEnsembleHighProbUpperBoundProp} from \eqref{consequenceOfMainEnsembleConcentrationBound}. Given any function ${g} \in \G_{\randomProjection} \subseteq  \measurableMaps_{\lossFunctionBound}(\Z)$ with $\randomProjection \in \setOfRandomProjections$ there exists $\functionGeneral \in \measurableMaps(\X,\actionSpace)$ with $g(x,y)=\lossFunction\left(\functionGeneral(x),y\right)-\lossFunction\left(\oracleFunction_{\mathrm{P}}(x),y\right)$ for $(x,y) \in \X\times \Y$. Since $\lossFunction$ and $\mathrm{P}$ satisfy the Bernstein condition with parameters $\bernsteinExponent$, $\bernsteinConstant$ we see by \eqref{consequenceOfMainEnsembleConcentrationBound} with probability at least $1-\delta$ the following holds for all
$(g_{\randomProjection_i})_{i=1}^m \in \prod_{i=1}^m\G_{\randomProjection_i}$,
\begin{align} 
\frac{1}{m}\sum_{i=1}^m \left|\empiricalProbDistributionSample(g_{\randomProjection_i})-\probDistribution(g_{\randomProjection_i})\right|d\randomProjectionMeasure(\randomProjection) &\leq  \sqrt{ \frac{1}{m}\sum_{i=1}^m \probDistribution(g_{\randomProjection_i}^2) \cdot \nkDeltaComplexityTerm }+\lossFunctionBound \cdot \nkDeltaComplexityTerm \nonumber\\
 &\leq  \sqrt{ \left( \bernsteinConstant \cdot \nkDeltaComplexityTerm\right) \cdot 
  \left\lbrace \frac{1}{m}\sum_{i=1}^m 
 \probDistribution\left(g_{\randomProjection_i}\right) 
 \right\rbrace^{\bernsteinExponent} }+\lossFunctionBound \cdot \nkDeltaComplexityTerm 
 .
\end{align}
where we used the Bernstein-Tsybakov condition and Jensen's inequality combined with the convexity of $z \mapsto z^{\frac{1}{\bernsteinExponent}}$.
\end{proof}

To complete the proof of Proposition \ref{infiniteEnsembleHighProbUpperBoundProp} we also require the following elementary lemma, which will be used to rearrange the bound into an additive form.

\begin{lemma}\label{elementarySelfBoundingLemmaKappa} Given $x,c,d>0$ satisfying $x \leq \sqrt{cx^{\bernsteinExponent}}+d$ and $\vartheta \in ( 1,2]$ we have
$x \leq 4 \big\{ \left(\vartheta-1\right)^{-2}\cdot c\big\}^{\frac{1}{2-\bernsteinExponent}}+\vartheta \cdot d$.
\end{lemma}

We are now ready to prove Proposition \ref{infiniteEnsembleHighProbUpperBoundProp}.

\begin{proof}[Proof of Proposition \ref{infiniteEnsembleHighProbUpperBoundProp}] Fix $\epsilon>0$ and for each $\randomProjection \in \setOfRandomProjections$ choose ${g}_{\randomProjection}^* \in \G_{\randomProjection}$ so that $\probDistribution\left({g}_{\randomProjection}^*\right) \leq \inf_{g \in \G_{\randomProjection}} \{ \probDistribution(g)\}+\epsilon$. We define a pair of events $E^1_{\epsilon,\delta}$ and $E^2_{\delta}$ by
\begin{align*}
E^1_{\epsilon,\delta}&:= \left\lbrace 
\frac{1}{m}\sum_{i=1}^m
\inf_{g \in \G_{\randomProjection_i}} \probDistribution(g)
 \le \compressibilityFunction_{\probDistribution,\delta,m}(k)+\epsilon\right\rbrace,\\  
E^2_{\delta}&:=\left\lbrace \frac{1}{m}\sum_{i=1}^m \left|\empiricalProbDistributionSample(g_{\randomProjection_i})-\probDistribution(g_{\randomProjection_i})\right| \leq  \sqrt{  \bernsteinConstant \cdot \nkDeltaComplexityTerm \cdot 
  \left\lbrace \frac{1}{m}\sum_{i=1}^m 
 \probDistribution\left(g_{\randomProjection_i}\right) 
 \right\rbrace^{\bernsteinExponent} }+\lossFunctionBound \cdot \nkDeltaComplexityTerm 
  \right\rbrace. \nonumber
\end{align*}
where the intersection is over all ${(g_{\randomProjection_i})_{i=1}^m \in \prod_{i=1}^m\G_{\randomProjection_i}}$. We claim that $\Prob(E^1_{\epsilon,\delta})\geq 1-\delta$. Indeed, by the definition of $\compressibilityFunction_{\probDistribution,\delta,m}(k)$ with probability at least $1-\delta$ we have 
\begin{align*}
\compressibilityFunction_{\delta,m}(k)+\epsilon&\geq \frac{1}{m}\sum_{i=1}^m \inf_{f \in \lowDimensionalFunctionClass
}\lbrace \excessRisk\left(f \circ \randomProjection_i\right)\rbrace \\
&\geq\frac{1}{m}\sum_{i=1}^m \inf_{f \in \lowDimensionalFunctionClass
}\left\lbrace \risk(f \circ \randomProjection_i)-\risk(\oracleFunction) \right\rbrace \\
&\geq\frac{1}{m}\sum_{i=1}^m \inf_{f \in \lowDimensionalFunctionClass
} \int \left\lbrace\lossFunction\left(f(A_i(x)),y\right)-\lossFunction\left(\oracleFunction(x),y\right)\right\rbrace d\probDistribution(x,y)=\frac{1}{m}\sum_{i=1}^m
\inf_{g \in \G_{\randomProjection_i}} \probDistribution(g).
\end{align*}
In addition, by Lemma \ref{uniformConvergenceBoundForProofOfInfiniteEnsembleHighProbUpperBoundProp} we have $\Prob\left(E^2_{\delta}\right)>1-\delta$. Thus, by the union bound we have $\Prob(E^1_{\epsilon,\delta} \cap E^2_{\delta})>1-2\delta$. We write $\averageOptimisationError:=\frac{1}{m}\sum_{i=1}^m
 \empiricalProbDistributionSample(\hat{g}_{\randomProjection_i})- \inf_{g \in \G_{\randomProjection_i}} \empiricalProbDistributionSample(g)$, so on the event $E^2_{\delta}$ we have
\begin{align*}
&\frac{1}{m}\sum_{i=1}^m  
\left( \probDistribution(\hat{g}_{\randomProjection_i})-\probDistribution({g}^*_{\randomProjection_i})\right) 
\nonumber \\
&\leq 
\frac{1}{m}\sum_{i=1}^m
\left( \empiricalProbDistributionSample(\hat{g}_{\randomProjection_i})-\empiricalProbDistributionSample({g}^*_{\randomProjection_i})\right) 
+
\frac{1}{m}\sum_{i=1}^m
\left|\empiricalProbDistributionSample(\hat{g}_{\randomProjection_i})-\probDistribution(\hat{g}_{\randomProjection_i})\right|
+\frac{1}{m}\sum_{i=1}^m
\left|\empiricalProbDistributionSample(g_{\randomProjection_i}^*)-\probDistribution(g_{\randomProjection_i}^*)\right|
\nonumber\\
& \leq  \averageOptimisationError+\sqrt{ \left( \bernsteinConstant \cdot \nkDeltaComplexityTerm\right) \cdot 
  \left\lbrace \frac{1}{m}\sum_{i=1}^m 
 \probDistribution\left(\hat{g}_{\randomProjection_i}\right) 
 \right\rbrace^{\bernsteinExponent} }+\sqrt{ \left( \bernsteinConstant \cdot \nkDeltaComplexityTerm\right) \cdot 
  \left\lbrace \frac{1}{m}\sum_{i=1}^m 
 \probDistribution\left({g}^*_{\randomProjection_i}\right) 
 \right\rbrace^{\bernsteinExponent} }+\lossFunctionBound \cdot 2\nkDeltaComplexityTerm 
  \\
& \leq 2\sqrt{ \left( \bernsteinConstant \cdot \nkDeltaComplexityTerm\right) \cdot 
  \left\lbrace \frac{1}{m}\sum_{i=1}^m 
 \probDistribution\left(\hat{g}_{\randomProjection_i}\right) 
 \right\rbrace^{\bernsteinExponent} }+\lossFunctionBound \cdot 2\nkDeltaComplexityTerm 
  +\averageOptimisationError.
\end{align*} 
Note also that on the event $E^1_{\epsilon,\delta}$ we have $\frac{1}{m}\sum_{i=1}^m\probDistribution({g}^*_{\randomProjection_i}) \leq \frac{1}{m}\sum_{i=1}^m
\inf_{g \in \G_{\randomProjection_i}} \{\probDistribution(g)\}+\epsilon \leq \compressibilityFunction_{\delta,m}(k)+2\epsilon$. Hence, on the event $E^1_{\epsilon,\delta} \cap E^2_{\delta}$ we have
\begin{align*}
\frac{1}{m}\sum_{i=1}^m  
 \probDistribution(\hat{g}_{\randomProjection_i}) 
 &\leq  2\sqrt{ \left( \bernsteinConstant \cdot \nkDeltaComplexityTerm\right) \cdot 
  \left\lbrace \frac{1}{m}\sum_{i=1}^m 
 \probDistribution\left(\hat{g}_{\randomProjection_i}\right) 
 \right\rbrace^{\bernsteinExponent} }+\averageOptimisationError+2\lossFunctionBound \cdot \nkDeltaComplexityTerm 
 +\frac{1}{m}\sum_{i=1}^m\probDistribution({g}^*_{\randomProjection_i})\\
&\leq 2\sqrt{ \left( \bernsteinConstant \cdot \nkDeltaComplexityTerm\right) \cdot 
  \left\lbrace \frac{1}{m}\sum_{i=1}^m 
 \probDistribution\left(\hat{g}_{\randomProjection_i}\right) 
 \right\rbrace^{\bernsteinExponent} }+\averageOptimisationError+2\lossFunctionBound \cdot \nkDeltaComplexityTerm 
 +\compressibilityFunction_{\delta,m}(k)+2\epsilon.
\end{align*} 
Hence, by Lemma \ref{elementarySelfBoundingLemmaKappa} with $\vartheta=2$ we see that on the event $E^1_{\epsilon,\delta} \cap E^2_{\delta}$ we have
\begin{align*}
\frac{1}{m}\sum_{i=1}^m  
 \probDistribution(\hat{g}_{\randomProjection_i}) 
 \leq 4\left(4 \bernsteinConstant   \nkDeltaComplexityTerm \right)^{\frac{1}{2-\bernsteinExponent}}+2\left\lbrace \averageOptimisationError+2\lossFunctionBound \cdot \nkDeltaComplexityTerm 
 \right\rbrace+2\compressibilityFunction_{\delta,m}(k)+4\epsilon.
\end{align*} 
Since $\Prob(E^1_{\epsilon,\delta} \cap E^2_{\delta})>1-2\delta$ it follows that with probability at least $1-2\delta$ we have
\begin{align*}
\frac{1}{m}\sum_{i=1}^m  
 \probDistribution(\hat{g}_{\randomProjection_i})  \leq 16\left( \bernsteinConstant   \nkDeltaComplexityTerm \right)^{\frac{1}{2-\bernsteinExponent}}+2\averageOptimisationError+4\lossFunctionBound \cdot \nkDeltaComplexityTerm+2\compressibilityFunction_{\delta,m}(k) 
 +4\epsilon.
\end{align*} 
Letting $\epsilon \rightarrow 0$ and noting that $\averageOptimisationError \leq \compressibilityFunction_{\delta,m}(k)$ gives the required result.
\end{proof}

\subsection{Completing the proof of Theorem 
\ref{ThmMainResult}}

Before completing the proof we apply Bennet's inequality to bound $ \compressibilityFunction_{\delta,m}(k)$ in terms of $\compressibilityFunction(k)$.

\begin{lemma}\label{finitepsi} 
Given any $m\in\N,\delta \in (0,1)$ we have $ \compressibilityFunction_{\delta,m}(k)  \le 2\compressibilityFunction(k) + 3 \lossFunctionBound \log(1/\delta)/(2m)$.
\end{lemma}
\begin{proof} The sequence of random variables $( \inf_{f \in \lowDimensionalFunctionClass
}\lbrace \excessRisk\left(f \circ \randomProjection_i\right)\rbrace )_{i \in [m]}$ are independent, bounded by $\lossFunctionBound$, have expectation $\compressibilityFunction(k)$, and variance no larger than $\lossFunctionBound \cdot \compressibilityFunction(k)$. Hence, the result follows from Bennett's inequality \cite[Chapter 2]{boucheron2013concentration}.
\end{proof}

\begin{proof}[Proof of Theorem \ref{ThmMainResult}] 
By Assumption \ref{quasiConvexityAssumption} and Proposition \ref{infiniteEnsembleHighProbUpperBoundProp} we have,
\begin{align}\label{workingOnTwoGoodEventsHighProbUB}
\excessRisk\left(\functionEstimate_{n,k,m}\right)&\leq \frac{\quasiConvexityConstant}{m}\sum_{i\in [m]}\excessRisk\left(\hat{f}_i \circ \randomProjection_i\right)\\
&=\frac{\quasiConvexityConstant}{m}\sum_{i\in [m]}\probDistribution(\hat{g}_{\randomProjection_i}) 
\nonumber\\
& \leq \quasiConvexityConstant\left\{
16\left( \bernsteinConstant   \nkDeltaComplexityTerm \right)^{\frac{1}{2-\bernsteinExponent}}+4\lossFunctionBound \cdot \nkDeltaComplexityTerm+4\compressibilityFunction_{\delta,m}(k)
\right\},
\end{align}
where 
$\mathfrak{C}_{n,\delta}\left(\randomZSequenceSizeN\right)=\left((4000 \coveringNumberConstant k)  \cdot \logBar\left(\lipschitzConstantLoss \hypothesisClassBound n\right)+{152\cdot\logBar(4\log(n)/\delta)}\right)\cdot n^{-1}$. By combining the above bound with Lemma \ref{finitepsi}, the result follows.
\end{proof}

\section{Proofs of minimax lower bounds}\label{Proofs4LB}
\newcommand{\qHammingDistance}{d_{\mathrm{H}}}
\newcommand{\totalVariation}{\mathrm{TV}}
\newcommand{\Ts}{\tilde{\alpha}}

This section proves the fundamental limits achievable for the distributional classes described in Definitions \ref{classificationMeasureClass} and \ref{regressionMeasureClass}. We shall begin by proving Theorem \ref{regRLBtext} in Section \ref{sec:proofOfRegressionLB} before moving onto the proof of Theorem \ref{clRLBtext}, which is given in Section \ref{sec:proofOfClassificationLB}. A key component in proving Theorem \ref{regRLBtext} is Lemma \ref{prop:mixtureLB}, which will also be applied in the proof of Theorem \ref{clRLBtext}. To prove our minimax lower bounds (Theorems \ref{clRLBtext} and \ref{regRLBtext}), we shall construct finite families of distributions for each $n \in \N$, such that (i) all distributions in the family satisfy the required conditions of the relevant distributional class, (ii) the distributions must be similar enough that they are difficult to identify based on an i.i.d. sample of size $n$, and (iii) they must be different enough so that failing to identify the generating distribution incurs a high excess risk.

Before presenting the proofs we first recall some useful terminology and results from the literature on minimax rates. For a comprehensive introduction see \cite{tsybakov2004introductiona}.  Given a convex function $f:\R \rightarrow \R$ with $f(1)=0$ and a pair of distributions $Q_0$ and $Q_1$ on a common measurable space $(\mathcal{Z},\mathcal{B})$ the $f$-divergence between $Q_0$ and $Q_1$ is  defined by
\begin{align*}
\mathrm{D}_f(Q_0,Q_1):=\int_{\mathcal{Z}} f\left(\frac{dQ_0}{dQ_1}\right)dQ_1,
\end{align*}
when $Q_0$ is absolutely continuous with respect to $Q_1$ and $\mathrm{D}_f(Q_0,Q_1)= \infty$ otherwise. The $\chi^2$-divergence $\chi^2(Q_0,Q_1)$ is the $f$-diverence with $f(z):=(z-1)^2$, the Kullback--Leibler divergence $\kullbackLeiblerDivergence(Q_0,Q_1)$ is the $f$-divergence with $f(z) := z\log z$. We shall also use the total variation distance, defined by
\begin{align*}
\totalVariation(Q_0,Q_1):=\sup_{A \in \mathcal{B}}|Q_0(A)-Q_1(A)|.
\end{align*}
The total variation, Kullback--Leibler and $\chi^2$-divergence are related as follows.

\begin{lemma}\label{lemma:boundTVWithChiSqr} Given distributions $Q_0$ and $Q_1$ on $(\mathcal{Z},\mathcal{B})$  we have 
\[2 \cdot \totalVariation^2(Q_0,Q_1)\leq \kullbackLeiblerDivergence(Q_0,Q_1) \leq {\chi^2(Q_0,Q_1)} .\]
\end{lemma}
\begin{proof}
This follows from \cite[Lemma 2.5  (i) combined with Lemma 2.7]{tsybakov2004introductiona}.
\end{proof}
We can use the following relationships to upper bound the divergence of products.
\begin{lemma}\label{lemma:probabilityDivergencesProducts} Given distributions $Q_0$ and $Q_1$ on $(\mathcal{Z},\mathcal{B})$  we have
\begin{align*}
\kullbackLeiblerDivergence\left(Q_0^{\otimes n},Q_1^{\otimes n}\right)&=n \cdot \kullbackLeiblerDivergence(Q_0,Q_1)\hspace{1cm}\text{  \&  }\hspace{1cm}
\totalVariation\left(Q_0^{\otimes n},Q_1^{\otimes n}\right)\leq n \cdot \totalVariation(Q_0,Q_1).
\end{align*}
\end{lemma}
\begin{proof}
See \cite[Section 2.4]{tsybakov2004optimal} for the result involving $\kullbackLeiblerDivergence$. See \cite[Lemma 2.1]{sendler1975note} for the total-variation inequality.
\end{proof}

Given $q \in \N$ we let $\Sigma \equiv \Sigma_q:= \{-1,1\}^q$ denote the $q$-dimensional binary hyper-cube and let $\qHammingDistance$ denote the associated Hamming distance i.e.
\begin{align*}
\qHammingDistance(\sigma,\sigma'):=\sum_{j \in [q]}\one\{\sigma_j \neq \sigma'_j\},
\end{align*}
for $\sigma = (\sigma_j)_{j \in [q]}$, $\sigma' = (\sigma'_j)_{j \in [q]}$ $\in \Sigma_q$. We shall leverage Assouad's lemma (Lemma \ref{assouad}), which is widely used for deriving minimax lower bounds \cite{tsybakov2004introductiona,wainwright2019high}.

\begin{lemma}[Assouad's lemma]\label{assouad} Let $(\mathcal{Z},\mathcal{B}_{\mathcal{Z}})$ be a measurable space and $\{\mathrm{Q}_\sigma\}_{\sigma \in \Sigma_q}$ a collection of probability measures on $(\mathcal{Z},\mathcal{B})$ such that  $\totalVariation(\mathrm{Q}_{\sigma},\mathrm{Q}_{\sigma'})\le 1/2$ for all $\sigma, \sigma'\in \Sigma_q$ with $\qHammingDistance(\sigma,\sigma') = 1$. Then, given any $\mathcal{B}_{\mathcal{Z}}$-measurable mapping $\hat{\sigma}:\mathcal{Z}\rightarrow \Sigma_q$ we have
\begin{align*}
\max_{\sigma \in \Sigma_q}~ \E_{\mathrm{Q}_\sigma}\left[\qHammingDistance(\hat{\sigma},\sigma)\right] \geq \frac{q}{4}.
\end{align*}
\end{lemma}
\begin{proof}
See e.g. \cite[Theorem 2.12  (ii)]{tsybakov2004introductiona}.
\end{proof}


\subsection{Proof of the regression lower bound (Theorem \ref{regRLBtext})}\label{sec:proofOfRegressionLB} 

We will begin by proving the simpler of the two lower bounds (Theorem \ref{regRLBtext}). The core of the proof is Lemma \ref{prop:mixtureLB}, which will also be applied in the proof of Theorem \ref{clRLBtext}. Lemma \ref{prop:mixtureLB} shows that given any distribution $\probDistribution_0$, any empirical predictor $\hat{\phi}$ and any mixing proportion $\zeta$, we can generate a mixture distribution $\probDistribution:=(1-\zeta) \cdot \probDistribution_0+\zeta \cdot \probDistribution_1$ with respect to which $\probDistribution$ incurs an excess error that is at least a constant multiple of the mixing proportion $\zeta$.  

\begin{lemma}[Mixture lower bound]\label{prop:mixtureLB}
Let $\lossFunction: \actionSpace \times \Y \rightarrow [0,\infty)$ be a loss function satisfying Assumptions \ref{lipschitzLossFunctionAssumption} and \ref{ass:nonDegenerateLoss} where $\actionSpace$ is a compact metric space. Fix a distribution $\probDistribution_0$ on $\X \times \Y$ and take $\zeta \in (0,1]$. Choose $\nonDegeneracyConstant>0$ and $y_0$,$y_1 \in \Y$ so that \eqref{eq:nonDegeneracyIneqAssumption} holds. Given $n \in \N$, suppose that $\X$ contains $q \geq 2\zeta n $ distinct points $\{x_1,\ldots,x_q\}\subseteq \X$ such that $ \Prob_{(X,Y) \sim \probDistribution_0}(X = x_j)=0$ for each $j \in [q]$.  Given any empirical predictor $\hat{\phi}:(\X\times\Y)^n\times \X\rightarrow \Y$, there exists a distribution $\probDistribution_1$ on $\X \times \Y$ such that;
\begin{enumerate}[label=(\alph*)]
    \item The distribution $\probDistribution_1$ is supported on $\{x_1,\ldots,x_q\}\times\{y_0,y_1\} \subseteq \X \times \Y$;
    \item For all $x \in \{x_1,\ldots,x_q\}$ there exists $y \in \{y_0,y_1\}$ with $\probDistribution_1(Y=y|X=x)=1$;
    \item If we let $\probDistribution:=(1-\zeta) \cdot \probDistribution_0+\zeta \cdot \probDistribution_1$  then $\E_{\sampleXY\sim P^{\otimes n}}[\excessRiskArg_{\lossFunction,\mathrm{\probDistribution}}(\hat{\phi})]
\ge \frac{\zeta \nonDegeneracyConstant}{4}$.
\end{enumerate}
\end{lemma}
\begin{proof} We begin by constructing a family of measures $(\probDistribution_{1,\sigma})_{\sigma \in \Sigma}$ where $\Sigma = \{-1,1\}^q$ as follows. For each $\sigma = (\sigma_j)_{j \in [q]}\in \Sigma$ we let
\begin{align}\label{eq:P1ConstructionInMixtureProp}
\probDistribution_{1,\sigma}(\{(x_j,y_{\ell})\})=\begin{cases} \frac{1}{q} &\text{ if } \ell = \frac{\sigma_j+1}{2} \\
0 &\text{ otherwise.} 
\end{cases}
\end{align}
Note that $\probDistribution_{1,\sigma}$ satisfies (a) for each $\sigma \in \Sigma$. In addition, if $\sigma, \sigma'\in \Sigma$ satisfy $\qHammingDistance(\sigma,\sigma') = 1$ then $\totalVariation(\probDistribution_{\sigma},\probDistribution_{\sigma'})\le \frac{1}{q}$. Hence, if we define a family of mixture distributions $(\probDistribution_{\sigma})_{\sigma \in \Sigma}$ by $\probDistribution_\sigma:=(1-\zeta) \cdot \probDistribution_0+\zeta \cdot \probDistribution_{1,\sigma}$ then given any $\sigma, \sigma'\in \Sigma$ with $\qHammingDistance(\sigma,\sigma') = 1$ by Lemma \ref{lemma:probabilityDivergencesProducts} we have
\begin{align}\label{eq:tvBoundToApplyAssouadInMixtureProp}
\totalVariation\left\lbrace (\probDistribution_\sigma)^{\otimes n},(\probDistribution_{\sigma'})^{\otimes n}\right\rbrace \leq n \cdot    \totalVariation\left( \probDistribution_\sigma ,\probDistribution_{\sigma'}\right) \leq n \zeta \cdot \totalVariation\left( \probDistribution_{1,\sigma} ,\probDistribution_{1,\sigma'}\right) \leq \frac{n \zeta}{q} \leq \frac{1}{2}.
\end{align}
Moreover, given $j \in [q]$, $\sigma=(\sigma_{j'})_{j' \in [q]} \in \Sigma$,  and letting $\ell_j =(\sigma_j+1)/2$ we have,
\begin{align*}
\Prob_{(X,Y) \sim \probDistribution_{\sigma}}(Y=y_{\ell_j}|X=x_{j})=\Prob_{(X,Y) \sim \probDistribution_{1,\sigma}}(Y=y_{\ell_j}|X=x_{j})=1,    
\end{align*}
since $ \Prob_{(X,Y) \sim \probDistribution_0}(X = x_j)=0$. Hence, (b) holds for $\probDistribution_\sigma$ and $\lossFunction\left({\phi}^*_{\sigma}(x_{j}),y_{\ell_j}\right)= \inf_{v \in \actionSpace}\lossFunction\left(v,y_{\ell_j}\right)$ where we have used Assumption \ref{lipschitzLossFunctionAssumption} combined with the fact that $\actionSpace$ is compact.

To prove (c) we define an estimator $\hat{\sigma}:(\X\times \Y)^n \rightarrow \Sigma$ by
\begin{align*}
\hat{\sigma}_j:= \sign\left\lbrace \left(\lossFunction(\hat{\phi}(x_j),y_0)-\inf_{v \in \actionSpace} \lossFunction(v,y_0)\right) - \left(\lossFunction(\hat{\phi}(x_j),y_1)-\inf_{v \in \actionSpace} \lossFunction(v,y_1)\right)\right\rbrace,
\end{align*}
and letting $\hat{\sigma}=(\hat{\sigma}_j)_{j \in [q]}$, which implicitly depends upon the sample. It follows that for each $j \in [q]$ with $\hat{\sigma}_j\neq \sigma_j$ and $\ell_j =(\sigma_j+1)/2$ we have
\begin{align}\label{eq:orderingExcessLossesByDistMixLB}
\lossFunction(\hat{\phi}(x_j),y_{\ell_j})-\inf_{v \in \actionSpace} \lossFunction(v,y_{\ell_j}) \geq  \lossFunction(\hat{\phi}(x_j),y_{1-\ell_j})-\inf_{v \in \actionSpace} \lossFunction(v,y_{1-\ell_j}).
\end{align}
By Assumption \ref{eq:nonDegeneracyIneqAssumption} combined with \eqref{eq:orderingExcessLossesByDistMixLB} we deduce that for $j \in [q]$ with $\hat{\sigma}_j\neq \sigma_j$ and $\ell_j =(\sigma_j+1)/2$
\begin{align*}
\lossFunction(\hat{\phi}(x_j),y_{\ell_j})- \lossFunction(\phi^*_{\sigma}(x_j),y_{\ell_j}) =  \lossFunction(\hat{\phi}(x_j),y_{\ell_j})-\inf_{v \in \actionSpace} \lossFunction(v,y_{\ell_j}) \geq \nonDegeneracyConstant.
\end{align*}
Hence, for each $\sigma \in \Sigma$ we have
\begin{align*}
\excessRiskArg_{\lossFunction,\probDistribution_{\sigma}}(\hat{\phi})\geq \frac{\zeta}{q}\sum_{j \in [q]} \left(\lossFunction(\hat{\phi}(x_j),y_{\ell_j})- \lossFunction(\phi^*_{\probDistribution^{\sigma}}(x_j),y_{\ell_j})\right)  \geq \frac{\zeta \nonDegeneracyConstant}{q} \cdot \qHammingDistance(\hat{\sigma},\sigma).
\end{align*}
By Assoud's lemma (Lemma \ref{assouad}) combined with \eqref{eq:tvBoundToApplyAssouadInMixtureProp} there exists at least one $\sigma \in \Sigma$ with 
\begin{align*}
\E_{\sampleXY\sim (P_{\sigma})^{ n}}[\excessRiskArg_{\lossFunction,\mathrm{\probDistribution}}(\hat{\phi})] \geq \frac{\zeta \nonDegeneracyConstant}{q} \cdot \E_{\sampleXY\sim (P_{\sigma})^{ n}}[ \qHammingDistance(\hat{\sigma},\sigma)] \geq \frac{\zeta \nonDegeneracyConstant}{4},
\end{align*}
as required.
\end{proof}

\begin{lemma}\label{lemma:sufficientMarginalConditionsToBelongToRegressionClass} Let $\lossFunction:[-\hypothesisClassBound,\hypothesisClassBound] \times \Y \rightarrow [0,\lossFunctionBound]$  be a loss function satisfying Assumptions \ref{boundedLossFunctionAssumption}, \ref{lipschitzLossFunctionAssumption} and \ref{ass:nonDegenerateLoss} and take $\Gamma=(\Wmax,\spectralConstant,\spectralDecay) \in [1,\infty)^2\times (0,1)$ with $\Wmax \geq 2\hypothesisClassBound\cdot \spectralDecay^{-1/2}$
and $\approxError \in [0,1]$. Let $\X$ be a Hilbert space containing a unit vector $e_1$ and a zero vector $0_\X$. Suppose that $\probDistribution$ is a distribution on $\X \times \Y$ with marginal $\probDistribution_X$ on $\X$ satisfying both (i) $\probDistribution_X(\{{0}_\X, \sqrt{\spectralDecay} \cdot e_1\}) \geq 1-\approxError/\lossFunctionBound$ and (ii) $\probDistribution_X(\{(t\sqrt{\spectralDecay}) \cdot e_1 \}_{t \in [0, 1]})=1$ . Then $\probDistribution \in \setOfMeasures_{\lossFunction}(\Gamma,\approxError)$.
\end{lemma}
\begin{proof} First choose $t_\circ:= \oracleFunction(0_\X) \in [-\hypothesisClassBound,\hypothesisClassBound]$ and $w_\circ := \spectralDecay^{-1/2}\cdot (\oracleFunction(\sqrt{\spectralDecay}\cdot e_1)-t_\circ)\cdot e_1$ so that $\|w_\circ\|_2 =\spectralDecay^{-1/2}\cdot |\oracleFunction(\sqrt{\spectralDecay}\cdot e_1)-t_\circ| \leq 2\hypothesisClassBound \cdot \spectralDecay^{-1/2}\leq \Wmax$. Moreover, if we define $\phi_\circ(x):=\min\{\hypothesisClassBound,\max\{-\hypothesisClassBound,w_\circ^\top x+t_\circ\}\}$ then for $x \in \{{0}_\X, \sqrt{\spectralDecay} \cdot e_1\}$ we have $\phi_\circ(x)=\oracleFunction(x)$. Hence, by (i) we have
\[\excessRisk(\phi_\circ)\leq \lossFunctionBound \cdot \probDistribution_X(\X \backslash \{{0}_\X, \sqrt{\spectralDecay} \cdot e_1\})  \leq \approxError.\]
Thus, $\probDistribution$ satisfies Assumption \ref{ass:linearAproximationCondition} with parameters $(\Wmax,\approxError)$.

Second, it follows from (ii) that the covariance operator $\Upsilon$ corresponding to $\probDistribution_X$ has singular values $\lambda_1(\Upsilon)\leq \spectralDecay \leq  \spectralConstant \cdot \spectralDecay$ and $\lambda_r(\Upsilon)=0$ for $r \in \N \backslash\{1\}$. Thus, Assumption \ref{spectraldecay} also holds. By combining these two conclusions we see that $P\in \setOfMeasures_{\lossFunction}(\Gamma,\approxError)$.
\end{proof}

\begin{proof}[Proof of Theorem \ref{regRLBtext}] We begin by choosing $y_0 \in \Y$ and defining a distribution $\probDistribution_0$ on $\X\times \Y$ so that $\probDistribution_0(A)=\one\{(0_\X,y_0) \in A\}$ for Borel sets $A \subseteq \X \times \Y$. Note also that since $\X$ is a Hilbert space containing a non-zero element we may choose a unit vector $e_1 \in \X$.  We now consider two cases. 

First suppose that $\approxError \geq n^{-1} >0$. Take $\zeta:=\approxError/\lossFunctionBound$, $q := \lceil 2 \zeta n\rceil$, and choose a set of $q$ distinct points $\{x_1,\ldots,x_q\} \subset \{(t\sqrt{\spectralDecay}) \cdot e_1 \}_{t \in (0, 1]}$. By Lemma \ref{prop:mixtureLB} there exists a distribution $\probDistribution_1$ with marginal $(\probDistribution_1)_X$ supported on $\{x_1,\ldots,x_q\}$ such that $\probDistribution=(1-\zeta) \probDistribution_0+\zeta \probDistribution_1$ satisfies
\begin{align*}
\E_{\sampleXY\sim P^n}[\excessRiskArg_{\lossFunction,\mathrm{\probDistribution}}(\hat{\phi})]
\ge \frac{\zeta \nonDegeneracyConstant}{4} = \frac{\approxError \nonDegeneracyConstant}{4\lossFunctionBound} \geq \frac{\nonDegeneracyConstant}{8\lossFunctionBound}\cdot (n^{-1}+\approxError).  
\end{align*}
Moreover, since (i) $\probDistribution_X(\{{0}_\X\})= 1-{\approxError}/{\lossFunctionBound}$ and (ii)  $\probDistribution_X(\{0_\X\}\cup\{x_1,\ldots,x_q\})=1$, it follows from Lemma \ref{lemma:sufficientMarginalConditionsToBelongToRegressionClass} that $\probDistribution \in \setOfMeasures_{\lossFunction}(\Gamma,\approxError)$.

Now suppose that $\approxError < n^{-1}$. We apply Lemma \ref{prop:mixtureLB} once again with $\zeta = 1/(2n)$, $q=1$ and $x_1:=\spectralDecay^{-1/2}\cdot e_1$ to obtain a a distribution $\probDistribution_1$ with marginal $(\probDistribution_1)_X$ supported on $\{x_1\}$ such that the mixture $\probDistribution=(1-\zeta) \probDistribution_0+\zeta \probDistribution_1$ satisfies
\begin{align*}
\E_{\sampleXY\sim P^n}[\excessRiskArg_{\lossFunction,\mathrm{\probDistribution}}(\hat{\phi})]
\ge \frac{\zeta \nonDegeneracyConstant}{4} = \frac{ \nonDegeneracyConstant}{8 n} \geq  \frac{ \nonDegeneracyConstant}{16 }  \cdot (n^{-1}+\approxError). 
\end{align*}
Moreover, since $\probDistribution_X(\{{0}_\X,\spectralDecay^{-1/2}\cdot e_1\})= 1$ it follows from Lemma \ref{lemma:sufficientMarginalConditionsToBelongToRegressionClass} that $\probDistribution \in \setOfMeasures_{\lossFunction}(\Gamma,\approxError)$.
\end{proof}

\subsection{Proof of the classification lower bound (Theorem \ref{clRLBtext})}\label{sec:proofOfClassificationLB}

In this section we prove Theorem \ref{clRLBtext}. In order to apply Assoud's lemma (Lemma \ref{assouad}) we shall first construct a parameterised family of distributions. We then establish sufficient conditions for these distributions to belong to the class $\setOfMeasures_{\text{0,1}}(\Gamma,\approxError)$.

First suppose that $n \in \N\backslash\{1\}$. Given $q \in \{1,\ldots,n-1\}$, $r \in [1,\sqrt{q}]$, $v\in(0,1)$, $\epsilon \in (0,1/2)$ we shall define a family of measures $(\probDistribution^{\sigma})_{\sigma \in \Sigma}$ on $\X \times \Y$ indexed by $\Sigma = \{-1,1\}^q$. Our construction will begin by first choosing a set of $q$ ``difficult to classify'' points in $\X$, with total mass $v$ and norm $r$. We will require $r \leq \sqrt{q}$ to ensure that the points can be classified with sufficiently large margin. More precisely, let $\{e_0,e_1,\ldots,e_{q}\}\subseteq \X$ be a collection of $q+1$ orthonormal vectors. We let $x_0:=e_0$ and for $\ell \in [q]$ let $x_\ell:=r\cdot e_\ell$. All of our measures $\probDistribution^\sigma$ will share a common marginal distribution $\mu$ on $\X$ supported on $\{x_0,x_1,\ldots,x_q\}$ and defined by
\begin{align}\label{def:marginalConstructionMainClassificationLB}
\mu(\{x_\ell\})\equiv \mu_{q,r,v}(\{x_\ell\}):=\begin{cases} \frac{v}{q} & \text{ if } \ell \in [q]\\
1-v &\text{ if }\ell = 0.
\end{cases}
\end{align}
For each $\sigma = (\sigma_\ell)_{\ell \in [q]}\in \Sigma$ we define an associated regression function $\regressionFunction^\sigma: \X \rightarrow $ by
\begin{align}\label{def:regressionFunctionMainClassificationLB}
\regressionFunction^\sigma(x_\ell) \equiv \regressionFunction^\sigma_{q,r,v,\epsilon}(x_\ell):= \frac{1+\epsilon\cdot \sigma_\ell}{2},
\end{align}
for each $\ell \in [q]$ and $\regressionFunction^\sigma(x) = 1$ for $x \notin \{x_1,\ldots,x_q\}$. The family of measures $(\probDistribution^{\sigma})_{\sigma \in \Sigma}$ can now be defined by taking \begin{align}\label{def:probConstructionMainClassificationLB}
\probDistribution^\sigma(\{x,y\})\equiv \probDistribution^\sigma_{q,r,v,\epsilon}(\{x,y\}):= \mu(\{x\})\cdot \eta^{\sigma}(x)^{\frac{1+y}{2}}\cdot \{1-\eta^\sigma(x)\}^{\frac{1-y}{2}},
\end{align}
for all $(x,y) \in \X\times \{-1,+1\}$.

\begin{lemma}\label{lemma:classificationMainLBBelongToClass} Suppose that $ \Gamma=(  (\bernsteinExponent, \tysbakovNoiseConstant),
(\geometricMarginConstant, \geometricMarginExponent),(\momentConstant,\momentExponent))$ where $\bernsteinConstant$, $\momentConstant\geq 1$, $\geometricMarginConstant \ge  2^{\frac{\geometricMarginExponent}{2}}$, $\bernsteinExponent \in [0,1]$, $\geometricMarginExponent$, $\momentExponent \in (0, \infty)$ and take $\approxError \in [0,1]$. Then for each $\sigma\in\Sigma$,
  \begin{enumerate}[label={(\alph*)}]
      \item $\probDistribution^{\sigma}$ satisfies Assumption \ref{geometricMarginAssumption} with parameters $(\geometricMarginConstant,\geometricMarginExponent,\approxError)$ provided that $ \epsilon \cdot v \le \geometricMarginConstant  \left({r}/{\sqrt{2q}}\right)^{\geometricMarginExponent}$;
      \item $\probDistribution^{\sigma}$ satisfies Assumption \ref{momentAssumption} with parameters $(\momentConstant,\momentExponent)$ provided that $\epsilon\cdot v \le \momentConstant \cdot r^{-\momentExponent}$;
      \item $\probDistribution^{\sigma}$ satisfies Assumption \ref{tsybakovNoiseAssumption} with parameters $(\tysbakovNoiseConstant,\bernsteinExponent)$ provided that $v\le  \tysbakovNoiseConstant\cdot  \epsilon^{\frac{\bernsteinExponent}{1-\bernsteinExponent}}$.
  \end{enumerate}
\end{lemma}
\begin{proof} We fix $\sigma = (\sigma_\ell)_{\ell\in [q]}\in \Sigma$ and consider each assumption in turn.

\noindent(a) Suppose $ \epsilon \cdot v \le \geometricMarginConstant \cdot \left({r}/{\sqrt{2q}}\right)^{\geometricMarginExponent}$. Define $w_\circ\equiv w_\circ(\sigma):=\frac{1}{\sqrt{2}}\bigl(e_0+\frac{1}{\sqrt{q}}\sum_{\ell \in [q]}\sigma_\ell \cdot e_\ell\bigr)\in \X$, take $t_\circ:=0$ and let ${\classifier^{\circ}}(x):=\sign(w_{\circ}^{\top}x -t_{\circ})$ for $x \in \X$. Observe that $\|w_\circ\|_2 =1$. Moreover, it follows from \eqref{def:regressionFunctionMainClassificationLB} that 
\begin{align*}
\oracleFunction_{\probDistribution^\sigma}(x_\ell) = \sign(2\regressionFunction^{\sigma}(x_\ell)-1) = \sigma_\ell = \sign(w_\circ^\top x_\ell+t_\circ)= \phi_\circ(x_\ell),
\end{align*}
for each $\ell \in [q]$. Similarly, $\oracleFunction_{\probDistribution^\sigma}(x_0) = 1= \phi_\circ(x_0)$. Since $\mu=(\probDistribution^{\sigma})_X$ is supported $\{x_0,x_1,\ldots,x_q\}$ it follows that $\excessRisk(\phi_\circ)= 0\leq \approxError$. To complete the proof we must take $\nearDecisionBoundary:=\{x \in \X\hspace{0.5mm}:\hspace{0.5mm} |w_{\circ}^{\top}x-t_{\circ}|\leq \xi\}$  and show that 
\begin{align}\label{eq:mainClaimInGeomMarginAssumptionPfMainLBClassification}
\int_{\nearDecisionBoundary} \left|2{{\regressionFunction}}(x)-1\right| d{\mathrm{P}}_X(x)  \leq\geometricMarginConstant \cdot \xi^{\geometricMarginExponent},
\end{align}  
for each $\xi>0$. Observe that $|w_{\circ}^{\top}x_0-t_{\circ}|=2^{-1/2}$ and $|w_{\circ}^{\top}x_\ell-t_{\circ}|=r(2q)^{-{1}/{2}} \leq 2^{-1/2}$ for $\ell \in [q]$. We shall consider three cases. If (i) $\xi \in (0,r (2q)^{-{1}/{2}})$ then $\nearDecisionBoundary \cap \supp(\mu)=\emptyset$ so \eqref{eq:mainClaimInGeomMarginAssumptionPfMainLBClassification} holds. If (ii) $\xi \in [r\cdot (2q)^{-{1}/{2}},2^{-1/2})$ then $\nearDecisionBoundary \cap \supp(\mu)=\{x_1,\ldots,x_q\}$ and so 
\begin{align*}
\int_{\nearDecisionBoundary} \left|2{{\regressionFunction}}(x)-1\right| d{\mathrm{P}}_X(x)  = \sum_{\ell \in [q]} \mu(\{x_\ell\}) \cdot \left|2{{\regressionFunction}}(x)-1\right| = v \cdot \epsilon \leq \geometricMarginConstant \cdot \left(\frac{r}{\sqrt{2q}}\right)^{\geometricMarginExponent} \leq \geometricMarginConstant \cdot \xi^{\geometricMarginExponent},
\end{align*}  
as required. Finally, if (iii) $\xi \geq 2^{-1/2}$ then \eqref{eq:mainClaimInGeomMarginAssumptionPfMainLBClassification} follows from the assumption that $\geometricMarginConstant \ge  2^{\frac{\geometricMarginExponent}{2}}$. Hence, Assumption \ref{geometricMarginAssumption} holds with parameters $(\geometricMarginConstant,\geometricMarginExponent,\approxError)$.

\noindent(b) Suppose $\epsilon\cdot v \le \momentConstant \cdot r^{-\momentExponent}$. To prove the claim it suffices to show that for all $s \in (0,\infty)$,
\begin{align}\label{eq:momentBoundToProveMainLBForClassification}
\int_{\|x\|>s} \left|2{{\regressionFunction}}(x)-1\right|d \mathrm{P}_X(x)  \leq\momentConstant \cdot s^{-\momentExponent}
\end{align}
Recall that $\supp(\mu)=\supp\{(\probDistribution^{\sigma})_X\}=\{x_0,x_1,\ldots,x_q\}$ and note that $\|x_0\|=1$ and $\|x_\ell\|=r \geq 1$ for $\ell \in [q]$. As before, we consider three cases. If (i) $s \in (0,1]$ then \eqref{eq:momentBoundToProveMainLBForClassification} follows from $\momentConstant \ge 1$. If (ii) $s \in (1,r)$ then 
\begin{align*}
\int_{\|x\|>s} \left|2{{\regressionFunction}}(x)-1\right|d \mathrm{P}_X(x)  = \sum_{\ell \in [q]} \mu(\{x_\ell\}) \cdot \left|2{{\regressionFunction}}(x)-1\right| = v \cdot \epsilon \leq  \momentConstant \cdot r^{-\momentExponent} \leq\momentConstant \cdot s^{-\momentExponent},
\end{align*}
as required. Finally, if (iii) $s \in [r,\infty)$ then $\mu(\{x \in \X:\|x\|>s\})=0$ so \eqref{eq:momentBoundToProveMainLBForClassification} holds.

\noindent(c) Now suppose that $v\le  \tysbakovNoiseConstant\cdot  \epsilon^{\frac{\bernsteinExponent}{1-\bernsteinExponent}}$. Observe that $\left|{2{\regressionFunction}}(x_0)-1\right|=1$ and $\left|{2{\regressionFunction}}(x_\ell)-1\right|=\epsilon$ for $\ell \in [q]$. To prove the claim it suffices to show that for all $\zeta \in (0,1)$ we have 
\begin{align}\label{eq:marginBoundToProveMainLBForClassification}
\mu\left(\{ x \in \X:\left|{2{\regressionFunction}}(x)-1\right|\leq \zeta\}\right) \leq \tysbakovNoiseConstant \cdot \zeta^{\frac{\bernsteinExponent}{1-\bernsteinExponent}}.
\end{align}
We consider two cases. If $\zeta  \in (0,\epsilon)$ then $ \{ x \in \supp(\mu):\left|{2{\regressionFunction}}(x)-1\right|\leq \zeta\}=\emptyset$ so \eqref{eq:marginBoundToProveMainLBForClassification} holds. On the other hand, if $\zeta  \in [\epsilon,1)$ then 
\begin{align*}
\mu\left(\{ x \in \X:\left|{2{\regressionFunction}}(x)-1\right|\leq \zeta\}\right)=\mu(\{x_1,\ldots,x_q\})=v \leq \tysbakovNoiseConstant \cdot \epsilon^{\frac{\bernsteinExponent}{1-\bernsteinExponent}} \leq \tysbakovNoiseConstant \cdot \zeta^{\frac{\bernsteinExponent}{1-\bernsteinExponent}}.
\end{align*}
\end{proof}

\begin{lemma}\label{lemma:classificationLBMainSmallHammingDistance} Given 
 $\sigma, \sigma'\in \Sigma_q$ with $\qHammingDistance(\sigma,\sigma') = 1$ we have  $\chi^2(\probDistribution^{\sigma},\probDistribution^{\sigma'}) \leq {2^4\epsilon^2v}/q$, and hence $\totalVariation\{(\probDistribution^{\sigma})^{\otimes n},(\probDistribution^{\sigma'})^{\otimes n}\} \leq 1/2$ whenever ${2^5n\epsilon^2v}\leq {q}$.
\end{lemma}
\begin{proof} Choose $\ell_0 \in [q]$ so that $\sigma_{\ell_0}\neq \sigma_{\ell_0}'$ where $\sigma = (\sigma_\ell)_{\ell \in [q]}$ and  $\sigma' = (\sigma'_\ell)_{\ell \in [q]}$. Note that since $\qHammingDistance(\sigma,\sigma') = 1$ there is exactly one such $\ell_0 \in [q]$. Observe that by construction $\mu=(\probDistribution^{\sigma})_X=(\probDistribution^{\sigma'})_X$ is supported on $\{x_0,x_1,\ldots,x_q\}$ and $\regressionFunction^\sigma(x)=\regressionFunction^{\sigma'}(x)$ for $x \in \supp\{(\probDistribution^{\sigma'})_X\}\backslash\{x_{\ell_0}\}$. In addition, we have $\probDistribution^\sigma(\{x,y\})= \mu(\{x\})\cdot \eta^{\sigma}(x)^{\frac{1+y}{2}}\cdot \{1-\eta^\sigma(x)\}^{\frac{1-y}{2}}$ for all $x \in \X$ and $y \in \{-1,1\}$. It follows that $\probDistribution^\sigma(\{(x,y)\})=\probDistribution^{\sigma'}(\{(x,y)\})$ for all $x \in \supp\{(\probDistribution^{\sigma'})_X\}\backslash\{x_{\ell_0}\}$. In addition,  since $\epsilon \in (0,1/2)$, we have
\begin{align*}
\left( \frac{\probDistribution^{\sigma}(\{(x_{\ell_0},1)\})}{\probDistribution^{\sigma'}(\{(x_{\ell_0},1)\})}-1\right)^2=\left( \frac{\eta^{\sigma}(x_{\ell_0})}{ \eta^{\sigma'}(x_{\ell_0})}-1\right)^2=\left\lbrace \frac{(1+\epsilon\cdot \sigma_{\ell_0})-(1+\epsilon\cdot \sigma'_{\ell_0})}{{1+\epsilon\cdot \sigma'_{\ell_0}}}\right\rbrace^2\leq 2^4\epsilon^2,
\end{align*}
and similarly $\left( \frac{\probDistribution^{\sigma}(\{(x_{\ell_0},-1)\})}{\probDistribution^{\sigma'}(\{(x_{\ell_0},-1)\})}-1\right)^2 \leq 2^4\epsilon^2$. Hence,
\begin{align*}
\chi^2\left(\probDistribution^{\sigma},\probDistribution^{\sigma'}\right) &= \int_{\X\times \Y} \left( \frac{d\probDistribution^{\sigma}}{d\probDistribution^{\sigma'}}\bigg|_{(x,y)}-1\right)^2 d\probDistribution^{\sigma'}(x,y)\\  &= \sum_{\ell =0}^q\sum_{y \in \{-1,1\}}\left( \frac{\probDistribution^{\sigma}(\{(x_\ell,y)\})}{\probDistribution^{\sigma'}(\{(x_\ell,y)\})}-1\right)^2 \probDistribution^{\sigma'}(\{(x_\ell,y)\})\\ 
&= \sum_{y \in \{-1,1\}}\left( \frac{\probDistribution^{\sigma}(\{(x_{\ell_0},y)\})}{\probDistribution^{\sigma'}(\{(x_{\ell_0},y)\})}-1\right)^2 \probDistribution^{\sigma'}(\{(x_{\ell_0},y)\})\\ 
&\leq 2^4\epsilon^2 \sum_{y \in \{-1,1\}}\left( \mu(\{x_{\ell_0}\})\cdot \eta^{\sigma}(x_{\ell_0})^{\frac{1+y}{2}}\cdot \{1-\eta^\sigma(x_{\ell_0})\}^{\frac{1-y}{2}}\right) =\frac{2^4\epsilon^2v}{q}.
\end{align*}
Hence, by Lemmas \ref{lemma:boundTVWithChiSqr} and \ref{lemma:probabilityDivergencesProducts} we have
\begin{align*}
2\cdot \totalVariation^2\{(\probDistribution^{\sigma})^{\otimes n},(\probDistribution^{\sigma'})^{\otimes n}\} &\leq \kullbackLeiblerDivergence\{(\probDistribution^{\sigma})^{\otimes n},(\probDistribution^{\sigma'})^{\otimes n}\} \\ &\leq n \cdot \kullbackLeiblerDivergence\left(\probDistribution^{\sigma},\probDistribution^{\sigma'}\right)  \leq n \cdot \chi^2\left(\probDistribution^{\sigma},\probDistribution^{\sigma'}\right) \leq \frac{2^4n\epsilon^2v}{q}.
\end{align*}
Consequently, $\totalVariation\{(\probDistribution^{\sigma})^{\otimes n},(\probDistribution^{\sigma'})^{\otimes n}\}\leq 1/2$ provided ${2^5n\epsilon^2v}\leq {q}$.
\end{proof}

\begin{lemma}\label{lemma:hammingLossAndExcessError} Given a classifier ${\phi}:\X \rightarrow \{-1,1\}$ and $\sigma = (\sigma_\ell)_{\ell \in [q]} \in \Sigma$ we have
\begin{align*}
\excessRiskArg_{\zeroOneLoss,\probDistribution^\sigma}({\phi}) \geq \frac{\epsilon v}{q}\sum_{\ell \in [q]}\one\left\lbrace \phi(x_\ell) \neq \sigma_\ell\right\rbrace.   
\end{align*}
\end{lemma}
\begin{proof} Observe that for each $\ell \in [q]$ we have $\mu(\{x_\ell\})=v/q$, $\regressionFunction^\sigma(x_\ell)=(1+\epsilon\cdot \sigma_\ell)/2$ and hence $\probDistribution^\sigma$ has Bayes classifier $\oracleFunction_{\probDistribution^\sigma}(x_\ell)=\sigma_\ell$. Thus, we may lower bound the excess error by
\begin{align*}
\excessRiskArg_{\zeroOneLoss,\probDistribution^\sigma}({\phi}) & = \int_\X |2\regressionFunction^\sigma(x)-1| \cdot \one\{\phi(x) \neq \oracleFunction_{\probDistribution^\sigma}(x)\} \geq  \frac{\epsilon v}{q}\sum_{\ell \in [q]}\one\left\lbrace \phi(x_\ell) \neq \sigma_\ell\right\rbrace.
\end{align*}
\end{proof}

To prove Theorem \ref{clRLBtext} we shall consider three cases. The first case we consider will occur when the small approximation error $\approxError$ and large sample size $n$ regime in which the lower bound is attained by the family $(\probDistribution_{q,r,v,\epsilon}^{\sigma})_{\sigma \in \Sigma}$, for an appropriate choice of $q$,$r$, $v$, $\epsilon$. The lower bound in the second case in which both the approximation error $\approxError$ and the sample size $n$ are small will be deduced from the first. In the third case, where the approximation error $\approxError$ is large, we will deduce the lower bound from Lemma \ref{prop:mixtureLB}.

\begin{proof}[Proof of Theorem \ref{clRLBtext}]  First suppose that $n^{-\frac{\momentExponent\geometricMarginExponent}{2(\momentExponent+\geometricMarginExponent)+\momentExponent\geometricMarginExponent(2-\bernsteinExponent)}} \geq \approxError$ and
\[n \geq n_0 \equiv n_0(\bernsteinExponent,\geometricMarginExponent,\momentExponent):= \max\left\lbrace 1+2^{\frac{6\{2(\geometricMarginExponent+\momentExponent)+\geometricMarginExponent\momentExponent(2-\bernsteinExponent)\}}{\momentExponent\geometricMarginExponent(2-\bernsteinExponent)}},   2^{\frac{2(\geometricMarginExponent+\momentExponent)+\geometricMarginExponent\momentExponent(2-\bernsteinExponent)}{\momentExponent\geometricMarginExponent(1-\bernsteinExponent)}}        \right\rbrace.\]
Now define $q \in \N$, $r >0$ and $v$, $\epsilon \in (0,1)$ by 
\begin{align*}
q:=\lceil (2^5n)^{\frac{2(\geometricMarginExponent+\momentExponent)}{2(\geometricMarginExponent+\momentExponent)+\momentExponent\geometricMarginExponent (2-\bernsteinExponent)}}\rceil,\hspace{5mm} r:=q^{\frac{\geometricMarginExponent}{2(\geometricMarginExponent+\momentExponent)}},\hspace{5mm}v:=q^{-\frac{\momentExponent\geometricMarginExponent\bernsteinExponent}{2(\geometricMarginExponent+\momentExponent)}},\hspace{5mm} \epsilon:= q^{-\frac{\geometricMarginExponent\momentExponent (1- \bernsteinExponent)}{2(\geometricMarginExponent+\momentExponent)}}.
\end{align*}
Note that since $n \geq n_0$  we have $q <n$, $r \in [1,\sqrt{q}]$ and $\epsilon \in (0,1/2)$. Hence, we may apply the construction outlined in \eqref{def:marginalConstructionMainClassificationLB},
\eqref{def:regressionFunctionMainClassificationLB},
\eqref{def:probConstructionMainClassificationLB} to construct our family of distributions $(\probDistribution^{\sigma})_{\sigma \in \Sigma}$. Note that by construction (a) $\epsilon\cdot v = (r/\sqrt{q})^{\geometricMarginExponent} \leq \geometricMarginConstant  \left({r}/{\sqrt{2q}}\right)^{\geometricMarginExponent}$, (b) $\epsilon\cdot v = r^{-\momentExponent} \le \momentConstant \cdot r^{-\momentExponent}$ and (c)  $v=\epsilon^{\frac{\bernsteinExponent}{1-\bernsteinExponent}}\le  \tysbakovNoiseConstant\cdot  \epsilon^{\frac{\bernsteinExponent}{1-\bernsteinExponent}}$. Hence, by Lemma \ref{lemma:classificationMainLBBelongToClass} we have $\{\probDistribution^\sigma\}_{\sigma \in \Sigma} \subseteq \setOfMeasures_{\text{0,1}}(\Gamma,\approxError)$. Moreover, ${2^5n\epsilon^2v}\leq {q}$ so by Lemma \ref{lemma:classificationLBMainSmallHammingDistance} we have $\totalVariation\{(\probDistribution^{\sigma})^{\otimes n},(\probDistribution^{\sigma'})^{\otimes n}\} \leq 1/2$ for all $\sigma, \sigma'\in \Sigma_q$ with $\qHammingDistance(\sigma,\sigma') = 1$. Next we associate to our empirical classifier  $\hat{\phi}:(\X\times\Y)^n\times \X\rightarrow \Y$ a random binary $\hat{\sigma}= (\hat{\sigma}_\ell)_{\ell \in [q]}$, taking values in $\Sigma$, by $\hat{\sigma}_\ell = \hat{\phi}(x_\ell)$ for all $\ell \in [q]$. By Assouad's lemma (Lemma \ref{assouad}) there exists some $\sigma^\circ \in \Sigma$ with $\E_{(\probDistribution^{\sigma^\circ})^{\otimes n}}\left[\qHammingDistance(\hat{\sigma},\sigma^\circ)\right] \geq \frac{q}{4}$. Thus, by Lemma \ref{lemma:hammingLossAndExcessError},
\begin{align*}
\E_{(\probDistribution^{\sigma^{\circ}})^{\otimes n}}\left[\excessRiskArg_{\zeroOneLoss,\probDistribution^{\sigma^\circ}}({\phi}) \right]&\geq \frac{\epsilon v}{q}\cdot \E_{(\probDistribution^{\sigma^\circ})^{\otimes n}}\biggl[ \sum_{\ell \in [q]}\one\left\lbrace \phi(x_\ell) \neq \sigma^\circ_\ell\right\rbrace\biggr]= \frac{\epsilon v}{q}\cdot \E_{(\probDistribution^{\sigma^\circ})^{\otimes n}}\bigl[\qHammingDistance(\hat{\sigma},\sigma^\circ)\bigr]\\ &\geq \frac{\epsilon v}{4} =\frac{1}{4}\cdot q^{-\frac{\momentExponent\geometricMarginExponent}{2(\geometricMarginExponent+\momentExponent)}} \geq 2^{-\left(5+\frac{\momentExponent\geometricMarginExponent}{2(\geometricMarginExponent+\momentExponent)}\right)} \cdot n^{-\frac{\momentExponent\geometricMarginExponent}{2(\momentExponent+\geometricMarginExponent)+\momentExponent\geometricMarginExponent(2-\bernsteinExponent)}},
\end{align*}
provided $n \geq n_0$ and $n^{-\frac{\momentExponent\geometricMarginExponent}{2(\momentExponent+\geometricMarginExponent)+\momentExponent\geometricMarginExponent(2-\bernsteinExponent)}} \geq \approxError$. Moreover, any empirical classifier $\hat{\phi}:(\X\times\Y)^n\times \X\rightarrow \Y$ based on a sample of size $n \in \{1,\ldots,n_0\}$ may be viewed as a special case of an empirical classifier with a sample size $n_0$ which disregards a fraction  of the data. Hence, for all $n \in \N$ with $n^{-\frac{\momentExponent\geometricMarginExponent}{2(\momentExponent+\geometricMarginExponent)+\momentExponent\geometricMarginExponent(2-\bernsteinExponent)}} \geq \approxError$ there exists $\probDistribution \in \setOfMeasures_{\text{0,1}}(\Gamma,\approxError)$ with 
\begin{align}\label{eq:conclusionMainPartClassificationLB}
\E_{(\probDistribution)^{\otimes n}}\left[\excessRiskArg_{\zeroOneLoss,\probDistribution}({\phi}) \right] &\geq 2^{-\left(5+\frac{\momentExponent\geometricMarginExponent}{2(\geometricMarginExponent+\momentExponent)}\right)} \cdot (n_0 \cdot n)^{-\frac{\momentExponent\geometricMarginExponent}{2(\momentExponent+\geometricMarginExponent)+\momentExponent\geometricMarginExponent(2-\bernsteinExponent)}} \nonumber \\ &\geq 2^{-\left(6+\frac{\momentExponent\geometricMarginExponent}{2(\geometricMarginExponent+\momentExponent)}\right)} \cdot n_0^{-\frac{\momentExponent\geometricMarginExponent}{2(\momentExponent+\geometricMarginExponent)+\momentExponent\geometricMarginExponent(2-\bernsteinExponent)}}\cdot \{n^{-\frac{\momentExponent\geometricMarginExponent}{2(\momentExponent+\geometricMarginExponent)+\momentExponent\geometricMarginExponent(2-\bernsteinExponent)}}+\approxError\}.
\end{align}
Now suppose that $n^{-\frac{\momentExponent\geometricMarginExponent}{2(\momentExponent+\geometricMarginExponent)+\momentExponent\geometricMarginExponent(2-\bernsteinExponent)}} < \approxError$. Note that the zero-one loss $\zeroOneLoss$ satisfies Assumption \ref{ass:nonDegenerateLoss} with $y_0=-1$, $y_1=1$ and $\nonDegeneracyConstant =1/2$. Let $\probDistribution_0$ denote the distribution on $\X\times \Y$ defined by $\probDistribution_0(A):=\one\{(0_\X,1) \in A\}$ for Borel sets $A \subseteq \X \times \Y$. Note also that since $\X$ is a Hilbert space of dimension at least $n \geq 1$ we may choose a unit vector $e_1 \in \X$.  Now set $x_\ell:=\frac{\ell}{q}\cdot e_1$ for $\ell \in [q]$. Thus, by Lemma \ref{prop:mixtureLB} there exists a distribution $\probDistribution_1$ supported on $\{x_1,\ldots,x_q\}\times \{-1,+1\}$, with regression function $\eta_{\probDistribution_1}(x)=\Prob_{(X,Y)\sim \probDistribution_1}(Y=1|X=x)\in \{0,1\}$ for all $x \in \{x_1,\ldots,x_1\}$ such that the mixture $\probDistribution:=(1-\approxError) \cdot \probDistribution_0+\approxError \cdot \probDistribution_1$ satisfies 
\begin{align*}
\E_{\sampleXY\sim P^n}[\excessRiskArg_{\lossFunction,\mathrm{\probDistribution}}(\hat{\phi})]
\geq \frac{\approxError}{8} >  2^{-4} \cdot  \left( n^{-\frac{\momentExponent\geometricMarginExponent}{2(\momentExponent+\geometricMarginExponent)+\momentExponent\geometricMarginExponent(2-\bernsteinExponent)}}+ \approxError\right).
\end{align*}
Note that $\probDistribution(\{(0,1)\})=\approxError$ so Assumption \ref{geometricMarginAssumption} holds with $w_\circ = 0_{\X}$ and $t_\circ =1$. Moreover, Assumption \ref{momentAssumption} also holds since $\probDistribution_X$ is supported on $\{0_\X\}\cup\{x_\ell\}_{\ell \in [q]} \subseteq \{x \in \X:\|x\| \leq \}$. Finally Assumption \ref{tsybakovNoiseAssumption} holds since $\regressionFunction_{\probDistribution}(x)= \Prob_{(X,Y)\sim \probDistribution}(Y=1| X=x) \in \{0,1\}$ for all $x \in \supp(\probDistribution_X)$. By combining with \eqref{eq:conclusionMainPartClassificationLB} we see that \eqref{eq:mainClaimClassificationLB} holds with $ c_{\ref{clRLBtext}}:=2^{-\left(6+\frac{\momentExponent\geometricMarginExponent}{2(\geometricMarginExponent+\momentExponent)}\right)} \cdot n_0^{-\frac{\momentExponent\geometricMarginExponent}{2(\momentExponent+\geometricMarginExponent)+\momentExponent\geometricMarginExponent(2-\bernsteinExponent)}}$.
\end{proof}

\section{Conclusions}
We presented a general analytic framework for the theoretical study of compressive learning with arbitrary size ensembles, on arbitrary high dimensional data sets. 
This yields high probability risk guarantees that are able to take advantage of both statistical and geometric structure in the underlying data distribution. We demonstrated in our framework that new conditions can be unearthed specifically for compressive learning, which differ from those of both compressive sensing and traditional learning, and which lead to a better understanding of compressive learning. In particular, we found some distributional characteristics under which the worst-case excess risk of voting compressive ERMs nearly attain the minimax-optimal rate w.r.t. the sample size.  To our knowledge, these are the first statistical optimality guarantees for compressive ERM learning machines and ensembles thereof.
In addition, a key ingredient in the proof of our general upper bound may find applications in other areas.

Our results shed light on the question of when and why compressive learning with Johnson-Lindenstrauss compressors works well, and provides new insights that may eventually inform data pre-processing and approximate algorithm design in future work.
Several questions remain for future research: 
How to make practical use of the theoretically optimal values of $k$? How to extend the analysis to other data sketching schemes? What other predictor classes enjoy similar guarantees?

\section*{Acknowledgment}
Both authors were funded by EPSRC grant EP/P004245/1 ``FORGING: Fortuitous Geometries and
Compressive Learning" at the University of Birmingham, where a significant proportion of this project was undertaken.

\appendix

\section{Verifying assumptions for the examples}\label{verifying}

In this section we verify the properties asserted in Section \ref{statisticalSettingSec}.

\subsection{Bounded Lipschitz loss functions}\label{appendixForBoundedLipschitzAssumptionSubSec}

\begin{prop}\label{verifyingLipschitzBoundedAssumptionsProp} Examples \ref{e1}-\ref{e3}  satisfy Assumptions \ref{boundedLossFunctionAssumption} and \ref{lipschitzLossFunctionAssumption}, as follows:
\begin{enumerate}
    \item The zero-one loss $\zeroOneLoss$ satisfies Assumption \ref{boundedLossFunctionAssumption} with $\lossFunctionBound =1$ and Assumption \ref{lipschitzLossFunctionAssumption} with the standard metric $\metricActionSpace(v_0,v_1)=|v_0-v_1|$ and Lipschitz constant $\lipschitzConstantLoss = 1/2$.
    \item 
    The {squared} loss $\sqrLoss$ satisfies Assumption \ref{boundedLossFunctionAssumption} with $\lossFunctionBound =4\hypothesisClassBound^2$ and Assumption \ref{lipschitzLossFunctionAssumption} with the standard metric and $\lipschitzConstantLoss = 4 \hypothesisClassBound$.
    \item 
    The \emph{Kullback Leibler} loss function $\klLoss$ satisfies Assumption \ref{boundedLossFunctionAssumption} with $\lossFunctionBound = \hypothesisClassBound+\log(2)$, and Assumption \ref{lipschitzLossFunctionAssumption} with the standard metric on $[-\hypothesisClassBound,\hypothesisClassBound]$ and $\lipschitzConstantLoss = 1$.
\end{enumerate}
\end{prop}
\begin{proof}[Proof of Proposition \ref{verifyingLipschitzBoundedAssumptionsProp}]
\begin{enumerate}
\item Since $\zeroOneLoss(v,y) \in \{0,1\}$, Assumption \ref{boundedLossFunctionAssumption} with $\lossFunctionBound =1$ is immediate. Moreover, if $\zeroOneLoss(v_0,y) \neq \zeroOneLoss(v_1,y)$ then $|v_0-v_1|=2$ so Assumption \ref{lipschitzLossFunctionAssumption} holds with $\lipschitzConstantLoss =1/2$.
\item If $v,y \in [-\hypothesisClassBound,\hypothesisClassBound]$ then $|v-y|\leq 2\hypothesisClassBound$ so $\sqrLoss(v,y) =(v-y)^2 \leq 4\hypothesisClassBound^2$. Moreover, $\frac{\partial}{\partial v}\left(\sqrLoss(v,y)\right) = 2(v-y)$ so by the mean value theorem $\sqrLoss$ is $4\hypothesisClassBound$-Lipschitz.
\item 
Given $a\in\actionSpace=[-\hypothesisClassBound,\hypothesisClassBound]$, 
\begin{align}
\klLoss(a,y) &= y\log\frac{y}{\pi(a)}+(1-y)\log\frac{1-y}{1-\pi(a)}
\le \log\left(\frac{y^2}{\pi(a)}+\frac{(1-y)^2}{1-\pi(a)} \right)\\
&\le \max\left(\log\frac{1}{\pi(a)},\log\frac{1}{1-\pi(a)}\right)
\le \max\{\log(1+\exp(-a)),\log(1+\exp(a))\}\\
&\le \hypothesisClassBound + \log(2).
\label{H}
\end{align}
This confirms Assumption \ref{boundedLossFunctionAssumption} with $\lossFunctionBound= \hypothesisClassBound+\log(2)$. 
Furthermore, 
\begin{align}
\klLoss(a,y)&=y\log\frac{y}{\pi(a)}+(1-y)\log\frac{1-y}{1-\pi(a)}
=-[ya - \log(1+\exp(a))]\\
&=\log(1+\exp(-(2y-1)a))
\end{align}
so the first derivative is bounded as  $|\frac{\partial}{\partial a}\klLoss(a,y)|= 
|y-\frac{\exp(a)}{1+\exp(a)}|\le 1$.
By the mean value theorem, this confirms Assumption \ref{lipschitzLossFunctionAssumption} with $\lipschitzConstantLoss=1$ w.r.t. the standard metric on $[-\hypothesisClassBound,\hypothesisClassBound]$.
\end{enumerate}
\end{proof}

\subsection{Quasi-convex loss functions}\label{proofsForaveragingFunctionsAndQuasiConvexityAssumptionSubSecAppendix}

\begin{lemma}\label{verifyingQuasiConvexityAssumptionsLemma} Examples \ref{e1}-\ref{e3} satisfy Assumption \ref{quasiConvexityAssumption} for any distribution $\probDistribution$: 
\begin{enumerate}
    \item The zero-one loss $\zeroOneLoss$ satisfies Assumption \ref{quasiConvexityAssumption} with $\quasiConvexityConstant = 2$, and $\Avg$ equal to the modal average: $\Avg(\{\functionGeneral_i(x)\}_{i\in[m]})=\sign(\sum_{i\in[m]}\functionGeneral_i(x))$;
    \item The {squared} loss $\sqrLoss$ satisfies Assumption \ref{quasiConvexityAssumption} with $\quasiConvexityConstant=1$, and $\Avg$ equal to the arithmetic average of the bounded-linear outputs;
    \item The Kullback Leibler loss $\klLoss$ satisfies Assumption \ref{quasiConvexityAssumption} with $\quasiConvexityConstant=1$, and $\Avg$ equal to the arithmetic average of the bounded-linear outputs.
\end{enumerate}
\end{lemma}

\begin{proof} To prove the first claim we  take $\eta(x)=\Prob[Y=1|X=x]$ and $\probDistribution_X(A)=\Prob(X \in A)$. So, $\oracleFunction(x) = \one(\eta(x)\ge 1/2)\cdot 2-1$ and we can write
\begin{align}\label{standardExpressionForRelativeRiskWithZeroOneLossEq}
\risk\left(h\right)- \risk\left(\oracleFunction\right)&=\int |\eta(x)-1/2|\cdot \one\left\lbrace {h(x)\neq \oracleFunction(x)}\right\rbrace d\probDistribution_X(x).
\end{align}
Given $\{\phi_i\}_{i \in [m]} \in \measurableMaps(\X,\actionSpace)^m$, if for some $x \in \X$, $\Avg\left(\{\phi_i\}_{i \in [m]}\right)(x) \neq \oracleFunction(x)$, where $\Avg$ denotes the modal average, then we have $\phi_i(x)\neq \oracleFunction(x)$ for at least $m/2$ values of $i \in [m]$. Thus,
\begin{align*}
 \one\left\lbrace {\Avg\left(\{\phi_i\}_{i \in [m]}\right)(x)\neq \oracleFunction(x)}\right\rbrace  \leq \frac{2}{m} \cdot  \sum_{i \in [m]}\one\left\lbrace {\phi_i}(x)\neq \oracleFunction\right\rbrace.
\end{align*}
Integrating over $\probDistribution_X$ and substituting into \eqref{standardExpressionForRelativeRiskWithZeroOneLossEq} proves the claim.

 Since $\sqrLoss$ is convex the second result is a consequence of Lemma \ref{convexImpliesQuasiConvex}. Similarly, the function $\klLoss(a,y)=\text{kl}(\pi(a),y)$ is convex in $a$. Hence the result again follows from Lemma \ref{convexImpliesQuasiConvex}.
\end{proof}

We note that for the $\klLoss$, the arithmetic average is equivalent to the product-of-experts combination of the nonlinear probabilistic outputs.

\subsection{Covering number assumptions}\label{appendixForCoveringNumberAssumptionSubSec}

Let us recall some useful terminology and foundational results on \emph{psuedo-dimension} which may be found in \cite{anthony2009neural}.

\begin{defn}[Pseudo-shattering]\label{pseudoShatteringDef} Take $\F \subseteq \measurableMaps(\X,\R)$. A set $\xSequenceSizeN = \{x_j\}_{j \in [n]} \in \X^n$ is said to be \emph{pseudo-shattered} by $\F$ if there exists real numbers $\bm{r}_{1:n}=\{r_j\}_{j\in [n]}$ such that for every $\sigmaSequenceSizeN = \{\sigma_j\}_{j \in [n]} \in \{-1,1\}^n$ there exists $f_{\sigmaSequenceSizeN} \in \F$ with $\sign\left(f_{\sigmaSequenceSizeN}(x_j)-r_j\right)=\sigma_j$ for every $j \in [n]$. 
\end{defn}

The pseudo-dimension is the natural analogue of the VC dimension for real-valued functions.

\begin{defn}[Pseudo-dimension]\label{pseudoDimensionDef} Take $\F \subseteq \measurableMaps(\X,\R)$. The \emph{pseudo-dimension} of $\F$, denoted by $\pseudoDimension(\F)$, refers to the maximum cardinality $n \in \N$ of a set $\xSequenceSizeN \in \X^n$ which is pseudo-shattered by $\F$. If there are sets of arbitrarily large cardinality which are pseudo-shattered by $\F$ then we say that $\F$ has infinite pseudo-dimension $\pseudoDimension(\F)=\infty$.
\end{defn}

We use the following well known results.

\begin{theorem}\label{pseudoDimensionAffineFunctionsBartlettThm} Let $\F_{\text{linear}} \subseteq \measurableMaps\left(\R^k,\R\right)$ be the set of affine functions
\begin{align*}
\F_{\text{linear}}:=\left\lbrace \bm{u} \mapsto \bm{w} \cdot \bm{u}-t:\hspace{2mm}\bm{w}\in \R^k,\hspace{1mm}t \in \R\right\rbrace.    
\end{align*}
Then $\F_{\text{linear}}$ has pseudo-dimension $\pseudoDimension\left(\F_{\text{linear}}\right) =k+1$.
\end{theorem}
\begin{proof} See \cite[Theorem 11.6]{anthony2009neural}.
\end{proof}

\begin{theorem}\label{pseudoDimensionAndNonDecreasingFunctionsBartlettThm} Suppose that $\F \subseteq \measurableMaps\left(\X,\R\right)$ and let $g \in \measurableMaps\left(\R,\R\right)$ be a non-decreasing function. Then the set $g\left(\F\right):=\left\lbrace g \circ f: f \in \F \right\rbrace$ has pseudo-dimension $\pseudoDimension\left(g\left(\F\right)\right) \leq \pseudoDimension\left(\F\right)$.
\end{theorem}
\begin{proof} See \cite[Theorem 11.3]{anthony2009neural}.
\end{proof}
Given any function class $\F \subseteq \measurableMaps(\X,\R)$ and any $\xSequenceSizeN = \{x_j\}_{j \in [n]} \in \X^n$, the empirical $\ell_{\infty}$ metric $\dist_{\xSequenceSizeN}^{\infty}$ on $\F$ is defined for $f_0$, $f_1 \in \F$ by
\begin{align*}
\dist_{\xSequenceSizeN}^{\infty}(f_0,f_1):=\max_{j \in [n]}\left\lbrace \left|f_0(x_j)-f_1(x_j)\right|\right\rbrace.
\end{align*} 

\begin{theorem}\label{pseudoDimensionAndCoveringNumbersBartlettThm} Suppose $\hypothesisClassBound>0$, $d \in \N$ and $\F \subseteq \measurableMaps_{\hypothesisClassBound}\left(\X\right)$ has pseudo-dimension $\pseudoDimension\left(\F\right) \leq d$. Then for all $n\geq d$, $\xSequenceSizeN \in \X^n$ and $\epsilon>0$,
\begin{align*}
\log \left(\coveringNumber\left(\F,\dist_{\xSequenceSizeN}^{\infty},\epsilon\right)\right) \leq   d \cdot \left(\logBar \left(\frac{ \hypothesisClassBound n}{\epsilon \cdot d }\right)+2\right).
\end{align*}
\end{theorem}
\begin{proof} By \cite[Theorem 12.2]{anthony2009neural} for $\epsilon \in [0,2\hypothesisClassBound)$ we have
\begin{align*}
\log \left(\coveringNumber\left(\F,\dist_{\xSequenceSizeN}^{\infty},\epsilon\right)\right)  \leq   d \cdot \log \left(\frac{2e \hypothesisClassBound n}{\epsilon \cdot d }\right) \leq   d \cdot \left(\logBar \left(\frac{ \hypothesisClassBound n}{\epsilon \cdot d }\right)+2\right).
\end{align*}
For $\epsilon \geq 2 \hypothesisClassBound$ the claim follows immediately from $\F \subseteq \measurableMaps_{\hypothesisClassBound}(\X)$ so $\coveringNumber\left(\F,\dist_{\xSequenceSizeN}^{\infty},\epsilon\right) \leq 1$.
\end{proof}

We now verify Assumption \ref{logarithmicCoveringNumbersAssumption} for the function classes used in our examples, $\lowDimensionalFunctionClass^{\text{0,1}}$ (as in Example \ref{e1}) and $\lowDimensionalFunctionClass^{{bl}}$ (as in Examples \ref{e2}-\ref{e3}). We note that other nonlinear transformations of the linear function class may also satisfy this assumption provided that the nonlinearity only changes the log covering number by a constant factor. 

\begin{prop}\label{verifyingLogCovering}
Examples \ref{e1}-\ref{e3} satisfy Assumption \ref{logarithmicCoveringNumbersAssumption}. More precisely:
\begin{enumerate}
    \item The class $\lowDimensionalFunctionClass^{\text{0,1}}=\left\lbrace \bm{u} \mapsto \sign\left(\bm{w} \cdot \bm{u}-t\right):\hspace{2mm}\bm{w}\in \R^k,\hspace{1mm}t \in \R\right\rbrace \subseteq \measurableMaps\left(\R^k,\{-1,+1\}\right)$ satisfies Assumption \ref{logarithmicCoveringNumbersAssumption} with constant $\coveringNumberConstant=4$ and bound $\hypothesisClassBound=1$.
    \item The class $\lowDimensionalFunctionClass^{{bl}}=\left\lbrace \bm{u} \mapsto \max\left\lbrace \min\left\lbrace \bm{w} \cdot \bm{u}-t,\hypothesisClassBound\right\rbrace,-\hypothesisClassBound\right\rbrace:\hspace{2mm}\bm{w}\in \R^k,\hspace{1mm}t \in \R\right\rbrace \subseteq \measurableMaps_{\hypothesisClassBound}\left(\R^k\right)$ satisfies Assumption \ref{logarithmicCoveringNumbersAssumption} with constant $\coveringNumberConstant=4$ and bound $\hypothesisClassBound$.
\end{enumerate}
\end{prop}

\begin{proof}[Proof of Proposition \ref{verifyingLogCovering}]
\begin{enumerate}
    \item By Theorem \ref{pseudoDimensionAffineFunctionsBartlettThm}, $\pseudoDimension\left(\F_{\text{linear}}\right)=k+1$. Moreover, the mapping $t \mapsto \sign(t)$ is non-decreasing, so by Theorem \ref{pseudoDimensionAndNonDecreasingFunctionsBartlettThm}, $\pseudoDimension\left(\lowDimensionalFunctionClass^{\text{0,1}}\right) \leq k+1$. Note that for all $\uSequenceSizeN$, and any $f_0,f_1$ we have $\dist_{\uSequenceSizeN}^{\actionSpace}(f_0,f_1) \leq \dist_{\uSequenceSizeN}^{\infty}(f_0,f_1)$, where $\dist_{\uSequenceSizeN}^{\actionSpace}$ is the empirical $\ell_2$ metric. 
    By Theorem \ref{pseudoDimensionAndCoveringNumbersBartlettThm}, for every $n, k \in \N$ with $n > k$, $\uSequenceSizeN \in (\R^k)^n$, and $\epsilon>0$,
\begin{align*}
\log \left(\coveringNumber\left(\lowDimensionalFunctionClass^{\text{0,1}},\dist_{\uSequenceSizeN}^{\actionSpace},\epsilon\right)\right) &\leq \log \left(\coveringNumber\left(\lowDimensionalFunctionClass^{\text{0,1}},\dist_{\uSequenceSizeN}^{\infty},\epsilon\right)\right) \\ & \leq (k+1) \cdot \left(\logBar \left(\frac{ n}{\epsilon \cdot (k+1) }\right)+2\right) \leq 4k\cdot \logBar \left(\frac{ n}{\epsilon \cdot k }\right).
\end{align*}
\item  Since $\pseudoDimension\left(\F_{\text{linear}}\right)=k+1$ and the mapping $t \mapsto \max\{\min\{t,\hypothesisClassBound\},-\hypothesisClassBound\}$ is non-decreasing, Theorem \ref{pseudoDimensionAndNonDecreasingFunctionsBartlettThm} implies $\pseudoDimension\left(\lowDimensionalFunctionClass^{{bl}}\right) \leq k+1$. By Theorem \ref{pseudoDimensionAndCoveringNumbersBartlettThm} for every $n, k \in \N$ with $n > k$, $\uSequenceSizeN \in \left(\R^k\right)^n$ and $\epsilon>0$,
\begin{align}\label{infNormCoveringBoundForCheckingLogCoverAsssumption}
\log \left(\coveringNumber\left(\lowDimensionalFunctionClass^{{bl}},\dist_{\uSequenceSizeN}^{\infty},\epsilon\right)\right) & \leq (k+1) \cdot \left(\logBar \left(\frac{ \hypothesisClassBound n}{\epsilon \cdot (k+1) }\right)+2\right) \leq 4k\cdot \logBar \left(\frac{ \hypothesisClassBound n}{\epsilon \cdot k }\right).
\end{align}
Hence, since $\dist_{\uSequenceSizeN}^{\actionSpace} \leq \dist_{\uSequenceSizeN}^{\infty}$ we have $\log \left(\coveringNumber\left(\lowDimensionalFunctionClass^{{bl}},\dist_{\uSequenceSizeN}^{\actionSpace},\epsilon\right)\right) \leq  4k\cdot \logBar \left(\frac{ \hypothesisClassBound n}{\epsilon \cdot k }\right)$.
\end{enumerate}
\end{proof}

\subsection{Bernstein-Tsybakov condition}\label{appendixForBernsteinSubSec}

We now consider Assumption \ref{bernsteinTysbakovMarginTypeAssumption}. For example \ref{e1}, the 0-1 loss $\zeroOneLoss$ (example \ref{e1}), the Tsybakov noise condition implies the Bernstein condition \cite[Proposition 1]{tsybakov2004optimal}. For the squared loss and the Kullback-Leibler loss we utilise the following lemma.

\begin{lemma}\label{lemma:BernsteinConditionForLipschitzStronglyConvexLosses} Suppose that $\lossFunction: [-\hypothesisClassBound,\hypothesisClassBound]\times \Y \rightarrow [0,\lossFunctionBound]$ satisfies Assumptions \ref{lipschitzLossFunctionAssumption} and \ref{stronglyConvexLossAssumption}. Then satisfy the Bernstein-Tsybakov condition with exponent $\bernsteinExponent =1$ and constant $\bernsteinConstant = {4\lipschitzConstantLoss^2}/{\strongConvexityConstant}$.
\end{lemma}

\begin{proof}
Choose  $\functionGeneral \in \measurableMaps(\X,\actionSpace)$ and define $\functionGeneral_{1/2} \in \measurableMaps(\X,\actionSpace)$ by $\functionGeneral_{1/2}(x) = (\functionGeneral(x)+\oracleFunction(x))/2$ for $x \in \X$. By Assumption \ref{stronglyConvexLossAssumption} we have
\begin{align*}
\left(\functionGeneral(x)-\oracleFunction(x)\right)^2 \leq \frac{4}{\strongConvexityConstant} \left(\left\lbrace\lossFunction(\functionGeneral(x),y)-\lossFunction(\oracleFunction(x),y)\right\rbrace+2\left\lbrace \lossFunction\left(\functionGeneral_{1/2}(x),y\right)-\lossFunction(\oracleFunction(x),y)\right\rbrace\right)
\end{align*} 
for all $x \in \X$ and $y \in \Y$. Hence, by applying Assumption \ref{lipschitzLossFunctionAssumption} we have
\begin{align*}
\int\left\lbrace \lossFunction\left(\functionGeneral(x),y\right)-\lossFunction\left(\oracleFunction(x),y\right)\right\rbrace^2d\probDistribution(x,y) & \leq  \lipschitzConstantLoss^2 \int\left(\functionGeneral(x)-\oracleFunction(x)\right)^2d\probDistribution_X(x)\\
& \leq  \frac{4\lipschitzConstantLoss^2}{\strongConvexityConstant} \left(\excessRisk(\functionGeneral)-2\excessRisk(\functionGeneral_{1/2})\right) \\ &\leq \frac{4\lipschitzConstantLoss^2}{\strongConvexityConstant} \excessRisk(\functionGeneral),
\end{align*} 
as required.
\end{proof}

\begin{lemma}\label{lemma:striclyPositiveSecondDerivativeImpliesStronglyConvex} A twice differentiable function $\varphi:[a,b]\rightarrow \R$ with second derivative $\varphi'' \geq \strongConvexityConstant$ is  $\strongConvexityConstant$-strongly convex on $[a,b]$.
\end{lemma}
\begin{proof} Take $v_0$, $v_1 \in [a,b]$ with $v_0 \leq v_1$, let $v_t:=tv_0+(1-t)v_1$ and $\Delta:=v_1-v_0$. By a second order Taylor expansion we have at $x_t$ we have
\begin{align*}
\varphi(v_0) &\geq \varphi(v_t)+\varphi'(v_t)(v_0-v_t)+\frac{\strongConvexityConstant}{2}(v_0-v_t)^2=\varphi(v_t)-\varphi'(v_t)\Delta t+\frac{\strongConvexityConstant}{2}\Delta^2t^2\\
\varphi(v_1) &\geq \varphi(v_t)+\varphi'(v_t)(v_1-v_t)+\frac{\strongConvexityConstant}{2}(v_1-v_t)^2=\varphi(v_t)+\varphi'(v_t)\Delta(1-t)+\frac{\strongConvexityConstant}{2}\Delta^2(1-t)^2.
\end{align*}
By rearranging we have $\varphi(v_t) \leq (1-t) \varphi(v_0)+t\varphi(v_1) - ({\strongConvexityConstant}/{2})\cdot t(1-t)\cdot \Delta^2$, as required.
\end{proof}

\begin{corollary}\label{verifyingBernstein} The square loss $\sqrLoss$ and the Kullback Leibler loss $\klLoss$ satisfy Assumption \ref{bernsteinTysbakovMarginTypeAssumption} with exponent $\bernsteinExponent=1$.
\end{corollary}

\begin{proof} By Lemmas \ref{lemma:BernsteinConditionForLipschitzStronglyConvexLosses} and \ref{lemma:striclyPositiveSecondDerivativeImpliesStronglyConvex} it suffices to show that both $u\mapsto \sqrLoss(u,y)$ and $\klLoss(u,y)$ are twice differentiable with a uniform lower bound of $\strongConvexityConstant$. For the square loss we have $\frac{\partial^2}{\partial u^2}\{\sqrLoss(u,y)\}\equiv 2$ and for the Kullback-Leibler loss we have
\[    \frac{\partial^2}{\partial u^2}\{\klLoss(u,y)\}=\frac{\exp(u)}{(1+\exp(u))^2}\geq \frac{\exp(\hypothesisClassBound)}{(1+\exp(\hypothesisClassBound))^2},\] 
for all $u\in[-\hypothesisClassBound,\hypothesisClassBound]$ and $y\in\{0,1\}$, as required.
\end{proof}

\section{Local Rademacher concentration inequality} \label{A1}
In this section we give a local Rademacher concentration inequality for the deviation of an empirical process from its expectation (Theorem \ref{localRademacherConcentrationInequality}), which is used in the proof of Theorem \ref{ThmMainUniform}. In our context, it corresponds to the special case of Theorem \ref{ThmMainUniform} where $\Omega$ consists of a single point. We note that several similar results are available in the literature \cite{massart2000some,massart2006risk,koltchinskii2006local,bartlett2005local,boucheron2013concentration} which highlight the role of variance in obtaining tight concentration guarantees. However, we for completeness we provide a proof Theorem \ref{localRademacherConcentrationInequality} which is well-suited to our setting. We shall begin by recalling several useful results required for the proof. The first is Bartlett et al's variance dependent Rademacher bound.

\begin{theorem}[\cite{bartlett2005local}]\label{rademacherConcentrationWithVarianceBound} Suppose we have a class of functions $\G\subseteq \measurableMaps_{1}(\Z)$. Suppose $Z$ is a random variable taking values in $\Z$ with distribution $\probDistributionZ$ and take $\variance:= \sup_{g \in \G}\{\int (g-\int gd\probDistribution)^2d\probDistribution\}$. Given $n \in \N$ we let $\randomZSequenceSizeN=\{Z_j\}_{j \in [n]}$ be a sequence of independent random variables with distribution $\probDistributionZ$. Given any $\delta \in (0,1)$, the following holds with probability at least $1-\delta$ over $\randomZSequenceSizeN$ we have
\begin{align*}
\sup_{{g} \in {\G}}\left\lbrace \left|\empiricalProbDistributionRandomZ(g)-\probDistribution(g)\right|\right\rbrace \leq 6 \cdot  \rademacherComplexityEmpirical\left(\G,\randomZSequenceSizeN\right)+2\sqrt{ \frac{\variance\log(1/\delta)}{n}}+\frac{11\log(1/\delta)}{n}.
\end{align*}
\end{theorem}

We shall also use a variant of Dudley's inequality which allows us to control the Rademacher complexity of a function class in terms of its covering numbers \cite{dudley1967sizes}. Given a function class $\G\subseteq \measurableMaps(\Z,\R)$ and a sequence $\zSequenceSizeN = \{z_j\}_{j \in [n]}\in \Z^n$ we define a data dependent metric $\dist_{\zSequenceSizeN}$ on $\G$ by $\dist_{\zSequenceSizeN}(g_0,g_1) = \sqrt{\empiricalProbDistributionDeterministicZ\left\lbrace\left(g_0-g_1\right)^2\right\rbrace}$, for $g_0$, $g_1 \in \G$. The following refinement of Dudley's inequality due to Srebro and Sridharan \cite{srebro2010note}.

\begin{theorem}[\cite{dudley1967sizes,srebro2010note}]\label{dudleysInequalityTheorem} Suppose we have a function class $\G\subseteq \measurableMaps(\Z,\R)$ and a sequence $\zSequenceSizeN = \{z_j\}_{j \in [n]}\in \Z^n$, we have the following bound,
\begin{align*}
\rademacherComplexityEmpirical\left(\G,\zSequenceSizeN\right) \leq \inf_{\epsilon>0}\left\lbrace 4\epsilon+\int_{\epsilon}^{\sup_{g \in \G}\left\lbrace \sqrt{\empiricalProbDistributionDeterministicZ\left(g^2\right)}\right\rbrace} \sqrt{\frac{\log \coveringNumber\left(\G, \dist_{\zSequenceSizeN} ,\epsilon \right) }{n}}\right\rbrace.
\end{align*}
\end{theorem}

We also utilize Talagrand's contraction inequality \cite{ledoux2013probability}.

\begin{lemma}[\cite{ledoux2013probability}]\label{talagrandsContractionInequalityLemma} Suppose that $\varphi:\R \rightarrow \R$ is a $L$-Lipschitz function and $\G \subseteq \measurableMaps(\Z)$ is a function class. Then for any $\zSequenceSizeN \in \Z^n$ we have 
$\rademacherComplexityEmpirical\left(\left\lbrace \varphi \circ g: g \in \G\right\rbrace,\zSequenceSizeN \right) \leq L \cdot \rademacherComplexityEmpirical\left( \G, \zSequenceSizeN\right)$. 
\end{lemma}
With these results in hand we give the following concentration inequality which adapts ideas from  \cite{massart2000some,massart2006risk,koltchinskii2006local,bartlett2005local,boucheron2013concentration} to our setting.
\begin{theorem}[Local Rademacher concentration inequality]\label{localRademacherConcentrationInequality}
Suppose we have a countable class of functions $\G\subseteq \measurableMaps_{1}(\Z)$ and a function $\phi:\Z^n \times (0,\infty) \rightarrow (0,\infty)$ such that for each $\zSequenceSizeN \in \Z^n$, $r\mapsto \phi(\zSequenceSizeN,r)$ is a non-decreasing function such that $r \mapsto \phi(\zSequenceSizeN,r)/\sqrt{r}$ is non-increasing and for all $r>0$ and $\zSequenceSizeN = \{z_j\}_{j \in [n]} \in \Z^n$,
\begin{align*}
\rademacherComplexityEmpirical\left( \left\lbrace g \in \G: {\frac{1}{n}\sum_{j \in [n]}g^2(z_j)}\leq r\right\rbrace,\zSequenceSizeN\right) \leq \phi(\zSequenceSizeN,r).
\end{align*}
For each $\zSequenceSizeN \in \Z^n$ we choose $\rho^*(\zSequenceSizeN)\in (0,\infty)$ so that $\phi(\zSequenceSizeN,\rho^*(\zSequenceSizeN)))={\rho^*(\zSequenceSizeN)}$.  Suppose we have a sequence of independent random variables $\randomZSequenceSizeN=\{Z_j\}_{j \in [n]}$  with common distribution $\probDistributionZ$. Given any $\delta \in (0,1)$, the following holds with probability at least $1-\delta$ over $\randomZSequenceSizeN$ for all $g \in \G$,
{\small
\begin{align*}
\left|\empiricalProbDistributionRandomZ(g)-\probDistribution(g)\right| &\leq    2\cdot \sqrt{\probDistribution(g^2)\cdot\left( 72\cdot \rho^*(\randomZSequenceSizeN)+\frac{ 2\log(4\log(n)/\delta)}{n}\right)}+ \left(132 \cdot \rho^*(\randomZSequenceSizeN)+\frac{30\log(4\log(n)/\delta)}{n}\right).
\end{align*}
}
\end{theorem}
Before the proof we recall the following elementary lemma.
\begin{lemma}\label{standardBoundingByTheRootLemma} Suppose that $x$, $A$, $B$ $>0$ satisfy $x \leq A\sqrt{x}+B$. Then $x \leq A^2 +2B$.
\end{lemma}

\begin{proof}[Proof of Theorem \ref{localRademacherConcentrationInequality}] For each $k \in \{ 1,\cdots, \lfloor \log_2(n) \rfloor\}$ we let $v_k=2^{k+1}/n$. We let $\G_1 = \{g \in \G: \probDistribution(g^2)\leq v_1\}$ and for each $k \in  \{2,\cdots, \lfloor \log_2(n) \rfloor\}$ we let $\G_k = \{ g \in \G: v_{k-1} < \probDistribution(g^2)\leq v_k\}$. Observe that $v_{ \lfloor \log_2(n) \rfloor}\geq 1$ and since $\G \subseteq \measurableMaps_1(\Z)$ we have $\probDistribution(g^2) \leq 1$ for all $g \in \G$, so $\G \subseteq \bigcup_{k = 1}^{\lfloor \log_2(n) \rfloor}\G_k$. Fix $k \in \{1,\cdots, {\lfloor \log_2(n) \rfloor}\}$. By Theorem \ref{rademacherConcentrationWithVarianceBound} we see that with probability at least $1-\delta$, for all $g \in \G_k$ we have
\begin{align}\label{firstEmpiricalDevPfOflocalRademacherConcentrationInequality}
 \left|\empiricalProbDistributionRandomZ(g)-\probDistribution(g)\right| &\leq 6 \cdot  \rademacherComplexityEmpirical\left(\G_k,\randomZSequenceSizeN\right)+2 \sqrt{ \frac{v_k\log(1/\delta)}{n}}+\frac{11\log(1/\delta)}{n}.
\end{align} Let $\G_k^2 = \{g^2: g \in \G_k\}$. Note also that $z\mapsto z^2$ is $2$-Lipschitz so by Lemma \ref{talagrandsContractionInequalityLemma} we have $\rademacherComplexityEmpirical\left(\G_k^2,\randomZSequenceSizeN\right)\leq 2\rademacherComplexityEmpirical\left(\G_k,\randomZSequenceSizeN\right)$. Note also that for $g \in \G_k$, $\probDistribution((g^2)^2)\leq \probDistribution(g^2) \leq v_k$, since $g$ is bounded by 1, so applying Theorem \ref{rademacherConcentrationWithVarianceBound} once again, we see that with probability at least $1-\delta$, for all $g \in \G_k$ we have
\begin{align}\label{secondEmpiricalDevPfOflocalRademacherConcentrationInequality}
 \left|\empiricalProbDistributionRandomZ(g^2)-\probDistribution(g^2)\right| &\leq 12 \cdot  \rademacherComplexityEmpirical\left(\G_k,\randomZSequenceSizeN\right)+2\sqrt{ \frac{v_k\log(1/\delta)}{n}}+\frac{11\log(1/\delta)}{n}.
\end{align}
Thus, by the union bound (\ref{firstEmpiricalDevPfOflocalRademacherConcentrationInequality}) and (\ref{secondEmpiricalDevPfOflocalRademacherConcentrationInequality}) both hold for all  $k \in \{1,\cdots, {\lfloor \log_2(n) \rfloor}\}$ with probability at least $1-2\log_2(n)\cdot \delta\geq 1-4\log(n) \cdot \delta$. Observe that for all $k \in \{1,\cdots, {\lfloor \log_2(n) \rfloor}\}$  and $g \in \G_k$, $\probDistribution(g^2)\leq \variance_k$, so given  (\ref{secondEmpiricalDevPfOflocalRademacherConcentrationInequality}) we have
\begin{align*}
\empiricalProbDistributionRandomZ(g^2) &\leq \variance_k  +12 \cdot  \rademacherComplexityEmpirical\left(\G_k,\randomZSequenceSizeN\right)+2\sqrt{ \frac{v_k\log(1/\delta)}{n}}+\frac{11\log(1/\delta)}{n}.
\end{align*}
Let $u_k:= \variance_k  +12 \cdot  \rademacherComplexityEmpirical\left(\G_k,\randomZSequenceSizeN\right)+2\sqrt{ {v_k\log(1/\delta)}/{n}}+{11\log(1/\delta)}/{n}$. We deduce that $\G_k \subseteq \{g \in \G: \empiricalProbDistributionRandomZ(g^2)\leq u_k\}$ and so 
\begin{align}\label{rademacherLeqPhiBoundIneq}
\rademacherComplexityEmpirical(\G_k,\randomZSequenceSizeN) \leq     
\rademacherComplexityEmpirical\left( \left\lbrace g \in \G: {\frac{1}{n}\sum_{j \in [n]}g^2(z_j)}\leq u_k\right\rbrace,\zSequenceSizeN\right) \leq \phi\left(\randomZSequenceSizeN,u_k\right).
\end{align}
Plugging back into the previous inequality gives
\begin{align*}
u_k \leq  \variance_k + 12 \cdot \phi({\randomZSequenceSizeN},u_k)+2\sqrt{ \frac{v_k\log(1/\delta)}{n}}+\frac{11\log(1/\delta)}{n}.
\end{align*}
We claim that 
\begin{align}\label{ukClaimReturnTo}
u_k \leq  \variance_k + 12 \cdot \sqrt{\rho^*(\randomZSequenceSizeN)\cdot u_k}+2\sqrt{ \frac{v_k\log(1/\delta)}{n}}+\frac{11\log(1/\delta)}{n}.
\end{align}
Indeed, either $u_k\leq \rho^*(\randomZSequenceSizeN)$, in which case the claim (\ref{ukClaimReturnTo}) holds, or $u_k>\rho^*(\randomZSequenceSizeN)$. If the latter holds we combine $\phi(\randomZSequenceSizeN,\rho^*(\randomZSequenceSizeN))=\rho^*(\randomZSequenceSizeN)$ with the fact that $\phi(\randomZSequenceSizeN,r)/\sqrt{r}$ is non-increasing implies $\sqrt{\rho^*(\randomZSequenceSizeN)}=\phi({\randomZSequenceSizeN},\rho^*(\randomZSequenceSizeN))/\sqrt{\rho^*(\randomZSequenceSizeN)}\geq \phi({\randomZSequenceSizeN},u_k)/\sqrt{u_k}$, which by (\ref{rademacherLeqPhiBoundIneq}) yields the claim (\ref{ukClaimReturnTo}). Hence, in either case (\ref{ukClaimReturnTo}) holds. Now by plugging in the definition for $u_k$ and subtracting $\variance_k +2\sqrt{ {v_k\log(1/\delta)}/{n}}+{11\log(1/\delta)}/{n}$ from both sides we obtain
\begin{align*}
\rademacherComplexityEmpirical\left(\G_k,\randomZSequenceSizeN\right) &\leq   \sqrt{\rho^*(\randomZSequenceSizeN)\cdot \left( \variance_k  +12 \cdot  \rademacherComplexityEmpirical\left(\G_k,\randomZSequenceSizeN\right)+2\sqrt{ \frac{v_k\log(1/\delta)}{n}}+\frac{11\log(1/\delta)}{n}\right) }\\
& \leq \sqrt{(12 \rho^*(\randomZSequenceSizeN)) \cdot \rademacherComplexityEmpirical\left(\G_k,\randomZSequenceSizeN\right)}+\sqrt{\rho^*(\randomZSequenceSizeN)\cdot\left(2 v_k +\frac{12\log(1/\delta)}{n}\right)}.
\end{align*}
By Lemma \ref{standardBoundingByTheRootLemma} this implies
\begin{align*}
\rademacherComplexityEmpirical\left(\G_k,\randomZSequenceSizeN\right) \leq 12 \cdot \rho^*(\randomZSequenceSizeN)+2\sqrt{\rho^*(\randomZSequenceSizeN)\cdot\left(2 v_k +\frac{12\log(1/\delta)}{n}\right)}.
\end{align*}
Hence, by (\ref{firstEmpiricalDevPfOflocalRademacherConcentrationInequality}) with probability at least $1-4\log(n) \cdot \delta$ the following holds for all $k \in\{0,\cdot \lfloor \log_2(n)\rfloor -1\}$ and $g \in \G_k$ we have
{\small
\begin{align*}
\left|\empiricalProbDistributionRandomZ(g)-\probDistribution(g)\right| &\leq 6 \cdot  \rademacherComplexityEmpirical\left(\G_k,\randomZSequenceSizeN\right)+2 \sqrt{ \frac{v_k\log(1/\delta)}{n}}+\frac{11\log(1/\delta)}{n} \nonumber\\
  &\leq
  72 \cdot \rho^*(\randomZSequenceSizeN)+12\sqrt{\rho^*(\randomZSequenceSizeN)\cdot\left(2 v_k +\frac{12\log(1/\delta)}{n}\right)} +2 \sqrt{ \frac{v_k\log(1/\delta)}{n}}+\frac{11\log(1/\delta)}{n} \nonumber\\
   &\leq
  72 \cdot \rho^*(\randomZSequenceSizeN)+12\sqrt{\rho^*(\randomZSequenceSizeN)\cdot\left( \max\left\lbrace \frac{8}{n}, 4\probDistribution(g^2)\right\rbrace +\frac{12\log(1/\delta)}{n}\right)}\\
  &\hspace{2cm}+2 \sqrt{ \frac{ \max\left\lbrace \frac{8}{n}, 4\probDistribution(g^2)\right\rbrace\cdot \log(1/\delta)}{n}}+\frac{11\log(1/\delta)}{n} \nonumber\\
&\leq
  72 \cdot \rho^*(\randomZSequenceSizeN)+24\sqrt{\rho^*(\randomZSequenceSizeN)\cdot\left(  \probDistribution(g^2) +\frac{5\log(1/\delta)}{n}\right)}+4 \sqrt{ \frac{ \probDistribution(g^2)\cdot \log(1/\delta)}{n}}+\frac{17\log(1/\delta)}{n} \nonumber\\
&\leq2\cdot \sqrt{\probDistribution(g^2)\cdot\left( 72\cdot \rho^*(\randomZSequenceSizeN)+\frac{ 2\log(1/\delta)}{n}\right)}+ 132 \cdot \rho^*(\randomZSequenceSizeN)+\frac{30\log(1/\delta)}{n}. 
\end{align*}
}
By noting that $\G \subseteq \bigcup_{k = 0}^{\lfloor \log_2(n) \rfloor-1}\G_k$ and taking $\delta/4\log(n)$ in place of $\delta$ in the above bound, the conclusion of the theorem follows. 
\end{proof}

\section{The existence of the Bayes optimal predictor}\label{BayesExistenceAppendix}
In this section we consider the existence of Bayes optimal predictor. Recall that, given a loss function $\lossFunction: \actionSpace\times \Y \rightarrow [0,\lossFunctionBound]$ and a Borel probability distribution $\probDistribution$ on $X \times \Y$, a function $\oracleFunction \equiv \oracleFunction_{\lossFunction,\probDistribution}\in \measurableMaps(\X,\actionSpace)$ is said to be a Bayes optimal predictor if it satisfies 
\begin{align}\label{defnBayesOptimalPredictorAppendix}
 \risk(\oracleFunction) = \inf_{{\functionGeneral} \in \measurableMaps(\X,\actionSpace)}\left\lbrace \risk({\functionGeneral})\right\rbrace. 
\end{align}
In this section we shall show that a Bayes optimal predictor exists under mild conditions (see Proposition \ref{bayesOptimalExistsWhenLossContinuousActionSpaceCompact}). We require the measurable maximum theorem.

\begin{defn}[Carath\'{e}odory function] A function $g:\X \times \actionSpace \rightarrow \R$ is said to be a Carath\'{e}odory function if:
\begin{itemize}
    \item For all $x \in \X$, the map $u\mapsto g(x,u)$ is a continuous function of $u \in \actionSpace$,
    \item For all $u \in \actionSpace$, the map $x \mapsto g(x,u)$ is a measurable function of $x \in \X$.
\end{itemize}
\end{defn}
We shall utilise the following simplified version of the measurable maximum theorem.
\begin{lemma}[Measurable maximum theorem]\label{measurableMaximumThm} Suppose that $\X$ is a measurable space and $\actionSpace$ is a compact metric space. Let  $g:\X \times \actionSpace \rightarrow \R$ be a Carath\'{e}odory function. Define $m:\X \rightarrow \R$ by $m(x):=\sup_{u \in \actionSpace}\left\lbrace g(x,u)\right\rbrace$. Then $m_g$ is a measurable function and there exists a measurable function $h:\X \rightarrow \actionSpace$ with the property that $g(x,h(x))=m(x)$ for all $x \in \X$.
\end{lemma}
\begin{proof}
This is a special case of \cite[Theorem 18.19]{charalambos2013infinite}.
\end{proof}
We also utilise regular conditional distributions.
\begin{lemma}[Regular conditional distributions]\label{regularConditionalProbsExistThm} Let $\probDistribution$ be a Borel probability measure on the product space $\X\times \Y$, where $\X$ and $\Y$ are complete separable metric spaces. Then there exists a family $(\probDistribution^x)_{x \in \X}$ is a family of Borel probability measures on $\Y$ such that for all measurable functions $h:\X \times \Y \rightarrow \R$ with $\int |h(x,y)| d\probDistribution(x,y)<\infty$, the mapping $x\mapsto \int_{\Y} h(x,y)d\probDistribution^x(y)$ is measurable and satisfies
\begin{align}\label{regularConditionalProbsExistEq}
\int_{\X\times \Y} h(x,y) d\probDistribution(x,y) = \int_{\X}\left(\int_{\Y} h(x,y)d\probDistribution^x(y)\right)d\probDistribution_X(x).
\end{align}
\end{lemma}
\begin{proof}
Apply \cite[Theorem 8.37]{klenke2013probability} combined with the fact that both $\X$ and $\Y$ are complete separable metric spaces, and so $\X\times \Y$ is separable and completely meterizable, and hence Borel.
\end{proof}
We shall refer to $(\probDistribution^x)_{x \in \X}$ as the regular conditional distribution with respect to $X$. We can now give a proposition which gives sufficient conditions for the existence of a Bayes optimal predictor $\oracleFunction$.

\begin{prop}\label{bayesOptimalExistsWhenLossContinuousActionSpaceCompact} Suppose that $\actionSpace$ is a compact metric space, $\probDistribution$ is a Borel probability distribution on $\X \times \Y$ and $\lossFunction:\actionSpace \times \Y \rightarrow [0,\lossFunctionBound]$ is a bounded loss function which is continuous in its first argument. Then there exists a Bayes optimal predictor $\oracleFunction \in \measurableMaps(\X,\actionSpace)$. Moreover, any Bayes optimal predictor $\oracleFunction \in \measurableMaps(\X,\actionSpace)$ satisfies
\begin{align}\label{conclusionFromBayesOptimalExistsWhenLossContinuousActionSpaceCompact}
\oracleFunction(x) \in \argmin_{u \in \actionSpace}\left\lbrace \int_{\Y} \lossFunction(u,y)d\probDistribution^x(y)\right\rbrace
\end{align}
for 
$\probDistribution_X$ almost every $x \in \X$.
\end{prop}
\begin{proof} We begin by applying Theorem \ref{regularConditionalProbsExistThm} to obtain a  regular conditional distribution  $(\probDistribution^x)_{x \in \X}$ such that given any $h:\X \times \Y \rightarrow \R$ with $\int |h(x,y)| d\probDistribution(x,y)<\infty$ the map $x\mapsto \int_{\Y} h(x,y)d\probDistribution^x(y)$ is measurable and satisfies \eqref{regularConditionalProbsExistEq}. In particular, given any $\phi \in \measurableMaps(\X,\actionSpace)$ the map $(x,y)\mapsto \lossFunction(\phi(x),y)$ is a bounded measurable map and so $x \mapsto \int_{\Y} \lossFunction(\phi(x),y)d\probDistribution^x(y)$ is a measurable map satisfying,
\begin{align}\label{riskViaConditionalDistributions}
\risk(\phi) = \int_{\X\times \Y} \lossFunction(\phi(x),y)d\probDistribution(x,y)= \int_{\X}\left(\int_{\Y} \lossFunction(\phi(x),y)d\probDistribution^x(y)\right) d\probDistribution_X(x).
\end{align}
Now define a function $g:\X \times \actionSpace \rightarrow \R$ by $g(x,u):=-\int_{\Y} \lossFunction(u,y)d\probDistribution^x(y)$. Observe that for each $u \in \actionSpace$, $x\mapsto g(x,u)$ is measurable. Moreover, for each $y \in \Y$, the function $u \mapsto -\lossFunction(u,y)$ is continuous and so by the dominated convergence theorem it follows that $u \mapsto g(x,u)$ is continuous for each $x \in \X$. Hence, we have confirmed that $g$ is a Carath\'{e}odory function. Since $\actionSpace$ is compact we can apply Theorem \ref{measurableMaximumThm} to see that there exists a function $\oracleFunction:\X \rightarrow \actionSpace$ with the property that $g(x,\oracleFunction(x))=\sup_{u \in \actionSpace}\{g(x,u)\}$ for all $x \in \X$. Hence, for all $x \in \X$, we have $\int_{\Y} \lossFunction(\oracleFunction(x),y)d\probDistribution^x(y) \leq \inf_{u \in \actionSpace}\{ \int_{\Y} \lossFunction(u,y)d\probDistribution^x(y)\}$ so \eqref{conclusionFromBayesOptimalExistsWhenLossContinuousActionSpaceCompact} holds. By \eqref{riskViaConditionalDistributions} it follows that $\oracleFunction$ satisfies \eqref{defnBayesOptimalPredictorAppendix} and so is a Bayes optimal predictor. Suppose on the other hand that $\tilde{\phi} \in \measurableMaps(\X,\actionSpace)$ is such that $\tilde{\phi}(x)\notin \arginf_{u \in \actionSpace}\left\lbrace \int_{\Y} \lossFunction(u,y)d\probDistribution^x(y)\right\rbrace$ on a set $A\subseteq \X$ of positive $\probDistribution_X$ measure. Then by \eqref{riskViaConditionalDistributions} it follows that $\risk(\tilde{\phi})>\risk(\oracleFunction)$ and so $\tilde{\phi}$ is not a Bayes optimal predictor.
\end{proof}

\section{Proof of Lemma \ref{logCoverNumbersImplySmallLocalRademacherBoundLemma} }\label{A2}

To prove Lemma \ref{logCoverNumbersImplySmallLocalRademacherBoundLemma} we require the following elementary lemma.
\begin{lemma}\label{logIntegralBoundLemma} Given any $\Delta$, $T>0$ we have
\begin{align*}
\int_0^{\Delta} \logBar^{\frac{1}{2}}\left(\frac{{T}}{\epsilon}\right)d\epsilon \leq \Delta \cdot \left(\logBar^{\frac{1}{2}}\left(\frac{{T}}{\Delta}\right)+\frac{\sqrt{\pi}}{2}\right).
\end{align*}
\end{lemma}
\begin{proof}
First assume that ${T}/\Delta\geq e$, so we have
\begin{align*}
\int_0^{\Delta} \logBar^{\frac{1}{2}}\left(\frac{{T}}{\epsilon}\right)d\epsilon &=\int_0^{\Delta} \log^{\frac{1}{2}}\left(\frac{{T}}{\epsilon}\right)d\epsilon = \Delta \cdot \int_0^1 \log^{\frac{1}{2}}\left(\frac{{T}}{\Delta \cdot z}\right)dz\\ &\leq \Delta \cdot \left( \log^{\frac{1}{2}}\left(\frac{{T}}{\Delta}\right)+\int_0^1\log^{\frac{1}{2}}(1/z)dz\right)= \Delta \cdot \left(\log^{\frac{1}{2}}\left(\frac{{T}}{\Delta}\right)+\frac{\sqrt{\pi}}{2}\right),
\end{align*}
where we use the fact that $\int_0^1 {\log^{\frac{1}{2}}(1/z)}dz = \sqrt{\pi}/2$ by a change of variables. On the other hand, if  ${T}/\Delta< e$ we have
\begin{align*}
\int_0^{\Delta} \logBar^{\frac{1}{2}}\left(\frac{{T}}{\epsilon}\right)d\epsilon & = \int_0^{\frac{{T}}{e}} \log^{\frac{1}{2}}\left(\frac{{T}}{\epsilon}\right)d\epsilon +\left(\Delta -\frac{{T}}{e}\right)\\
& = \frac{{T}}{e} \cdot \left(\log^{\frac{1}{2}}\left(\frac{{T}}{{T}/e}\right)+\frac{\sqrt{\pi}}{2}\right)+\left(\Delta -\frac{{T}}{e}\right) = \Delta+ \frac{{T}}{e} \cdot \frac{\sqrt{\pi}}{2} \\
& < \Delta \cdot \left( 1+\frac{\sqrt{\pi}}{2}\right) = \Delta \cdot \left(\logBar^{\frac{1}{2}}\left(\frac{{T}}{\Delta}\right)+\frac{\sqrt{\pi}}{2}\right),
\end{align*}
where we have applied the previous inequality with ${T}/e$ in place of $\Delta$.

\end{proof}

\begin{proof}[Proof Lemma \ref{logCoverNumbersImplySmallLocalRademacherBoundLemma}] For the purpose of the proof define $\G_{\randomProjection}(r):= \{ g \in \G_{\randomProjection}: \empiricalProbDistributionDeterministicZ(g^2)\leq r\}$ and $\uSequenceSizeN = \{u_j\}_{j \in [n]} \in \left(\R^k\right)^n$ where $u_j = \randomProjection(x_j) \in \R^k$. Given any $g_0$, $g_1 \in \G_{\randomProjection}$ we may choose $f_0$, $f_1 \in \F_k$ such that $g_0(z)=\lossFunction(f_0(\randomProjection(x)),y)-\lossFunction(\oracleFunction(x),y)$ and $g_1(z)=\lossFunction(f_1(\randomProjection(x)),y)-\lossFunction(\oracleFunction(x),y)$ for all $z=(x,y) \in \X\times \Y$. By Assumption \ref{lipschitzLossFunctionAssumption} we have
\begin{align*}
\dist_{\zSequenceSizeN}(g_0,g_1) &:= \sqrt{\empiricalProbDistributionDeterministicZ\left((g_0-g_1)^2\right)}
=\sqrt{\frac{1}{n}\sum_{j \in [n]}\left(\lossFunction\left(f_0(u_j),y_j\right)-\lossFunction\left(f_0(u_j),y_j\right)\right)^2}\\
&\leq \lipschitzConstantLoss \cdot \sqrt{\frac{1}{n}\sum_{j \in [n]}\metricActionSpace\left(f_0(u_j),f_1(u_j)\right)^2}=\lipschitzConstantLoss \cdot {\dist_{\uSequenceSizeN}^{\actionSpace}}(f_0,f_1). 
\end{align*}
Hence, by Assumption \ref{logarithmicCoveringNumbersAssumption} for all $\epsilon>0$ we have 
\begin{align*}
\log \left(\coveringNumber\left(\G_{\randomProjection}(r),\dist_{\zSequenceSizeN},{\epsilon}\right)\right)
&\leq \log \left(\coveringNumber\left(\G_{\randomProjection},\dist_{\zSequenceSizeN},{\epsilon}\right)\right)\leq 
\log \left(\coveringNumber\left(\lowDimensionalFunctionClass,{\dist_{\uSequenceSizeN}^{\actionSpace}},\frac{\epsilon}{\lipschitzConstantLoss}\right)\right)\\
&\leq \coveringNumberConstant \cdot k \cdot \logBar \left(\frac{\lipschitzConstantLoss \hypothesisClassBound n}{\epsilon k}\right).
\end{align*}
Note also that $\sup_{g \in \G_{\randomProjection}(r)}\{  \sqrt{\empiricalProbDistributionDeterministicZ\left(g^2\right)} \}\leq \sqrt{r}$. Hence, by Dudley's inequality (Theorem \ref{dudleysInequalityTheorem}) we have
\begin{align*}
\rademacherComplexityEmpirical\left(\G_{\randomProjection}(A),\zSequenceSizeN\right) &\leq \int_{0}^{\sqrt{r}} \sqrt{\frac{\log \coveringNumber\left(\G, \dist_{\zSequenceSizeN} ,\epsilon \right) }{n}} d \epsilon\leq \sqrt{\frac{\coveringNumberConstant \cdot k}{n}} \cdot \int_{0}^{\sqrt{r}} { \logBar^{\frac{1}{2}} \left(\frac{\lipschitzConstantLoss \hypothesisClassBound n}{\epsilon k}\right)}d\epsilon\\
&\leq \sqrt{\frac{\coveringNumberConstant \cdot k \cdot r}{n}} \cdot  \left(\logBar^{\frac{1}{2}} \left(\frac{\lipschitzConstantLoss \hypothesisClassBound n}{ k \sqrt{r}}\right)+\frac{\sqrt{\pi}}{2}\right) \leq 2\sqrt{\frac{\coveringNumberConstant \cdot k \cdot r}{n}} \cdot  \logBar^{\frac{1}{2}} \left(\frac{\lipschitzConstantLoss \hypothesisClassBound n}{ k \sqrt{r}}\right),
\end{align*}
where the penultimate inequality follows from Lemma \ref{logIntegralBoundLemma}.

\end{proof}

\section{Proof of Lemma \ref{LSlemma}}\label{A:Slawski}
Before proving Lemma \ref{LSlemma} we recall the following useful result due to \cite{baraniuk2009random}.

\begin{lemma}\cite[Lemma 5.1]{baraniukSimple}\label{baraniuk}
Let $A\in\R^{k\times d}$ be a matrix that satisfies Assumption \ref{JLAssumption} with constant $\johnsonLindenstraussConstant \geq 1$. There exists a constant $\tjohnsonLindenstraussConstant>0$ depending only on $\johnsonLindenstraussConstant$ such that for any $r$-dimensional linear subspace $\mathcal{U} \subset \R^d$ with $r<k$ and $\epsilon\in(0,1)$, the following holds with  probability at least $1-(12/\epsilon)^re^{-\tjohnsonLindenstraussConstant\epsilon k}$,
\begin{align}
(1-\epsilon)\|x\|_2\le \|Ax\|_2\le (1+\epsilon)\|x\|_2, 
\end{align}
for all $x\in \mathcal{U}$.
\end{lemma}

\begin{proof}[Proof of Lemma \ref{LSlemma}] Let $m\leq \min\{q,d\}$ denote be the rank of $\XX$ and take a singular value decomposition $\XX=U\Lambda V^T$ so that $U \in \R^{d\times m}$, $\Lambda \in \R^{m\times m}$ and $V \in \R^{q \times m}$ with $\Lambda$ diagonal and $U^\top U = V^\top V = I_m$. Next let ${r_\circ}:=\min\{r,m\}$ and partition $U=[u_1,\ldots,u_m]$ column-wise into matrix $U_r \in \R^{d\times {r_\circ}}$, consisting of the first ${r_\circ}$ columns of $U$ and $\tilde{U}_r\in \R^{d\times(m- {r_\circ})}$ consisting of the remaining $m-{r_\circ}$ columns (so $\tilde{U}_r$ may be empty). Similarly,  let $\Lambda_r \in \R^{{r_\circ} \times {r_\circ}}$ be the upper diagonal block of $\Lambda$ containing the top ${r_\circ}$ singular values, and $\tL_r \in \R^{(m-{r_\circ}) \times (m-{r_\circ})}$ the lower (possibly empty) diagonal block. In particular, we have $U\Lambda=[U_r\Lambda_r,\tilde{U}_r\tilde{\Lambda}_r]$. Next we define a pair of events $E_{\mathrm{sp}}$ and $E_{\mathrm{ta}}$ by
\begin{align*}
E_{\mathrm{sp}}&:=\left\lbrace  \|U_rz\|_2\le 2 \|AU_rz\|_2 \text{ for all } z \in \R^{{r_\circ}}\right\rbrace,\\
E_{\mathrm{ta}}&:=\left\lbrace \|Au_\ell\|_2\le \sqrt{2}\text{ for }\ell \in \{{r_\circ}+1,\ldots,m\}\right\rbrace.
\end{align*}
By Lemma \ref{baraniuk} $E_{\mathrm{sp}}$ holds with probability at least $1-24^re^{-\tjohnsonLindenstraussConstant k/2}$. Moreover, by Assumption \ref{JLAssumption} $E_{\mathrm{ta}}$ holds with probability at least $1-(q-r+1)e^{-k/4\johnsonLindenstraussConstant}$. Hence, by the union bound it suffices to work on the event $E_{\mathrm{sp}}\cap E_{\mathrm{ta}}$ and show that \eqref{eq:conclusionOfSlawskiTypeLemma} holds. Now on the event $E_{\mathrm{sp}}$ the $k \times {r_\circ}$ matrix $AU_r$ is of rank $r_\circ$ with singular values at least $1/2$, so $\|(AU_r)^+\|_{\text{spec}} \leq 2$ and $(AU_r)^+(AU_r)=I_{r_\circ}$ is the ${r_\circ}\times {r_\circ}$ identity matrix ($M^+$ denotes the Moore-Penrose inverse of $M$). In addition, on the event $E_{\mathrm{ta}}$ writing $(\tilde{\Lambda}_{r,j})_{j \in [m-r_\circ]}$ diagonal elements of  $\tilde{\Lambda}_r$, given any $z=(z_j)_{j \in [m-r_\circ]} \in \R^{m-r_\circ}$ we have
\begin{align*}
\|\randomProjection\tilde{U}_r\tilde{\Lambda}_rz\|_2 & =\Bigl\|\randomProjection \bigl(\sum_{j = r_\circ+1}^m \tilde{\Lambda}_{r,j}z_{j} u_{j}\bigr) \Bigr\|_2 =\Bigl\| \sum_{j = r_\circ+1}^m \tilde{\Lambda}_{r,j}z_{j} \randomProjection u_{j} \Bigr\|_2\leq \sqrt{2} \sqrt{\sum_{j = r_\circ+1}^m \tilde{\Lambda}^2_{r,j}z_{j}^2 }\\ &\leq \sqrt{2} \|z\|_2 \sqrt{ \sum_{j = r_\circ+1}^m \tilde{\Lambda}_{r,j}^2} \leq \sqrt{2} \|z\| \sqrt{ \sum_{j \geq r+1} \lambda_j(\XX\XX^\top)}, 
\end{align*}
so $\|\randomProjection\tilde{U}_r\tilde{\Lambda}_r\|_{\mathrm{spec}}^2\leq  \sqrt{2}\sum_{j = r+1}^m \lambda_j(\XX\XX^\top)$. Hence, choosing $w\equiv w(A,r):=\{w_{\diamond}^\top U_r(AU_r)^+\}^{\top}$,
\begin{align*}
\inf_{\tilde{w}\in \R^k}\|\tilde{w}^\top\randomProjection \XX - w_{\diamond}^\top \XX\|_2^2& \leq \|( w^\top\randomProjection  - w_{\diamond}^\top) \XX\|_2^2
= \|(w^\top\randomProjection  - w_{\diamond}^\top) U\Lambda\|_2^2\\
&\leq \|(w^\top\randomProjection - w_{\diamond}^\top) U_r\Lambda_r\|_2^2+\|(w^\top\randomProjection - w_{\diamond}^\top) \tilde{U}_r\tilde{\Lambda}_r\|_2^2\\
&=\|w_{\diamond}^\top U_r(AU_r)^+(\randomProjection U_r)\Lambda_r - w_{\diamond}^\top U_r\Lambda_r\|_2^2+\|(w^\top\randomProjection - w_{\diamond}^\top) \tilde{U}_r\tilde{\Lambda}_r\|_2^2\\
&=\|w_{\diamond}^\top(U_r(AU_r)^+\randomProjection - I_d)\tilde{U}_r\tilde{\Lambda}_r\|_2^2\\ &
\leq \|w_{\diamond}\|_2^2\cdot\|(U_r(AU_r)^+\randomProjection - I_d)\tilde{U}_r\tilde{\Lambda}_r\|_{\mathrm{spec}}^2\\
&\leq 2\|w_{\diamond}\|_2^2\cdot \left(\|U_r(AU_r)^+\randomProjection\tilde{U}_r\tilde{\Lambda}_r\|_{\mathrm{spec}}^2+\lambda_{r+1}(\XX\XX^\top)\right)\\
&\leq 2\|w_{\diamond}\|_2^2\cdot \left(\|(AU_r)^+\|_{\mathrm{spec}}^2\cdot \|\randomProjection\tilde{U}_r\tilde{\Lambda}_r\|_{\mathrm{spec}}^2+\lambda_{r+1}(\XX\XX^\top)\right)\\
&\leq 18\|w_{\diamond}\|_2^2~\sum_{j = r+1}^m \lambda_j(\XX\XX^\top),
\end{align*}
as required.
\end{proof}

\bibliography{mybib}
\bibliographystyle{apalike}

\end{document}